%% file: main.tex
\documentclass[twoside]{article}

\usepackage[accepted]{aistats2021}
\usepackage[round]{natbib}

\usepackage[utf8]{inputenc} % allow utf-8 input
\usepackage[T1]{fontenc}    % use 8-bit T1 fonts
\usepackage{hyperref}       % hyperlinks
\usepackage{url}            % simple URL typesetting
\usepackage{booktabs}       % professional-quality tables
\usepackage{amsfonts}       % blackboard math symbols
\usepackage{nicefrac}       % compact symbols for 1/2, etc.
\usepackage{microtype}      % microtypography
\usepackage{graphicx}
\usepackage[labelformat=simple]{subcaption}

\usepackage{booktabs} % for professional tables
\usepackage{amsmath}
\usepackage{amssymb}
\usepackage{amsthm}
\usepackage{amsfonts}
\usepackage{mathtools}
\usepackage{bm}
\usepackage{comment}
\usepackage{float}
\usepackage{paralist}
\usepackage{multirow}
\usepackage[capitalize]{cleveref}
\usepackage[textsize=tiny]{todonotes}
\usepackage{array}
\usepackage[shortlabels]{enumitem}

\newcommand{\dx}{\mathrm{d}x}
\newcommand{\dy}{\mathrm{d}y}
\newcommand{\dz}{\mathrm{d}z}
\newcommand{\EE}{\mathbb{E}}

\newcommand{\RR}{\mathbb{R}}
\newcommand{\NN}{\mathbb{N}}

\newcommand{\zV}{\mathbf{0}}
\newcommand{\cF}{\mathcal{F}}
\newcommand{\df}{\phi}
\newcommand{\fP}{\psi}
\newcommand{\cC}{\mathcal{C}}
\newcommand{\cO}{\mathcal{O}}
\newcommand{\cN}{\mathcal{N}}
\newcommand{\cH}{\mathcal{H}}
\newcommand{\hp}{\hat{\gf}}
\newcommand{\ap}{\alpha}
\newcommand{\eps}{\epsilon}
\newcommand{\veps}{\varepsilon}
\newcommand{\Sp}{\Omega}
\newcommand{\D}{\delta}
\newcommand{\W}{\mathrm{W}}
\newcommand{\PD}{D_f}
\newcommand{\gf}{f}
\newcommand{\KL}{\mathrm{KL}}
\newcommand{\SKL}{\mathrm{SKL}}
\newcommand{\JS}{\mathrm{JS}}
\newcommand{\SH}{\mathrm{H}^2}
\newcommand{\TV}{\mathrm{TV}}
\newcommand{\CS}{\chi^2}
\newcommand{\HA}{\mathcal{H}_\alpha}
\newcommand{\PQ}[1]{(P_{#1},Q_{#1})}

\newcommand{\ee}{\mathrm{e}}
\newcommand{\vspan}{\mathrm{span}}
\newcommand{\cl}{\mathrm{cl}}
\newcommand{\trace}{\mathrm{Tr}}
\newcommand{\var}{\mathrm{Var}}
\newcommand{\cov}{\mathrm{Cov}}
\newcommand{\dom}{\mathrm{dom}}
\newcommand{\diam}{\mathrm{diam}}

\newcommand{\bigCI}{\mathrel{\text{\scalebox{1.07}{$\perp\mkern-10mu\perp$}}}}

\newtheoremstyle{break}% name
  {}%         Space above, empty = `usual value'
  {}%         Space below
  {\itshape}% Body font
  {}%         Indent amount (empty = no indent, \parindent = para indent)
  {\bfseries}% Thm head font
  {}%        Punctuation after thm head
  {\newline}% Space after thm head: \newline = linebreak
  {}%         Thm head spec
\newtheorem{definition}{Definition}
\newtheorem{theorem}{Theorem}
\newtheorem{corollary}[theorem]{Corollary}
\newtheorem{lemma}[theorem]{Lemma}
\newtheorem{example}{Example}
\newtheorem{counterexample}[example]{Counter-Example}

\newenvironment{prevproof}[1]{\noindent {\em {Proof of \cref{#1}:}}}{\hfill $\square$\vskip \belowdisplayskip}

\begin{document}

\runningtitle{GANs with Conditional Independence Graphs}


\twocolumn[

\aistatstitle{GANs with Conditional Independence Graphs:\\
On Subadditivity of Probability Divergences}

\aistatsauthor{Mucong Ding \And Constantinos Daskalakis \And Soheil Feizi}

\aistatsaddress{University of Maryland,\\College Park\\\href{mailto:mcding@umd.edu}{\nolinkurl{mcding@umd.edu}}  \And Massachusetts Institute of Technology \And University of Maryland,\\College Park}
]

\begin{abstract}
Generative Adversarial Networks (GANs) are modern methods to learn the underlying distribution of a data set. GANs have been widely used in sample synthesis, de-noising, domain transfer, etc. GANs, however, are designed in a {\it model-free} fashion where {\it no} additional information about the underlying distribution is available. In many applications, however, practitioners have access to the underlying independence graph of the variables, either as a Bayesian network or a Markov Random Field (MRF). We ask: how can one use this additional information in designing {\it model-based} GANs? In this paper, we provide theoretical foundations to answer this question by studying subadditivity properties of probability divergences, which establish upper bounds on the distance between two high-dimensional distributions by the sum of distances between their marginals over (local) neighborhoods of the graphical structure of the Bayes-net or the MRF. We prove that several popular probability divergences satisfy some notion of subadditivity under mild conditions. These results lead to a principled design of a model-based GAN that uses a set of simple discriminators on the neighborhoods of the Bayes-net/MRF, rather than a giant discriminator on the entire network, providing significant statistical and computational benefits. Our experiments on synthetic and real-world datasets demonstrate the benefits of our principled design of model-based GANs.
\end{abstract}

% Main sections
\input{Introduction}
\input{Notations}
\input{Bayes-Nets}
\input{MRFs}
\input{Local}
\input{GANs}
\input{Experiments}

% 9-page limit for the sections above
%\clearpage
%\newpage

\subsubsection*{Acknowledgements}
We thank the Simons Institute for the Theory of Computing, where this collaboration started, during the ``Foundations of Deep Learning'' program. This project was supported in part by NSF CAREER AWARD 1942230, Simons Fellowship, Qualcomm Faculty Award, IBM Faculty Award, DOE Award 302629-00001 and a sponsorship from Capital One. Constantinos Daskalakis was supported by NSF Awards IIS-1741137, CCF-1617730 and CCF-1901292, by a Simons Investigator Award, by the DOE PhILMs project (No. DE-AC05-76RL01830), by the DARPA award HR00111990021, by a Google Faculty award, and by the MIT Frank Quick Faculty Research and Innovation Fellowship.

\bibliographystyle{unsrtnat}
\bibliography{main}

% Appendices
\clearpage
\newpage
\onecolumn
\appendix
{\centering \bf \LARGE
    Appendix to\par
    GANs with Conditional Independence Graphs:\\On Subadditivity of Probability Divergences\par
}

\input{Proofs}
\input{Appendices}

\end{document}

%% file: Introduction.tex
\section{Introduction}
\label{sec:introduction}

Generative Adversarial Networks (GANs)~\citep{goodfellow2014generative} have been successfully used to model complex distributions such as image data. GANs model the learning problem as a {\it min-max} game between generator and discriminator functions. Depending on the specific cost function and constraints on the discriminator network, the associated optimization problem aims at estimating a Wasserstein distance \citep{arjovsky2017wasserstein}, an Integral Probability Measure (IPM)~\citep{muller1997integral}, an $f$-divergence \citep{nowozin2016f}, etc., between the target and generated distributions.

GANs are often designed in a {\it model-free} fashion where {\it no} additional information about the underlying distribution is available\footnote{Some works have studied GANs under some strict assumptions on the input data distribution. For example, \citet{feizi2017understanding} has designed GANs for multivariate Gaussians while \citet{balaji2019normalized} and \citet{farnia2020gat} have studied GANs for mixtures of Gaussians. In contrast, our method is applicable to any Bayesian network or Markov Random Field, which are significantly richer families of distributions.}. In some applications, however, one may have some side information about the data distribution. For example, one may know that there is a Markov chain governing the underlying independence graph of the variables. In general, the underlying independence graph of variables may be available as a Bayesian network (i.e. a directed graph) or a Markov Random Field (i.e. an undirected graph). In this paper, we ask: how can we use this additional information in a principled {\it model-based} design of GANs?

In this paper, we provide theoretical foundations to answer the aforementioned question for high-dimensional distributions with conditional independence structure captured by either a Bayesian network or a Markov Random Field (MRF). We mainly focus on the application to GANs, while the theory developed can be used by any other type of adversarial learning that exploits discriminator networks. The pertinent question is whether a known Bayes-net or MRF structure can be exploited to design a GAN with multiple discriminators that are localized and simple. In particular, we are interested in whether we can replace the large discriminator of the vanilla GAN implementation with several simple discriminators that are used to enforce constraints on local neighborhoods of the Bayes-net or the MRF (i.e. local discriminators). Ignoring the underlying conditional independence structure we might know about the target distribution and letting the GAN ``learn it on its own''  requires a very large discriminator network, especially in applications where data is gathered across many time steps. Large discriminators face computational and statistical challenges, given that min-max training is computationally challenging, and statistical hypothesis testing in large dimensions requires sample complexity exponential in the dimension; see e.g.~discussions by~\citet{daskalakis2017square,daskalakis2019testing,canonne2020testing}.

Our proposed framework is based on {\it subadditivity} properties of probability divergences over a Bayes-net or a MRF, which establish upper bounds on the distance between two high-dimensional distributions with the same Bayes-net or MRF structure by the sum of distances between their marginals over (local) neighborhoods of the graphical structure of the Bayes-net or the MRF~\citep{daskalakis2017square}. For a Bayes-Net, each local neighborhood is defined as the union of a node $i$ and its parents $\Pi_i$, as it is the smallest set that encodes conditional dependence. For a MRF, the set of local neighborhoods can be defined as the set of maximal cliques $\cC$ of the underlying graph.

Let $\D$ be some divergence or probability metric, such as some Wasserstein distance or $\gf$-divergence, that is estimated by each of the local discriminators in their dedicated neighborhood. If we train a generator with the set of local discriminators, it samples a distribution $Q$ that minimizes the sum of divergences $\D$ between marginals of $P$ and $Q$ over the local neighborhoods, where $P$ is the target distribution. As per our description of what the local neighborhoods are in each case, the optimization objective becomes $\sum_{i=1}^n\D\PQ{X_i\cup X_{\Pi_i}}$ on a Bayes-net, and $\sum_{C\in\cC}\D\PQ{X_C}$ on a MRF. However, our real goal is to minimize some divergence $\D'(P, Q)$ of interest measured on the joint (high-dimensional) distributions. We say that $\D(.,.)$ satisfies {\it generalized subadditivity} if the sum $\sum_{i=1}^n\D\PQ{X_i\cup X_{\Pi_i}}$ or $\sum_{C\in\cC}\D\PQ{X_C}$ upper-bounds the divergence $\D'\PQ{}$ of interest up to some constant factor $\ap>0$ and additive error $\eps\geq0$, i.e. $\D'(P, Q)-\eps \leq \ap\cdot\sum_{i=1}^n\D\PQ{X_i\cup X_{\Pi_i}}$ (on Bayes-nets), or $\D'(P, Q)-\eps \leq \ap\cdot\sum_{C\in\cC}\D\PQ{X_C}$ (on MRFs), where $\D'$ can be the same or different from $\D$. In this sense, the generator effectively minimizes $\D'\PQ{}$ by minimizing its upper-bound. Since, in many applications, local neighborhoods can be significantly smaller than the entire graph, local discriminators targeting each of these neighborhoods will enjoy improved computational and statistical properties in comparison to a global discriminator targeting the entire graph.

The key question is which divergences or metrics exhibit subadditivity to be used in our proposed framework. For testing the identity of Bayes-nets, \citet{daskalakis2017square} shows that squared Hellinger distance, Kullback-Leibler divergence, and Total Variation distance satisfy some notion of generalized subadditivity. Since our goal in this paper is to exploit subadditivity in the design of GANs, we are interested in establishing generalized subadditivity bounds for distances and divergences that are commonly used in GAN formulations. In this work, we prove that
\begin{itemize}[leftmargin=*, topsep=0pt]
    \item Jensen-Shannon divergence used in the original GAN model \citep{goodfellow2014generative},
    \item Wasserstein distance used in Wasserstein GANs \citep{arjovsky2017wasserstein}, and Integral Probability Metric (IPM) \citep{muller1997integral} used in Wasserstein, MMD and Energy-based GANs \citep{li2015generative,zhao2017energy},
    \item and nearly all $\gf$-divergences used in $\gf$-GANs \citep{nowozin2016f},
\end{itemize}
satisfy some notion of generalized subadditivity over Bayes-nets under some mild conditions.\footnote{We discuss the notion of ``local subadditivity" in~\cref{sec:local} and \cref{appendix:local}.} Moreover, we prove that under some mild conditions
\begin{itemize}[leftmargin=*, topsep=0pt]
    \item Wasserstein distance and IPM satisfy generalized subadditivity on MRFs. 
\end{itemize}
These results establish theoretical foundations for using underlying conditional independence graphs in GAN's designs. We demonstrate benefits of our design over several synthetic and real datasets such as the synthetic ``ball throwing trajectory'' dataset and two real Bayes-net datasets: {\it the EARTHQUAKE dataset}~\citep{korb2010bayesian} and {\it the CHILD dataset}~\citep{spiegelhalter1992learning}.

\section{Related Works}

In many applications, adversarial learning has been used in a broader sense where {\it multiple local discriminators} have been employed in the learning framework. For example, in image-to-image translation methods~\citep{isola2017image, zhu2017unpaired, yi2017dualgan, choi2018stargan, yu2019free, demir2018patch}, local discriminators are applied to different patches of images~\citep{li2016precomputed}. In the analysis of time-series data as well as natural language processing (NLP) tasks, local discriminators based on sliding windows~\citep{li2019mad}, self-attention~\citep{clark2019efficient}, recurrent neural networks (RNNs)~\citep{esteban2017real, mogren2016c}, convolution neural networks (CNNs)~\citep{nie2018relgan}, and dilated causal convolutions~\citep{oord2016wavenet, donahue2018adversarial} have been applied on different subsequences of the data. These models have been applied to a wide range of tasks including image style transfer~\citep{isola2017image, zhu2017unpaired, yi2017dualgan, choi2018stargan}, inpainting~\citep{yu2019free, demir2018patch}, and texture synthesis~\citep{li2016precomputed}, as well as time-series generation~\citep{esteban2017real, mogren2016c}, imputation~\citep{liu2019naomi}, anomaly detection~\citep{li2019mad}, and even video generation~\citep{clark2019efficient} and inpainting~\citep{chang2019free}. 

Intuitively, these methods aim at structuring the generation process and/or narrowing down the purview of the discriminator to capture known dependencies leading to improved computational and statistical properties. These methods, however, are mostly not accompanied by theoretical foundations. In particular, it is not clear what subset of features each local discriminator should be applied to, how many local discriminators should be used in the learning process, and what the effect of the discriminator localization is on estimating the distance between the generated and target distributions.

%% file: Notations.tex
\section{Notation}
\label{sec:notations}

Consider a Directed Acyclic Graph (DAG) $G$ with nodes $\{1, \ldots, n\}$. Let $\Pi_i$ be the set of parents of node $i$ in $G$. Assume that $(1, \ldots, n)$ is a topological ordering of~$G$, i.e. $\Pi_i\subseteq\{1,\ldots,i-1\}$ for all~$i$. A probability distribution $P(x)$ defined over space $\Sp=\{(x_1, \ldots, x_n)\}$ is a {\em Bayes-net with respect to graph $G$} if it can be factorized as $P(x)=\prod_{i=1}^n P_{X_i|X_{\Pi_i}}(x_i|x_{\Pi_i})$.

Given an undirected graph $G$ with nodes $\{1, \ldots, n\}$, a probability distribution $P(x)$ defined over space $\Sp=\{(x_1, \ldots, x_n)\}$ is a {\em MRF with respect to graph $G$} if any two disjoint subsets of variables $A, B\subseteq \{1, \ldots, n\}$ are conditionally independent conditioning on a separating subset $S$ of variables (i.e.~$S$ such that all paths in $G$ from nodes in $A$ to nodes in $B$ pass through $S$). This conditional independence property is denoted $X_A\bigCI X_B\:|\:X_S$. Such $P(x)$ can be factorized as $P(x)=\prod_{C\in\cC}\fP_C(X_C)$, where $\cC$ is the set of maximal cliques in $G$. In this paper, unless otherwise noted, we always assume $X_i\in\RR^d$, thus $\Sp\subseteq\RR^{nd}$, and use the Euclidean metric. We always assume the density exists.

%% file: Bayes-Nets.tex
\section{Generalized Subadditivity on Bayes-nets}
\label{sec:bayes-nets}

In this section, we define the notion of {\em generalized subadditivity} of a statistical divergence $\D$ on Bayes-nets. We discuss subadditivity on MRFs in \cref{sec:mrfs}.

\begin{definition}[Generalized Subadditivity of Divergences on Bayes-nets]
\label{def:subadditivity}
    Consider two Bayes-nets $P, Q$ over the same sample space $\Sp=\{(x_1, \ldots, x_n)\}$ and defined with respect to the same DAG, $G$, i.e. factorizable as $P(x)=\prod_{i=1}^{n} P_{X_i|X_{\Pi_i}}(x_i|x_{\Pi_i})$, $Q(x)=\prod_{i=1}^{n} Q_{X_i|X_{\Pi_i}}(x_i|x_{\Pi_i})$, where ${\Pi_i}$ is the set of parents of node $i$ in $G$.
    For a pair of statistical divergences $\D$ and $\D'$, and constants $\ap>0$ and $\eps\geq 0$, if the following holds for all Bayes-nets $P,Q$ as above:
    \[
        \D'(P, Q)-\eps \leq \ap\cdot\sum_{i=1}^n\D\PQ{X_i\cup X_{\Pi_i}},
    \]
    % for every Bayes-nets $P$ and $Q$ as above, for some constant $\ap>0$, $\eps\geq 0$ and a statistical divergence $\D'$, 
    then we say that $\D$ satisfies {\em $\ap$-linear subadditivity with error $\eps$ with respect to $\D'$ on Bayes-nets}. For the  common case $\eps=0$ and $\D'=\D$, we  say that $\D$ satisfies  {\em $\alpha$-linear subadditivity on Bayes-nets}. When additionally $\alpha=1$, we say that $\D$ satisfies subadditivity on Bayes-nets.
\end{definition}

We refer to the right-hand side of the subadditivity inequality as the subadditivity upper bound. If a statistical divergence $\D$ satisfies linear subadditivity with respect to $\D'$, minimizing the subadditivity upper bound serves as a proxy to minimizing  $\D'\PQ{}$. The subadditivity upper bound is often used as the objective function in adversarial learning when local discriminators are employed.

We argue that subadditivity of $\D$ on
\begin{inparaenum}[(1)]
	\item product measures, and
	\item length-$3$ Markov Chains
\end{inparaenum}
suffices to imply subadditivity on all Bayes-nets. The claim is implicit in the proof of Theorem 2.1 by~\citet{daskalakis2017square}; we state it explicitly here and provide its proof in \cref{proof:subadditivity-markov-chain} for completeness. Roughly speaking, the proof follows because we can always combine nodes of a Bayes-net into super-nodes to obtain a $3$-node Markov Chain or a $2$-node product measure, and apply the Markov Chain/Product Measure subadditivity property recursively.

\begin{theorem}
\label{thm:subadditivity-markov-chain}
    If a divergence $\D$ satisfies the following:
    \begin{compactenum}[(1)]
        \item For any two Bayes-nets $P$ and $Q$ on DAG $X\to Y\to Z$, the following subadditivity holds: $\D\PQ{XYZ}\leq \D\PQ{XY} + \D\PQ{YZ}$.
        \item For any two product measures $P$ and $Q$ over variables $X$ and $Y$, the following subadditivity holds: $\D\PQ{XY}\leq \D\PQ{X} + \D\PQ{Y}$.
    \end{compactenum} then $\D$ satisfies subadditivity on Bayes-nets.
\end{theorem}

Using \cref{thm:subadditivity-markov-chain}, it is not hard to prove that squared Hellinger distance has subadditivity on Bayes-nets, as shown by~\citet{daskalakis2017square}. For completeness, we provide proof of the following in \cref{proof:subadditivity-sh}

\begin{theorem}[Theorem 2.1 by~\citet{daskalakis2017square}]
\label{thm:subadditivity-sh}
    The squared Hellinger distance defined as $\SH(P, Q)\coloneqq1-\int\sqrt{PQ}~\dx$ satisfies subadditivity on Bayes-nets.
\end{theorem}

\subsection{Subadditivity of \texorpdfstring{$\gf$}{f}-Divergences}
\label{subsec:bayes-fdiv}

For two probability distributions $P$ and $Q$ on $\Omega$, the $\gf$-divergence of $P$ from $Q$, denoted $\PD(P, Q)$, is defined as $\PD(P, Q)=\int_\Sp \gf\left(P(x)/Q(x)\right)Q(x)\dx$. We assume $P$ is absolutely continuous with respect to $Q$, written as $P\ll Q$. Common $\gf$-divergences are Kullback-Leibler divergence ($\KL$), Symmetric KL divergence ($\SKL$), Jensen-Shannon divergence ($\JS$), and Total Variation distance ($\TV$); see \cref{appendix:fdivs}. The subadditivity of KL-divergence on Bayes-nets is claimed by~\citet{daskalakis2017square} without a proof. We provide a proof in \cref{proof:subadditivity-kl} for completeness.

\begin{theorem}[Claimed by~\citet{daskalakis2017square}]
\label{thm:subadditivity-kl} The KL-divergence defined as $\KL\PQ{}\coloneqq\int P\log\left(P/Q\right)\dx$ satisfies subadditivity on Bayes-nets.
\end{theorem}

It follows from the proof of \cref{thm:subadditivity-kl} that the following conditions suffice for the $\KL$ subadditivity to become additivity: $\forall i$, $P_{X_{\Pi_i}}=Q_{X_{\Pi_i}}$ (almost everywhere). From the investigation of local subadditivity of $\gf$-divergences (\cref{thm:local-perturbation-suadditivity} in \cref{appendix:local}), we will see that this is the minimum set of requirements possible. The subadditivity of KL divergence easily implies the subadditivity of the Symmetric KL divergence.

\begin{corollary}
\label{coro:subadditivity-jf}
The Symmetric KL divergence defined as $\SKL\PQ{} \coloneqq \KL\PQ{} + \KL(Q,P)$ satisfies subadditivity on Bayes-nets.
\end{corollary}
Moreover, the linear subadditivity of Jensen-Shannon divergence ($\JS$) follows from the subadditivity property of squared Hellinger distance; see \cref{proof:linear-subadditivity-js}.

\begin{corollary}
\label{coro:linear-subadditivity-js}
    The Jensen-Shannon divergence defined as $\JS\PQ{} \coloneqq \frac12\KL\left(P,(P+Q)/2\right)+\frac12\KL\left(Q,(P+Q)/2\right)$ satisfies $(1/\ln2)$-linear subadditivity on Bayes-nets.
\end{corollary}

Using a slightly modified version of \cref{thm:subadditivity-markov-chain}, it is not hard to derive the linear subadditivity of Total Variation distance, which is stated without proof by~\citet{daskalakis2017square}. We provide a proof in \cref{proof:linear-subadditivity-tv} for completeness.

\begin{theorem}[Claimed by~\citet{daskalakis2017square}]
\label{thm:linear-subadditivity-tv}
    The Total Variation distance defined as $\TV(P, Q) \coloneqq \frac12\int\left|P-Q\right|\dx$ satisfies $2$-linear subadditivity on Bayes-nets.
\end{theorem}

\subsection{Subadditivity of Wasserstein Distance and IPMs}
\label{sec:ipm}

Suppose $\Sp$ is a metric space with distance $d(\cdot,\cdot)$. The $p$-Wasserstein distance $\W_p$ is defined as $\W_p\PQ{}\coloneqq(\inf_{\gamma\in\Gamma\PQ{}}\int_{\Sp\times\Sp}d(x,y)^p\mathrm{d}\gamma(x,y))^{1/p}$, where $\gamma\in\Gamma\PQ{}$ denotes the set of all possible couplings of $P$ and $Q$; see \cref{appendix:wasserstein}.

In general, Wasserstein distance does not satisfy subadditivity on Bayes-nets and MRFs shown by a counter-example using Gaussian distributions (\cref{appendix:counter-cexp-wasserstein}). However, based on the linear subadditivity of TV on Bayes-nets, one can prove that all $p$-Wasserstein distances with $p\ge 1$ satisfy $\ap$-linear subadditivity when space $\Sp$ is discrete and finite (\cref{proof:linear-subadditivity-wp}).

\begin{corollary}
\label{coro:linear-subadditivity-wp}
    If $\Sp$ is a finite metric space, $p$-Wasserstein distance for $p\geq 1$ satisfies $(2^{1/p}\diam(\Sp)/d_{\min})$-linear subadditivity on Bayes-nets, where $\diam(\Sp)$ is the diameter and $d_{\min}$ is the smallest distance between pairs of distinct points in $\Sp$.
\end{corollary}

Integral Probability Metrics (IPMs) are a class of probability distances defined as $d_{\cF}\PQ{} \coloneqq \sup_{\df\in\cF}\left\{\EE_{x\sim P}[\df(x)]-\EE_{x\sim Q}[\df(x)]\right\}$, which include the Wasserstein distance, Maximum Mean Discrepancy, and Total Variation distance. The IPM with $\cF$ being all 1-Lipschitz functions is the 1-Wasserstein distance \citep{villani2008optimal}. Practical GANs take $\cF$ as a parametric function class, $\cF=\{\df_\theta(x)|\theta\in\Theta\}$, where $\df_\theta(x)$ is a neural network. The resulting~IPMs~are~called~neural~distances~\citep{arora2017generalization}.

Next, we prove that neural distances (even those expressible by a single ReLU neuron) satisfy generalized subadditivity with respect to the Symmetric KL divergence. This property establishes substantive theoretical justification for the local discriminators used in GANs based on IPMs.

\begin{theorem}
\label{thm:subadditivity-ipm}
    Consider two Bayes-nets $P, Q$ on $\Sp=\{(X_1, \ldots, X_n)\}\subseteq\RR^{nd}$ with a common DAG $G$, and any set of function classes $\{\cF_1, \ldots, \cF_n\}$. Suppose   the following  conditions are fulfilled:
    \begin{compactenum}[(1)]
        \item the space $\Sp$ is bounded, i.e. $\diam(\Sp)<\infty$;
        \item each discriminator class ($\cF_i$) is larger than the set of single neuron networks with ReLU activations, i.e. $\{\max\{w^Tx+b,0\}\big|\|[w,b]\|_2=1\}$; and
        \item $\log(P_{X_i\cup X_{\Pi_i}}/Q_{X_i\cup X_{\Pi_i}})$ are bounded and Lipschitz continuous for all $i$.
    \end{compactenum}
    Then the neural distances defined by $\cF_1, \ldots, \cF_n$ satisfy the following $\ap$-linear subadditivity with error $\eps$ with respect to the Symmetric KL divergence on Bayes-nets:
    \[
        \SKL\PQ{}-\eps\leq\ap\cdot\sum_{i=1}^n d_{\cF_i}\PQ{X_i\cup X_{\Pi_i}},
    \]
    where $\ap$ and $\eps$ are constants independent of $P, Q$ and $\{\cF_1, \ldots, \cF_n\}$, satisfying
    \[
        \ap>R\Big((k_{\max}+1)d\Big)~~ \text{and} ~~ \eps=\cO\left(n\ap^{-\frac{2}{(k_{\max}+1)d+1}}\log\ap\right),
    \]
    where $R((k_{\max}+1)d)$ is a function that only depends on $k_{\max}$ (the maximum in-degree of $G$) and $d$ (the dimensionality of each variable of the Bayes-net). 
\end{theorem}
Regarding condition (1), bounded space $\Sp$ still allows many real-world data-types, including images and videos. Regarding condition (2), all practical neural networks using ReLU activations satisfy this requirement. Thus, the only non-trivial requirement is  condition (3). In practical GAN training, $Q$ is the output distribution of a generative model, which can be regarded as a transformation of a Gaussian distribution. Thus, in general, $Q$ is bounded and Lipschitz. If we have $P\ll Q$, for bounded and Lipschitz real distribution $P$, the condition (3) is satisfied. If the subadditivity upper bound is minimized, we can minimize $\SKL\PQ{}$ up to $\cO(n)$. For the detailed proof, see \cref{proof:subadditivity-ipm}.

%% file: MRFs.tex
\section{Generalized Subadditivity on MRFs}
\label{sec:mrfs}

The definition of {\em generalized subadditivity} of a statistical divergence with respect to another one over MRFs is the same as in \cref{def:subadditivity}, except that the local neighborhoods are defined as maximal cliques $C\in\cC$ of the MRF. For an alternative definition of subadditivity on MRFs, see \cref{appendix:random-field}.

The clique factorization of MRFs (i.e. $P(x)=\prod_{C\in\cC}\fP^P_C(X_C)$) offers a special method to prove the subadditivity of IPMs on MRFs. Consider the Symmetric KL divergence $\SKL\PQ{}\coloneqq\KL\PQ{}+\KL(Q, P)=\EE_{x\sim P}[\log(P/Q)]-\EE_{x\sim Q}[\log(P/Q)]$. Clique factorization of $P$ and $Q$ decomposes $\SKL\PQ{}$ into $\SKL\PQ{} = \sum_{C\in\cC}(\EE_{x_C\sim P_{X_C}}[\log(\fP^P_C/\fP^Q_C)]-\EE_{x_C\sim Q_{X_C}}[\log(\fP^P_C/\fP^Q_C)])$, where each term in the summation is upper-bounded by an IPM $d_{\cF_C}\PQ{X_C}$ on the clique $C$, as long as $\log(\fP^P_C/\fP^Q_C)\in\cF_C$. This implies the subadditivity of $1$-Wasserstein distance with respect to the Symmetric KL divergence, whenever each $\log(\fP^P_C/\fP^Q_C)$ is Lipschitz continuous; see \cref{proof:wasserstein-mrf} for the proof.

\begin{theorem}
\label{thm:wasserstein-mrf}
    Consider two MRFs $P$, $Q$ with the same factorization. If any of the following  is fulfilled:
    \begin{compactenum}[(1)]
        \item The space $\Sp$ is discrete and finite.
        \item $\log(\fP^P_C/\fP^Q_C)$ are Lipschitz continuous for all $C\in\cC$.
    \end{compactenum}
    Then, the $1$-Wasserstein distance satisfies $\ap$-linear subadditivity with respect to the Symmetric KL Divergence on MRFs, for some constant $\ap>0$ independent of $P$ and $Q$.
\end{theorem}

Using the aforementioned property of Symmetric KL divergence, the subadditivity of neural distances (\cref{thm:subadditivity-ipm}) can be generalized to MRFs; see \cref{proof:subadditivity-ipm-mrf}.

\begin{corollary}
\label{coro:subadditivity-ipm-mrf}
    For two MRFs $P, Q$ on a common graph $G$ and a set of function classes $\{\cF_C|C\in\cC\}$, if all of the three conditions in \cref{thm:subadditivity-ipm} are fulfilled (with condition (3) replaced by: $\log(\fP^P_C/\fP^Q_C)$ are bounded and Lipschitz continuous for all $C\in\cC$), the neural distances induced by $\{\cF_C|C\in\cC\}$ satisfy $\ap$-linear subadditivity with error $\eps$ with respect to the Symmetric KL divergence on MRFs, i.e. $\SKL\PQ{}-\eps\leq\ap\cdot\sum_{C\in\cC} d_{\cF_C}\PQ{X_C}$, where $\ap$ and $\eps$ are constants independent of $P, Q$ and $\{\cF_C|C\in\cC\}$, satisfying $\ap>R(c_{\max}d)$ and $\eps=\cO\left(|\cC|\ap^{-\frac{2}{c_{\max}d+1}}\log\ap\right)$. $|\cC|$ is the number of maximal cliques in $G$ and $R(c_{\max}d)$ is a function that only depends on $c_{\max}=\max\{|C|\big|C\in\cC\}$ (the maximum size of the cliques in $G$) and $d$.
\end{corollary}

%% file: Local.tex
\section{Local Subadditivity}
\label{sec:local}

So far, we have stated and proved the subadditivity or generalized subadditivity of some $\gf$-divergences on Bayes-nets or MRFs. However, many divergences may not enjoy subadditivity property (see such a counter-example of $2$-Wasserstein distance in \cref{appendix:counter-cexp-wasserstein}). It is difficult to formulate a general framework for determining which divergence is subadditive.

In this section, we consider a particular scenario when two distributions $P, Q$ are {\it close} to each other, which can happen after some initial training steps in a GAN. In this case, we are able to determine if an arbitrary $\gf$-divergence satisfies generalized subadditivity on Bayes-nets. We only report our main results here. See \cref{appendix:local} and \cref{appendix:local-exp} for more details and proofs. We consider two notions of ``closeness'' for distributions.

\begin{definition}
\label{def:close-distributions}
    Distributions $P, Q$ are one-sided $\eps$-close for some $0<\eps<1$, if $\forall x\in\Sp\subseteq\RR^{nd}$, $P(x)/Q(x)<1+\eps$. Moreover, $P, Q$ are two-sided $\eps$-close, if $\forall x$, $1-\eps<P(x)/Q(x)<1+\eps$. Note this requires $P\ll\gg Q$.
\end{definition}

We find that most $\gf$-divergences satisfy generalized linear subadditivity when the distributions are one- or two-sided $\eps$-close.

\begin{theorem}
\label{thm:local-linear-subadditivity-two-sided-close}
    An $\gf$-divergence whose $\gf(\cdot)$ is continuous on $(0,\infty)$ and twice differentiable at $1$ with $\gf''(1)>0$ satisfies $\ap$-linear subadditivity, when $P, Q$ are two-sided $\eps(\ap)$-close with $\eps>0$, where $\eps(\ap)$ is a non-increasing function and $\lim_{\eps\downarrow0}\ap=1$.
\end{theorem}

\begin{theorem}
\label{thm:local-linear-subadditivity-one-sided-close}
    An $\gf$-divergence whose $\gf(\cdot)$ is continuous and strictly convex on $(0,\infty)$, twice differentiable at $t=1$, and has finite $\gf(0)=\lim_{t\downarrow0}\gf(t)$, satisfies $\ap$-linear subadditivity, when $P, Q$ are one-sided $\eps(\ap)$-close with $\eps>0$, where $\eps(\ap)$ is a non-increasing function and $\lim_{\eps\downarrow0}\ap>0$.
\end{theorem}

%% file: GANs.tex
\begin{figure}[t]
    \centering
    \includegraphics[width=.95\linewidth]{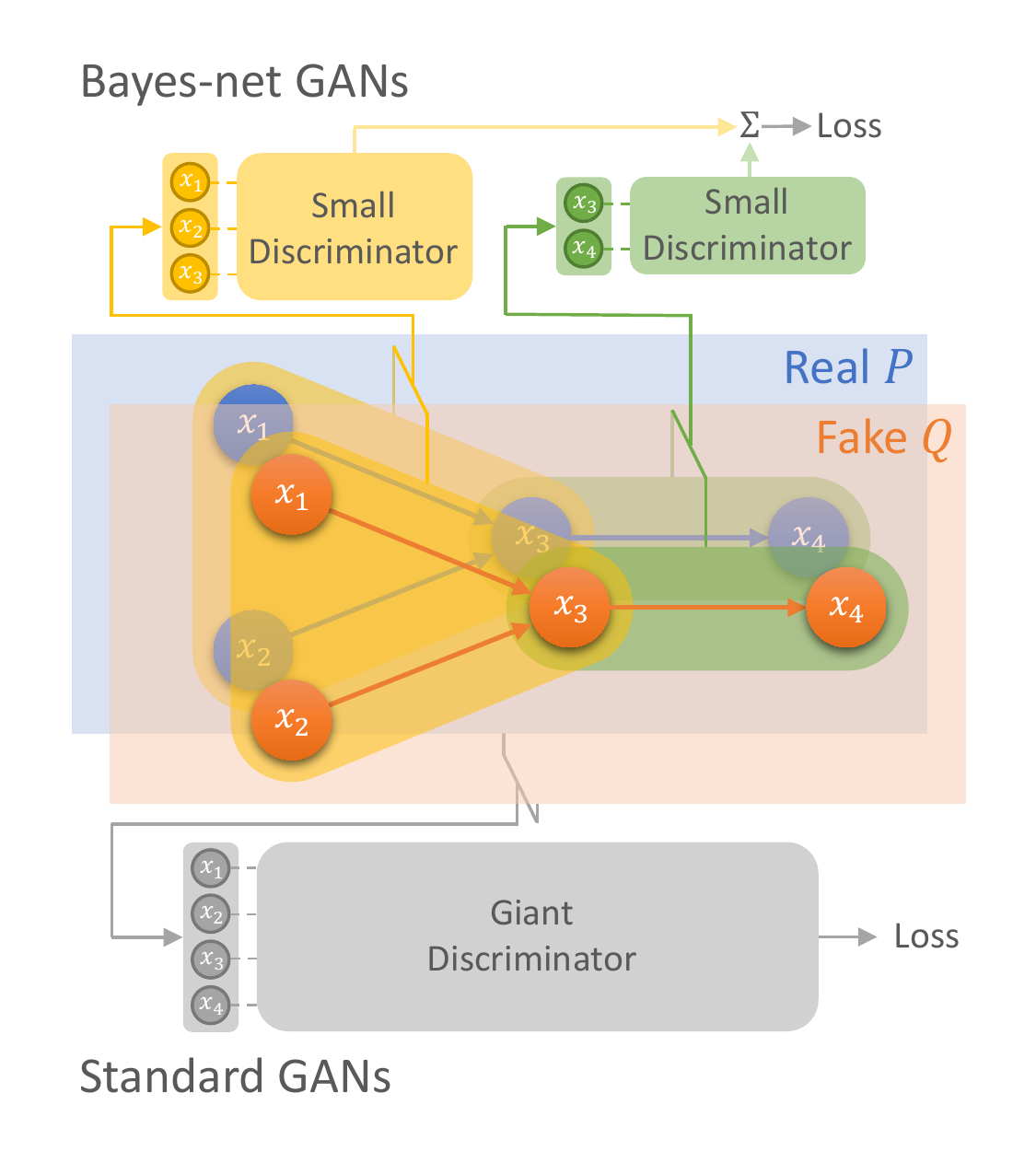}
    \caption{Conceptual diagram of the Bayes-net GANs with local discriminators compared with the standard GANs.}
    \label{fig:gans-conceptual}
\end{figure}

\section{GANs with Bayes-Nets/MRFs}
\label{sec:gans}

Our proposed model-based GAN minimizes the generalized subadditivity upper bound of a divergence measure $\D$. For example, a Bayes-net GAN\footnote{A model-based GAN on MRFs can be formulated similarly.} is formulated as the following optimization problem: 
\begin{align*}
\min_{Q}\quad \sum_{i=1}^n\D\PQ{X_i\cup X_{\Pi_i}}.   
\end{align*}
Similar to a standard GAN~\citep{goodfellow2014generative, arjovsky2017wasserstein}, the generated distribution $Q$ is characterized as $G(Z)$ where $G(.)$ is the generator function and $Z$ is a normal distribution. Note that the discriminator is implicit in the definition of the $\D$ (Figure \ref{fig:gans-conceptual}). Since local neighborhoods are often significantly smaller than the entire graph, our proposed model-based GAN enjoys improved computational and statistical properties compared to a model-free GAN that uses a global discriminator targeting the entire graph.

%% file: Experiments.tex
\begin{figure*}[t]
    \centering
    \begin{subfigure}{.25\linewidth}
        \centering
        \includegraphics[width=\linewidth]{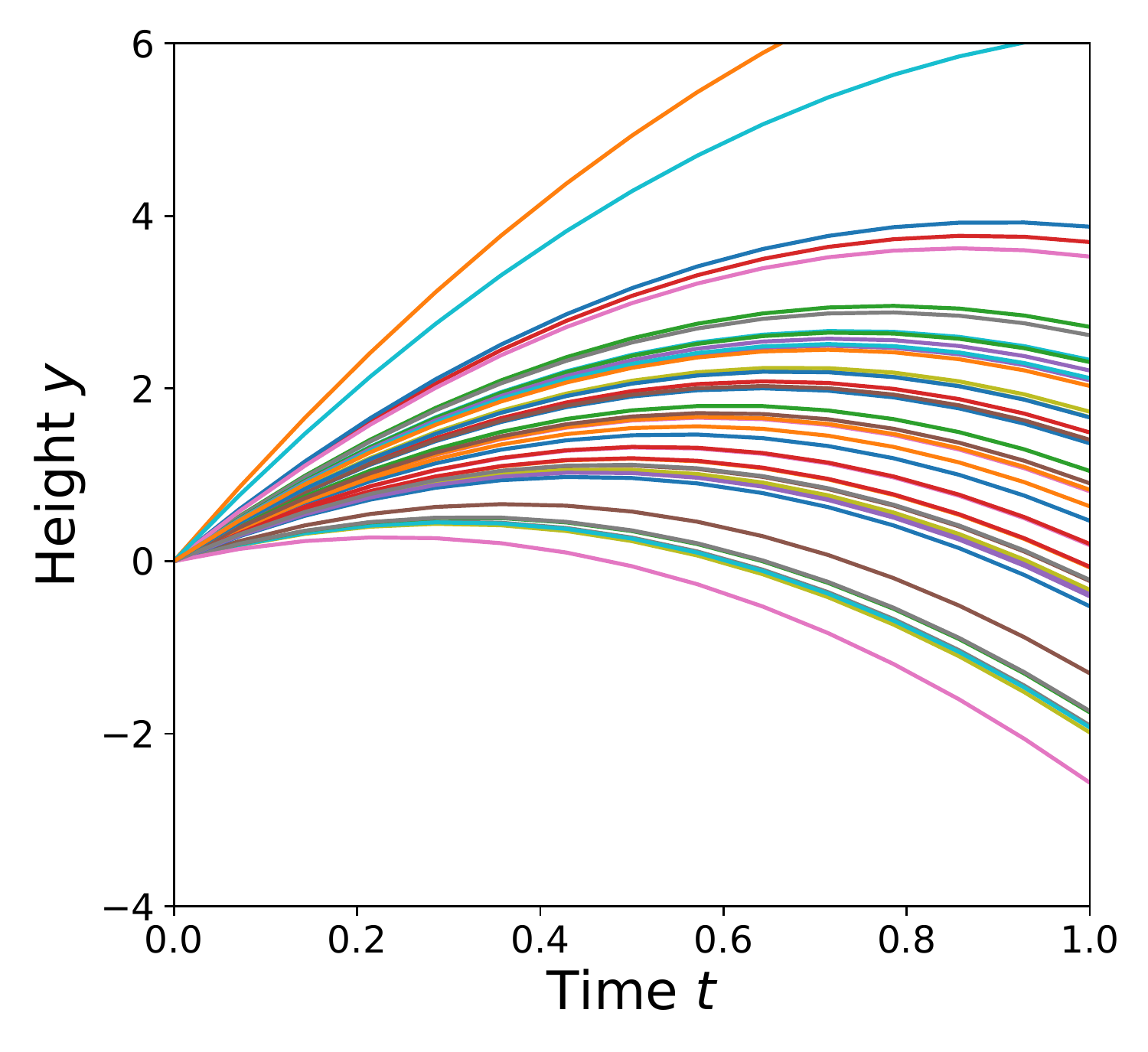}
        \caption{Ground truth}
        \label{fig:throwing-ball-samples-gt}
    \end{subfigure}%
    \begin{subfigure}{.25\linewidth}
        \centering
        \includegraphics[width=\linewidth]{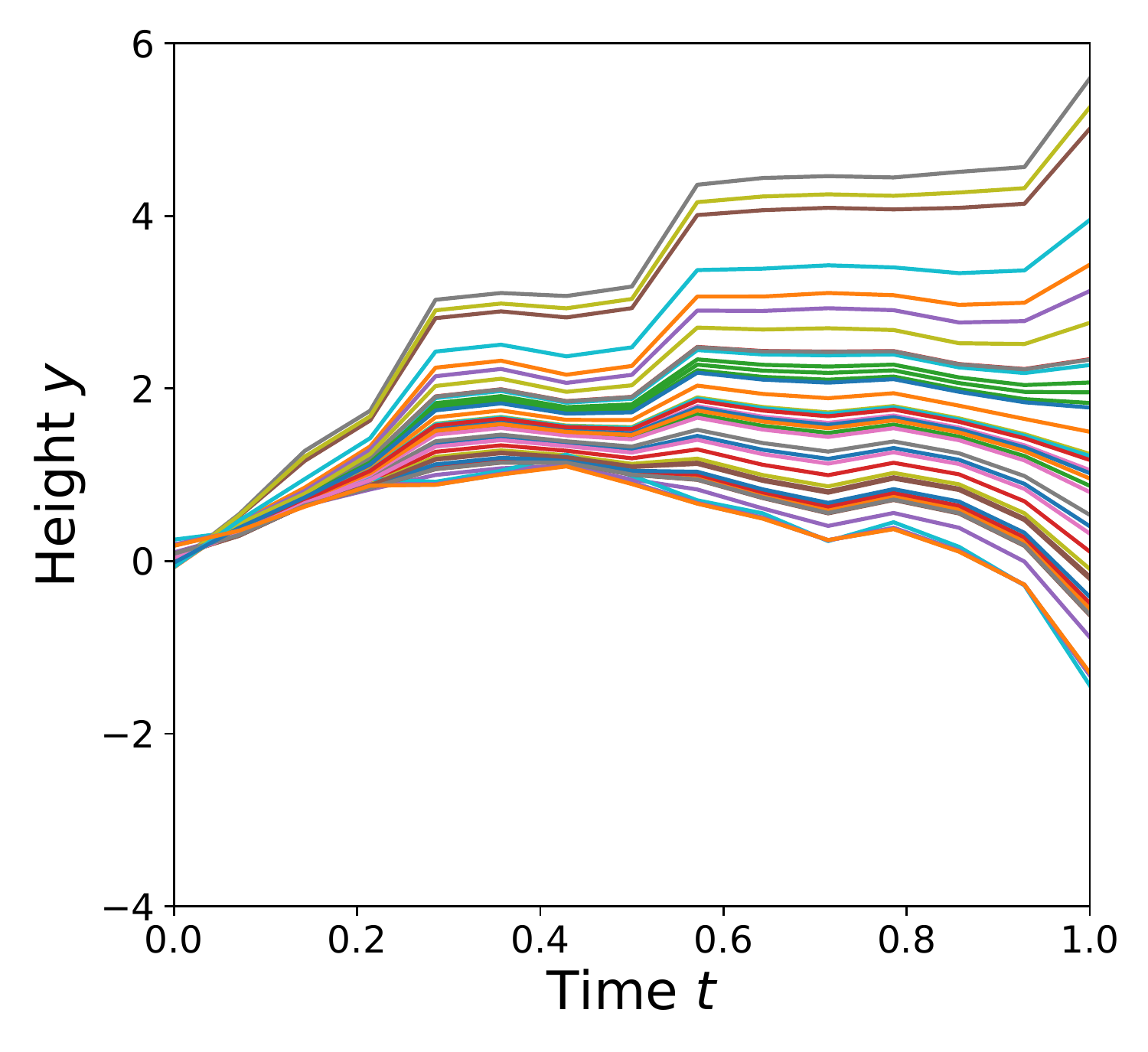}
        \caption{Local-width $2$}
        \label{fig:throwing-ball-samples-corrlength2}
    \end{subfigure}%
    \begin{subfigure}{.25\linewidth}
        \centering
        \includegraphics[width=\linewidth]{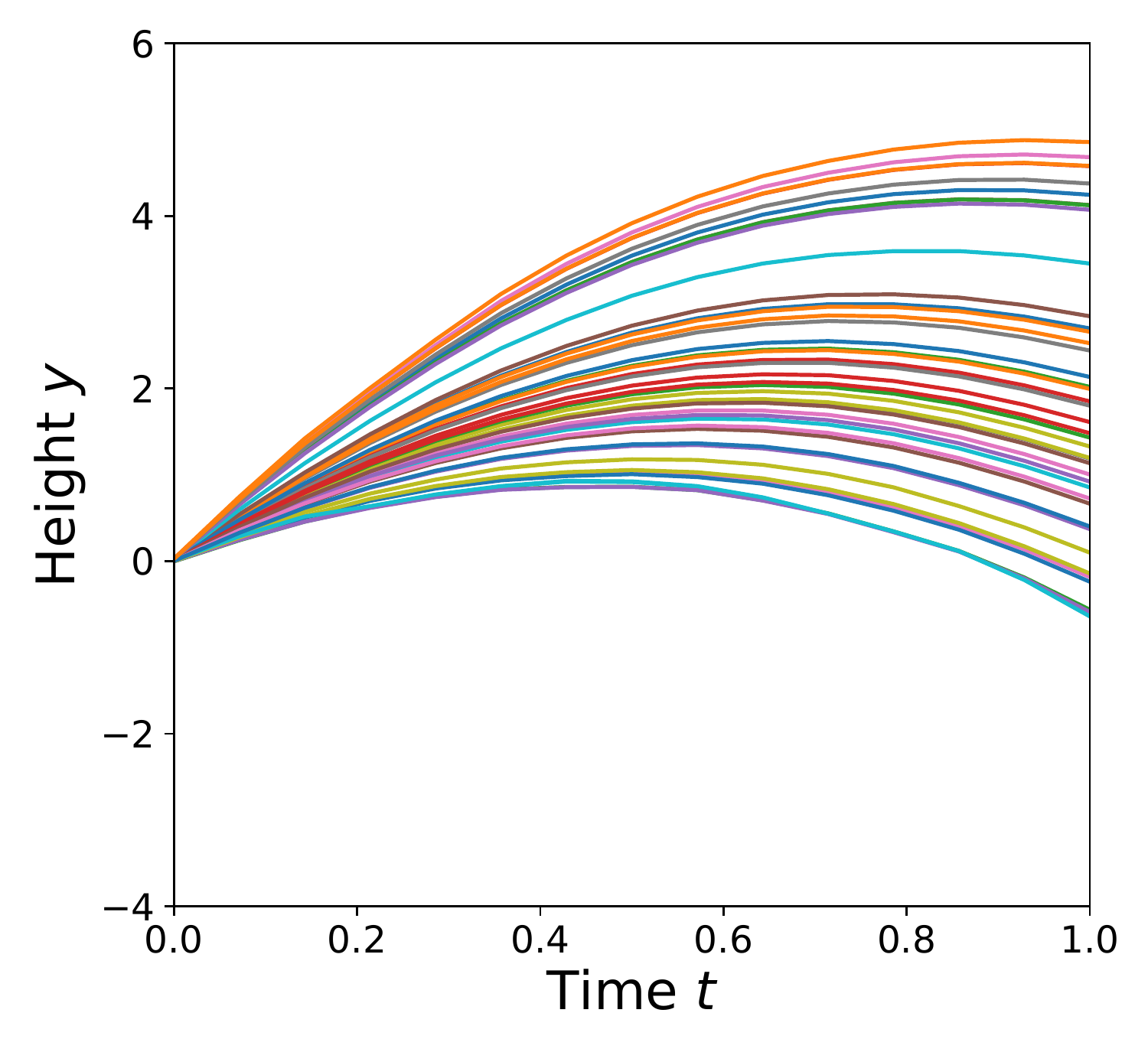}
        \caption{Local-width $3$}
        \label{fig:throwing-ball-samples-corrlength3}
    \end{subfigure}%
    \begin{subfigure}{.25\linewidth}
        \centering
        \includegraphics[width=\linewidth]{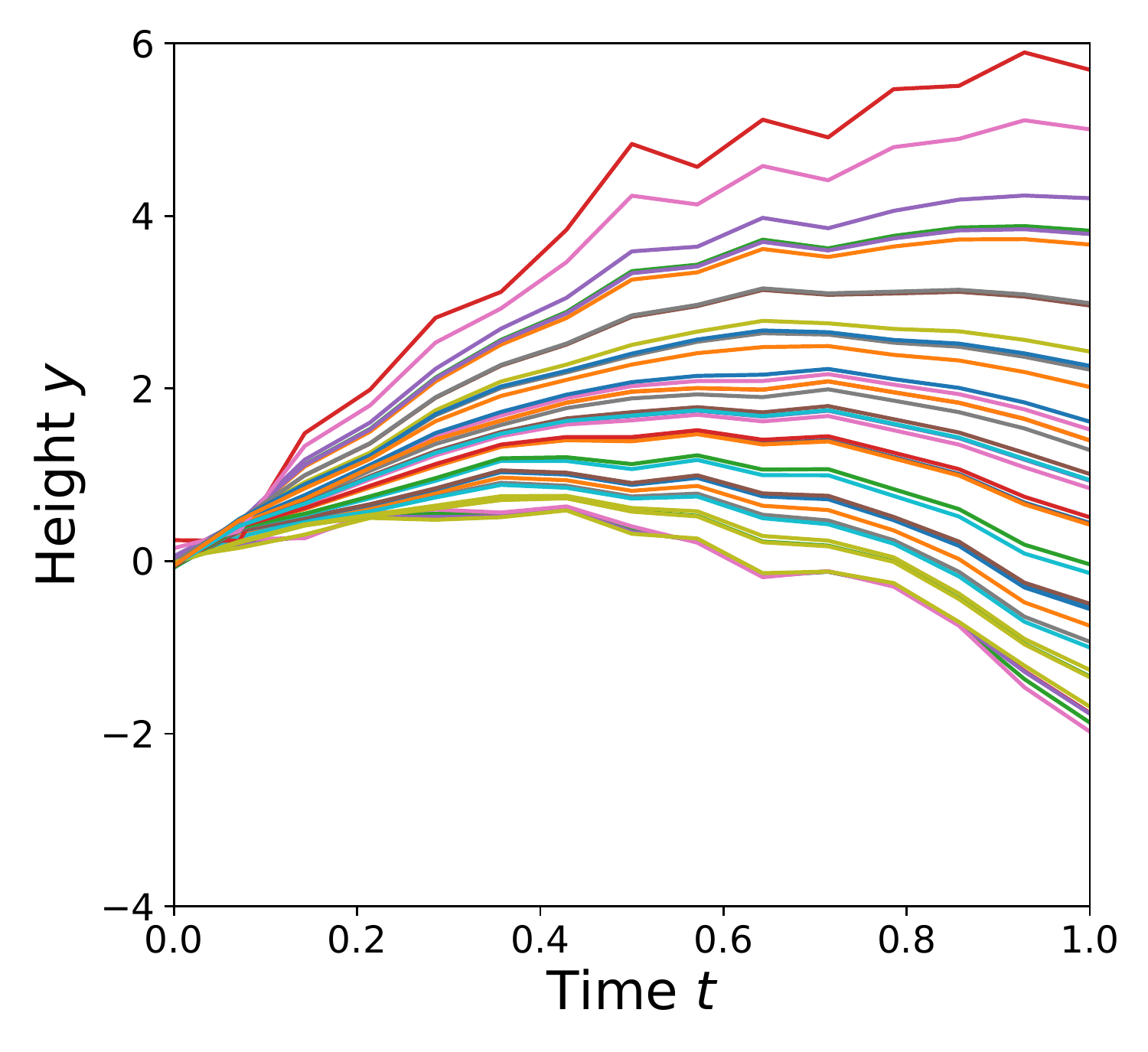}
        \caption{Standard GAN}
        \label{fig:throwing-ball-samples-corrlength15}
    \end{subfigure}%
    \caption{GAN-generated ball throwing trajectories by \subref{fig:throwing-ball-samples-corrlength2} the Bayes-net GAN (ours) with {\it localization width} $2$ (the width of the local neighborhoods that the discriminators test on), \subref{fig:throwing-ball-samples-corrlength3} the Bayes-net GAN with local-width $3$, and \subref{fig:throwing-ball-samples-corrlength15} the standard GAN.}
    \label{fig:throwing-ball-samples}
\end{figure*}

\begin{figure}[t]
    \centering
    \includegraphics[width=\linewidth]{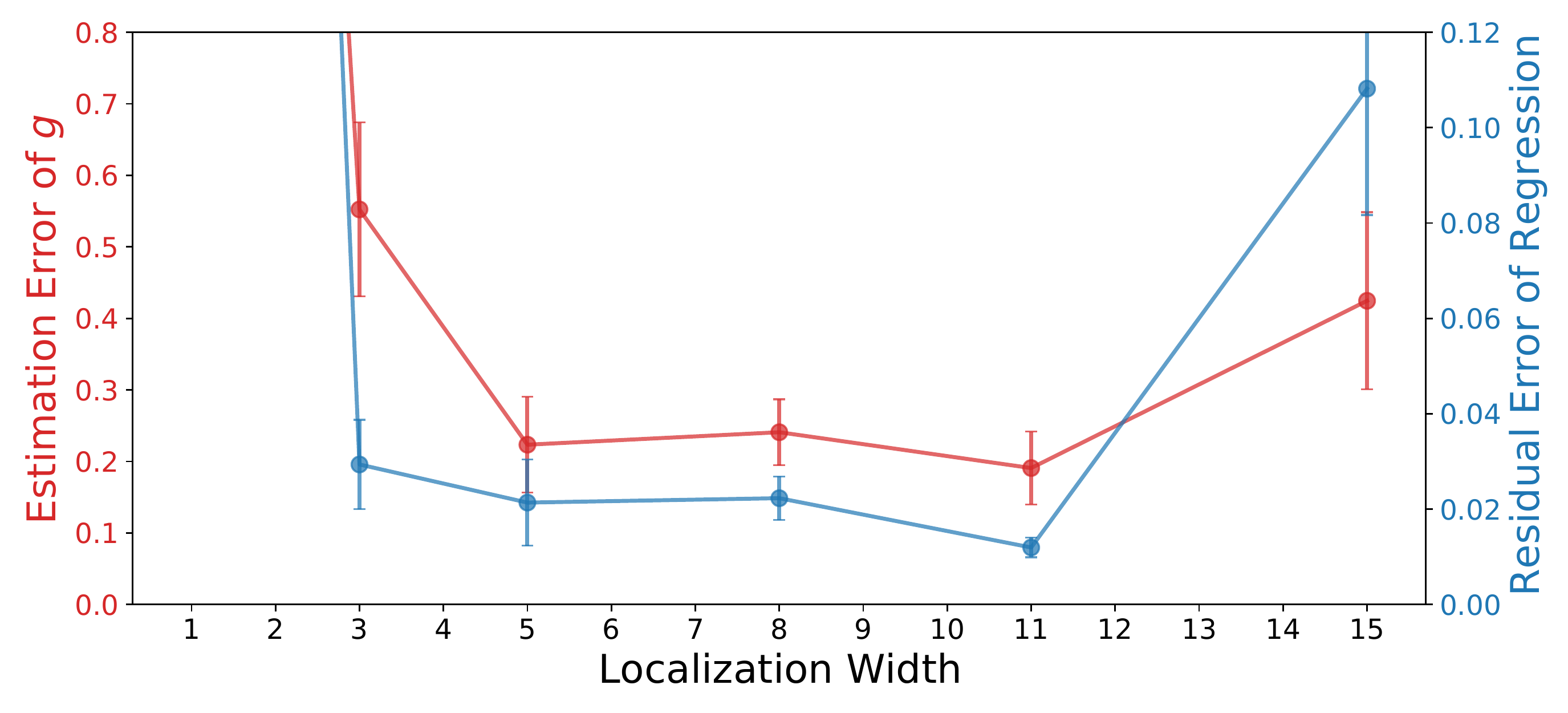}
    \caption{Estimation errors of gravitational acceleration $g$ and residual errors of degree-$2$ polynomial regression on the generated trajectories with varying localization width.}
    \label{fig:throwing-ball-metrics}
\end{figure}

\section{Experiments}
\label{sec:experiments}

In this section, we provide experimental results demonstrating the benefits of exploiting the underlying Bayes-net or MRF structure of the data in the design of model-based GANs. In our experiments, we consider a synthetic {\it ball throwing trajectory dataset} as well as two real Bayes-net datasets: {\it the EARTHQUAKE dataset}~\citep{korb2010bayesian} and {\it the CHILD dataset}~\citep{spiegelhalter1992learning}. Unless otherwise stated, the Wasserstein GAN~\citep{arjovsky2017wasserstein} with gradient penalty~\citep{gulrajani2017improved} is used in the experiments. Detailed experimental setups (including network architectures and hyper-parameters) can be found in \cref{appendix:experimental-setups}. The experiments on MRF datasets and more experimental findings on Bayes-nets including the sensitivity analysis of Bayes-net GANs are reported in \cref{appendix:more-experiment-results}.

\subsection{Synthetic Ball throwing trajectories}
\label{subsec:experiment-timeseries}

In this section, we consider a simple synthetic dataset that consists of single-variate time-series data $(y_1, \ldots, y_{15})$ representing the $y$-coordinates of ball throwing trajectories lasting 1 second, where $y_t = v_0*(t/15)-g(t/15)^2/2$. $v_0$ is a Gaussian random variable and $g=9.8$ is the gravitational acceleration. These trajectories are Bayes-nets, where the underlying DAG has the following structure: each node $t\in\{1, \ldots, 15\}$ has two parents, $(t-1)$ and $(t-2)$ (if they exist). This is because, given $g$ and without known $v_0$, one can determine $y_t$ from $y_{t-1}$ and $y_{t-2}$.

We train two types of GANs to generate ``ball throwing trajectories'': (1) Bayes-net GANs with local discriminators where each discriminator has a certain {\it time localization width} and (2) a standard GAN with one global discriminator. From the underlying physics of this dataset, we know that a proper discriminator design should have at least a localization width of $3$ since one needs at least three consecutive coordinates $y_{t-2}, y_{t-1}, y_t$ to estimate the gravitational acceleration $g$. Thus, from the theory, a GAN trained using local discriminators with a localization width of $2$ should not be able to generate high-quality samples. This is in fact verified by our experiments. In \cref{fig:throwing-ball-samples}, we see samples generated by the local-width $3$ GAN (\cref{fig:throwing-ball-samples-corrlength3}) are visually very similar to the ground truth trajectories (\cref{fig:throwing-ball-samples-gt}), while samples generated by the local-width $2$ GAN demonstrate poor quality.

\begin{table*}
    \centering
    \resizebox{\textwidth}{!}{%
        \begin{tabular}{cc|cccc}
        \hline
        Dataset & GAN used & \begin{tabular}[c]{@{}c@{}}{\bf Energy Stats.} ($\times 10^{-2}$)\\ {\small (smaller is better)}\end{tabular} & \begin{tabular}[c]{@{}c@{}}{\bf Detection AUC}\\ {\small (smaller is better)}\end{tabular} & \begin{tabular}[c]{@{}c@{}}{\bf Rel. BIC} ($\times 10^2$)\\ {\small (larger is better)}\end{tabular} & \begin{tabular}[c]{@{}c@{}}{\bf Rel. GED}\\ {\small (smaller is better)}\end{tabular} \\ \toprule
        \multirow{2}{*}{{\it EARTHQUAKE}} & Bayes-net (ours) & $0.24\pm0.04$ & $0.523\pm0.005$ & $+1.68\pm0.17$ & $0.4\pm0.7$ \\ \cline{2-6} 
        & Standard & $1.72\pm0.08$ & $0.564\pm0.012$ & $-4.30\pm0.21$ & $5.6\pm0.7$ \\ \midrule
        \multirow{2}{*}{{\it CHILD}} & Bayes-net (ours) & $2.37\pm0.10$ & $0.644\pm0.008$ & $+0.6\pm1.5$ & $9\pm4$ \\ \cline{2-6} 
        & Standard & $4.40\pm0.22$ & $0.689\pm0.019$ & $-7.1\pm2.0$ & $24\pm8$ \\ \bottomrule
        \end{tabular}%
    }
    \captionof{table}{Quality metrics of samples generated by the standard and Bayes-net GANs trained on the Bayes-nets.}
    \label{tab:bayesnets-scores}
\end{table*}

\begin{figure*}[t]
    \centering
    \begin{subfigure}[t]{.21\linewidth}
        \centering
        \includegraphics[width=\linewidth]{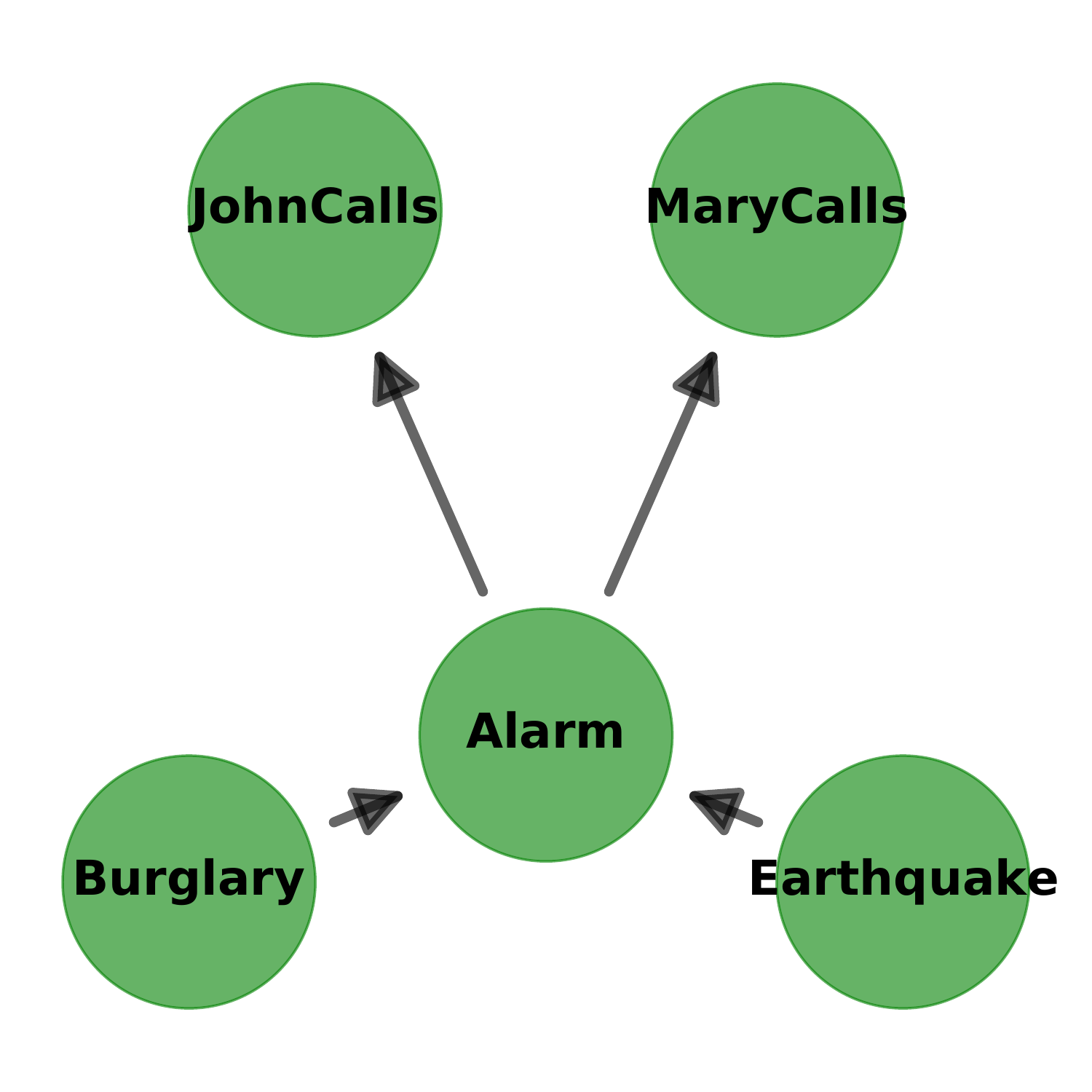}
        \caption{Ground truth}
        \label{fig:bayesnet-structure-ground-truth}
    \end{subfigure}%
    \hspace{.4em}
    \begin{subfigure}[t]{.21\linewidth}
        \centering
        \includegraphics[width=\linewidth]{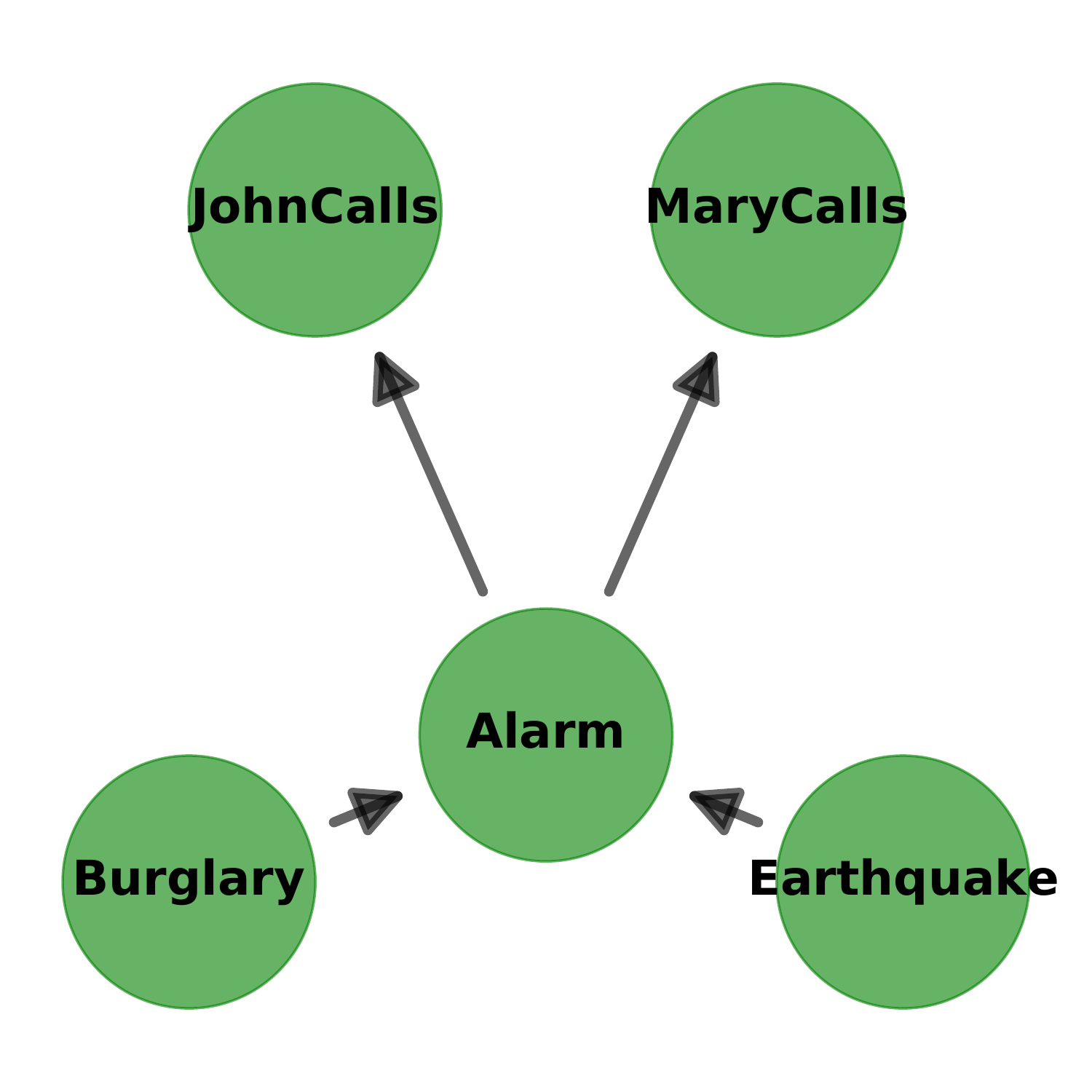}
        \caption{Direct Prediction}
        \label{fig:bayesnet-structure-directly-sampled}
    \end{subfigure}%
    \hspace{.4em}
    \begin{subfigure}[t]{.21\linewidth}
        \centering
        \includegraphics[width=\linewidth]{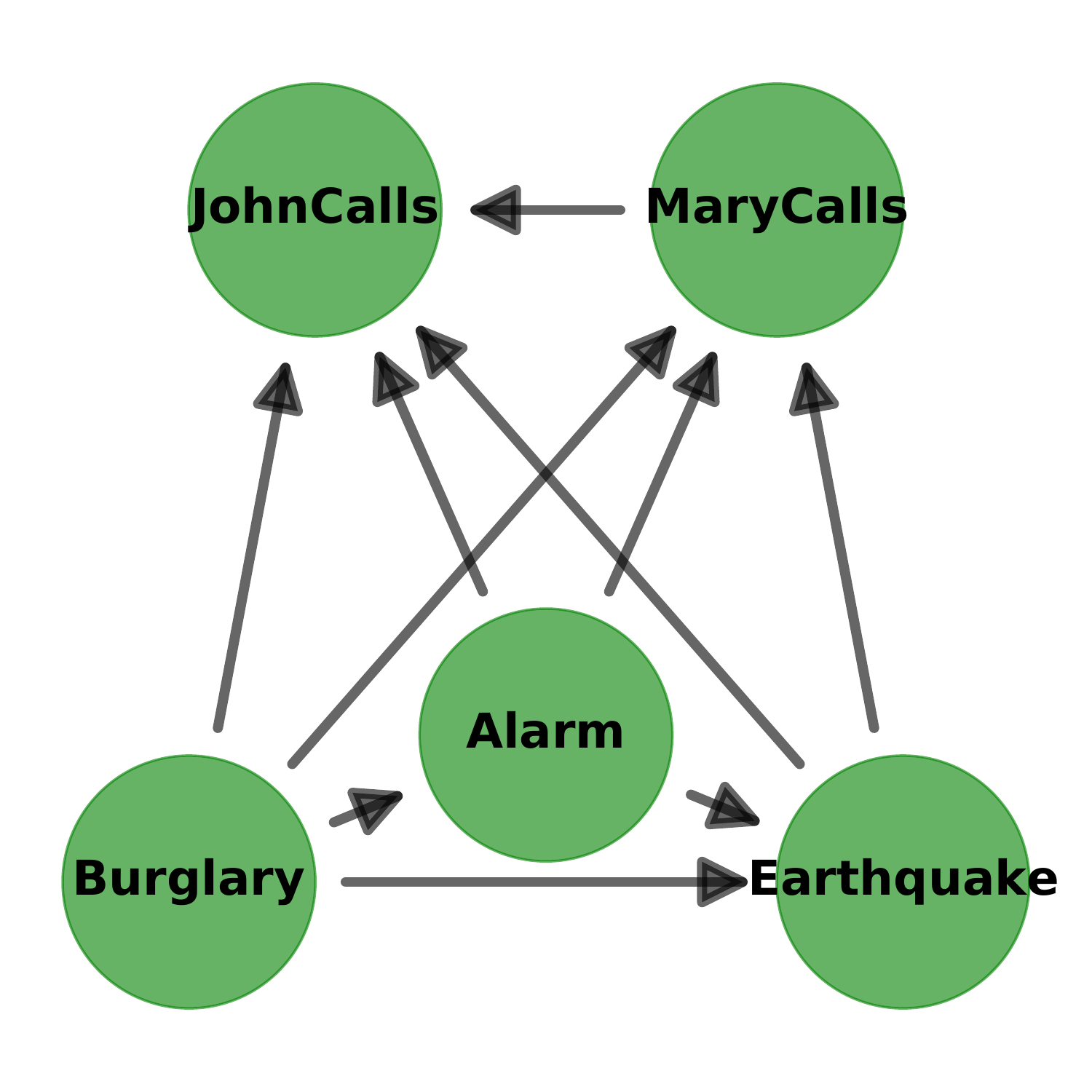}
        \caption{Standard GAN}
        \label{fig:bayesnet-structure-standard-gan}
    \end{subfigure}%
    \hspace{.4em}
    \begin{subfigure}[t]{.21\linewidth}
        \centering
        \includegraphics[width=\linewidth]{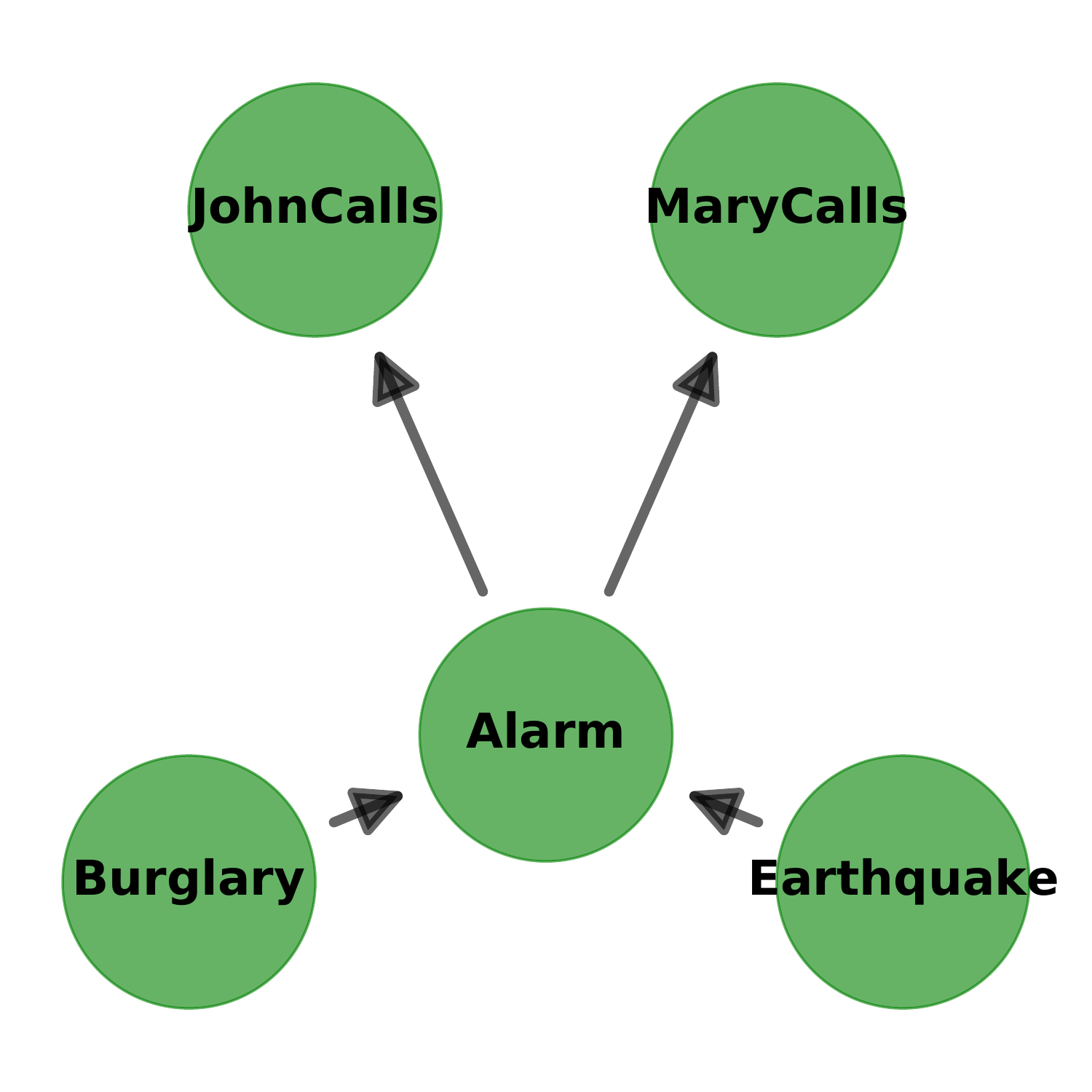}
        \caption{Bayes-net GAN (ours)}
        \label{fig:bayesnet-structure-local-true}
    \end{subfigure}%
    \caption{Causal structures predicted from \subref{fig:bayesnet-structure-directly-sampled} the observed data, \subref{fig:bayesnet-structure-standard-gan} the data generated by the standard GAN, and \subref{fig:bayesnet-structure-local-true} the data generated by the Bayes-net GAN (ours).}
    \label{fig:bayesnet-structure}
\end{figure*}

\begin{figure*}[t]
    \centering
    \includegraphics[width=\linewidth]{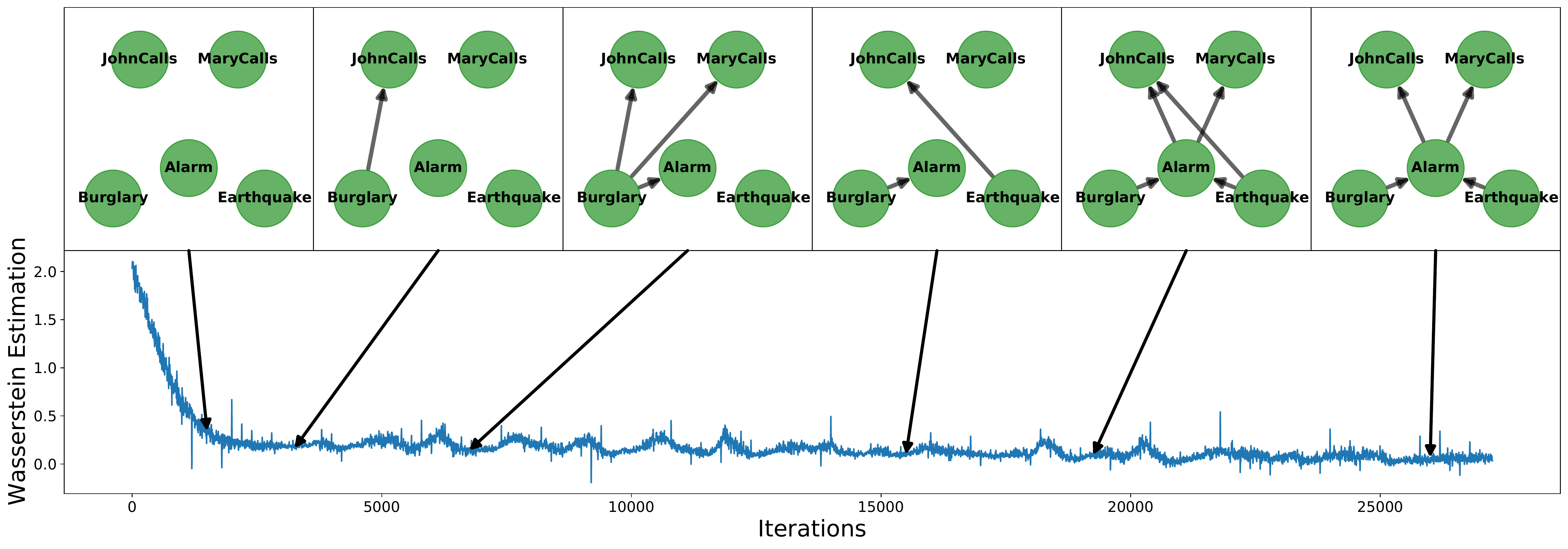}
    \caption{Causal structures predicted from the data generated by the Bayes-net GAN at different stages of training and the Wasserstein loss curve.}
    \label{fig:bayesnet-wasserstein-structure}
\end{figure*}

\begin{figure}[t]
    \centering
    \includegraphics[width=\linewidth]{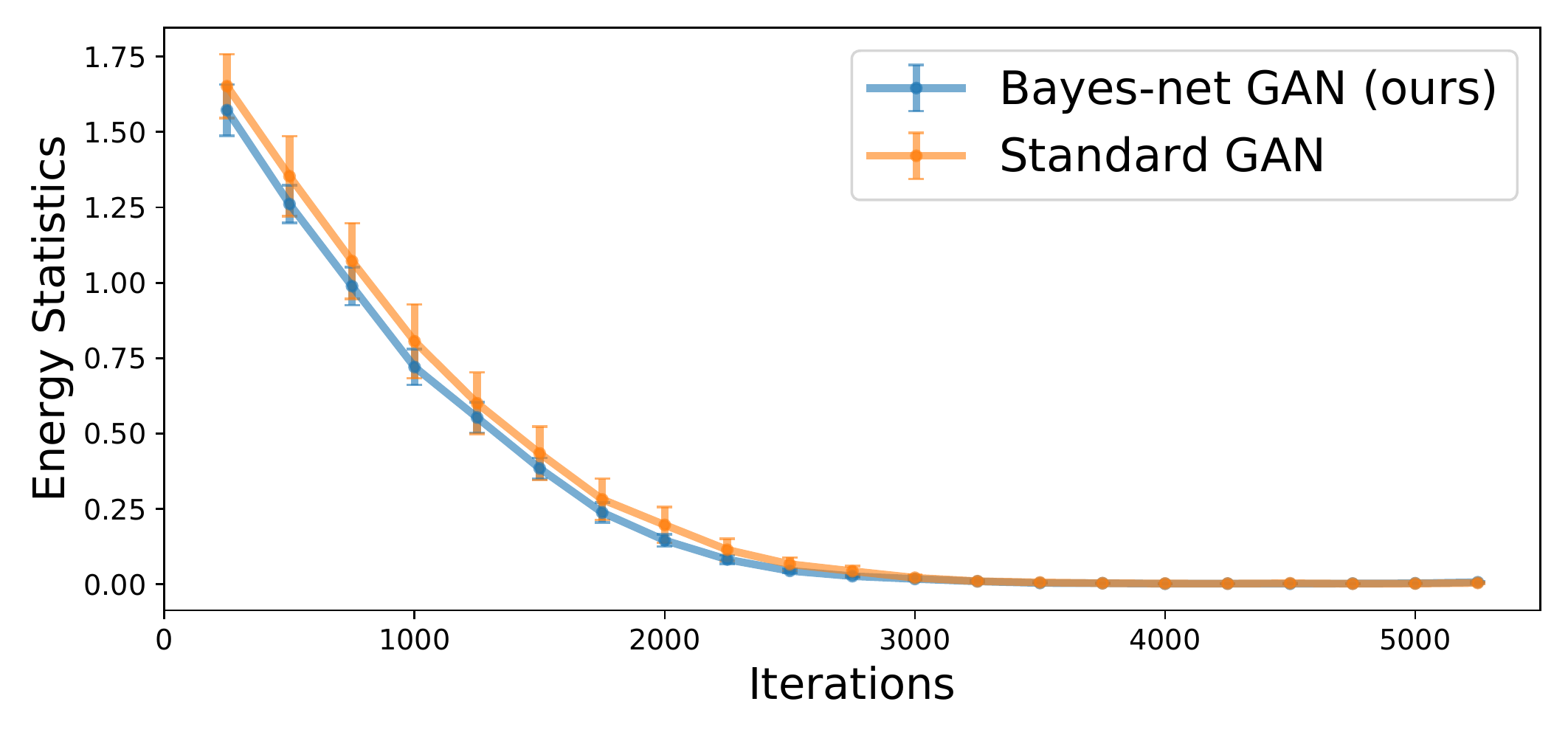}
    \caption{Energy statistics between generated and observed samples at different stages of training.}
    \label{fig:bayesnet-loss-energy-statistics}
\end{figure}

Note that increasing the localization width of the discriminators enhances their discrimination power, but at the same time, it increases the model complexity, which can cause statistical and computational issues during the training. To understand this trade-off, we progressively increase the localization width from 3 to 15, obtaining one giant discriminator at the end. The quality of generated trajectories from the standard GAN (corresponding to the giant discriminator) is, in fact, worse (\cref{fig:throwing-ball-samples-corrlength15}).

In \cref{fig:throwing-ball-metrics}, we compare the estimation errors of the gravitational acceleration $g$ and the residual errors of degree-$2$ polynomial regression (which evaluate the ``smoothness'' of generated trajectories) among GANs with different localization widths. Interestingly, the curves of both metrics demonstrate a $U$-shaped behavior indicating that there is an optimal localization width balancing between the discrimination power and the model complexity and its resulting statistical/computational burden.

\subsection{Real Bayes-nets}
\label{subsec:experiment-bayesnet}
Next, we consider two real Bayes-net datasets:
\begin{inparaenum}[(1)]
    \item {\it the EARTHQUAKE dataset} which is a small Bayes-net with $5$ nodes and $4$ edges characterizing the alarm system against burglary which can get occasionally set off by an earthquake~\citep{korb2010bayesian}, and
    \item {\it the CHILD dataset} which is a Bayes-net for diagnosing congenital heart disease in a newborn ``blue baby''~\citep{spiegelhalter1992learning}, with $20$ nodes and $25$ edges.
\end{inparaenum}
The underlying Bayes-nets of both datasets are known. We first generate samples from the Bayes-nets, then train both standard GANs and Bayes-net GANs (using the subadditivity upper-bound as objectives) on them (Since all the features are categorical, we use {\it Gumbel-Softmax}~\citep{jang2016categorical} as a differentiable approximation to the {\it Softmax} function in the generator; see \cref{appendix:experimental-setups}.)

If a GAN learns the Bayes-net well, it should learn both the joint distribution and the conditional dependencies. We evaluate the quality of the generated samples by four scores:
\begin{itemize}[leftmargin=*, itemsep=1pt, topsep=0pt]
	\item {\bf Energy Statistics} measuring how close the real and fake empirical distributions based on a statistical potential energy (a function of distances between observables)~\citep{szekely2013energy},
	\item {\bf Detection AUC}: AUC scores of binary classifiers trained to distinguish fake samples from real ones,
	\item {\bf Relative BIC}: the Bayesian information criterion of fake samples (a log-likelihood score with an additional penalty for the network complexity) \citep{koller2009probabilistic} subtracted by the BIC of real ones, and
	\item {\bf Relative GED}: the graph editing distance between the DAGs predicted from the fake and real samples by a greedy search starting from the ground truth DAG.
\end{itemize}
The first two metrics characterize the similarity between the joint distributions, while the last two evaluate how accurately the causal structure is learned.

We find that the Bayes-net GAN using the ground truth causal graph consistently outperforms the model-free standard GAN on all four quality metrics (\cref{tab:bayesnets-scores}). For Bayes-net GANs, the relative BIC scores (the second last column) are positive, i.e., the BIC of samples generated by the Bayes-net GANs is even higher than the BIC of observed data. Because the Bayes-net GANs are designed to conveniently capture the ground truth causal dependencies (compared to the other correlations), the likelihood of the ground truth causal structure can further increase. On {\it the EARTHQUAKE dataset}, we can usually recover the true causal graph from the data generated by the Bayes-net GAN (\cref{fig:bayesnet-structure-local-true}). This is not the case if we use standard GANs (\cref{fig:bayesnet-structure-standard-gan}), where any pair of nodes are directly dependent on each other. In this regard, we conclude the standard GANs cannot efficiently capture the conditional independence relationships among variables. 

Next, we study how a Bayes-net GAN learns the causal structure during the training (\cref{fig:bayesnet-wasserstein-structure}). In general, discrete Bayes-nets are multi-modals. The Bayes-net GAN learns some strong conditional dependencies at first, e.g. ``Burglary'' leads to ``JohnCalls'' in the second snapshot, although it is not a direct dependence (in fact, ``Burglary'' triggers ``Alarm'', then ``JohnCalls''). After some training, the dependence relation is further specified, and the edge (``Burglary''$\to$``JohnCalls'') is replaced by a pair of new edges, (``Burglary''$\to$``Alarm'') and (``Alarm''$\to$``JohnCalls'') in the second last snapshot. During training, we rarely observe that the Bayes-net GAN captures any non-existing dependencies (e.g. ``Earthquake'' and ``Burglary''). However, this happens often for standard GANs; see \cref{fig:bayesnet-structure-standard-gan} for an 
example.

The success of learning causal independence structures also simplifies the task of learning joint distribution. Without changing any setup or hyper-parameters, replacing the discriminator with a set of local discriminators brings a performance gain on the first two scores as well (\cref{tab:bayesnets-scores}). Moreover, Bayes-net GANs are computationally efficient when the Bayes-nets are not very large. On average, they converge faster than the standard GAN on Bayes-nets; see \cref{fig:bayesnet-loss-energy-statistics} for the averaged curves of energy statistics on {\it the EARTHQUAKE dataset}. These results highlight the statistical and computational benefits of our principled design of Bayes-net GANs.

%% file: Proofs.tex
\section{Proofs}
\label{appendix:proofs}

\subsection{Proof of \texorpdfstring{\cref{thm:subadditivity-markov-chain}}{Theorem \ref{thm:subadditivity-markov-chain}}}
\label{proof:subadditivity-markov-chain}

\begin{proof}
    The theorem is implicit in \citep{daskalakis2017square}. For completeness, we provide a full argument here.
    
    For a pair of Bayes-nets $P$ and $Q$ with respect to a Directed Acyclic Graph (DAG) $G$, consider the topological ordering $(1, \cdots, n)$ of the nodes of $G$. Consistent with the topological ordering, consider the following Markov Chain on super-nodes: $X_{\{1,\cdots,n-1\}\setminus\Pi_n}\to X_{\Pi_n}\to X_n$, where $\Pi_n$ is the set of parents of node $n$ and $\Pi_n\subseteq\{1,\cdots,n-1\}$. We distinguish three cases:
    \begin{enumerate}
        \item $\Pi_n\neq\varnothing$ and $\Pi_n\subsetneqq \{1,\cdots,n-1\}$: In this case, we apply the subadditivity property of $\D$ with respect to Markov Chains to obtain
        $\D\PQ{}\leq\D\PQ{\cup_{i=1}^{n-1}X_i}+\D\PQ{X_{\Pi_n}\cup X_n}$.
        \item $\Pi_n=\{1,\cdots,n-1\}$: In this case, it is trivial that $\D\PQ{}\equiv\D\PQ{X_{\Pi_n}\cup X_n}\leq\D\PQ{\cup_{i=1}^{n-1}X_i}+\D\PQ{X_{\Pi_n}\cup X_n}$.
        \item $\Pi_n=\varnothing$: In this case, $X_n$ is independent from $(X_1,\ldots,X_{n-1})$ in both Bayes-nets. Thus we apply the subadditivity of $\D$ with respect to product measures to obtain $\D\PQ{}\leq \D\PQ{\cup_{i=1}^{n-1}X_i}+\D(P_{X_n},Q_{X_n})\equiv\D\PQ{\cup_{i=1}^{n-1}X_i}+\D\PQ{X_{\Pi_n}\cup X_n}$.
    \end{enumerate}
    We proceed by induction. For each inductive step $k=1,\cdots,n-2$, we consider the following Markov Chain on super-nodes: $X_{\{1,\cdots,n-k-1\}\setminus\Pi_{n-k}}\to X_{\Pi_{n-k}}\to X_{n-k}$. No matter what $\Pi_{n-k}$ is, we always have: $\D\PQ{\cup_{i=1}^{n-k}X_i}\leq\D\PQ{\cup_{i=1}^{n-k-1}X_i}+\D\PQ{X_{\Pi_{n-k}}\cup X_{n-k}}$. In the end of the induction, we obtain: $\D(P, Q)\leq\D\PQ{X_1}+\sum_{i=2}^n\D\PQ{\Pi_{i}\cup X_{i}}\equiv\sum_{i=1}^n\D\PQ{\Pi_{i}\cup X_{i}}$, since $\Pi_1\equiv\varnothing$. The subadditivity of $\D$ on Bayes-nets is proved.
\end{proof}

\subsection{Proof of \texorpdfstring{\cref{thm:subadditivity-sh}}{Theorem \ref{thm:subadditivity-sh}}}
\label{proof:subadditivity-sh}

\begin{proof}
    The subadditivity of squared Hellinger distance is proved in Theorem 2.1 of \citep{daskalakis2017square}. Here, we repeat the proof for completeness.
    
    Given \cref{thm:subadditivity-markov-chain}, we only need to show the following:
    \begin{enumerate}
        \item For two Markov Chains $P, Q$ on variables $X\to Y\to Z$, it holds that $\SH\PQ{XYZ}\leq \SH\PQ{XY}+\SH\PQ{YZ}$.
        \item For two product measures $P, Q$ on variables $X, Y$, it holds that $\SH\PQ{XY}\leq \SH\PQ{X}+\SH\PQ{Y}$.
    \end{enumerate} 
    We first show the subadditivity with respect to Markov Chains. Using the Markov property, we know $P_{XYZ}=P_{Z|XY}P_{XY}=P_{Z|Y}P_{XY}$ (and the same holds for $Q$), thus,
    \[
        \begin{split}
            &\SH\PQ{XYZ}\\
            &=1-\int\sqrt{P_{XYZ}Q_{XYZ}}\dx\dy\dz\\
            &=1-\int\sqrt{P_{XY}Q_{XY}}\left(\int\sqrt{P_{Z|Y}Q_{Z|Y}}\dz\right)\dx\dy\\
            &=1-\int\frac12(P_{Y}+Q_{Y})\left(\int\sqrt{P_{Z|Y}Q_{Z|Y}}\dz\right)\dy+\int\frac12\left(\sqrt{P_{XY}}-\sqrt{Q_{XY}}\right)^2\left(\int\sqrt{P_{Z|Y}Q_{Z|Y}}\dz\right)\dx\dy\\
        \end{split}
    \]
    Since all densities are non-negative, we have $\sqrt{P_{Y}Q_{Y}}\leq\frac12\left(P_{Y}+Q_{Y}\right)$ and $\sqrt{P_{Z|Y}Q_{Z|Y}}\leq\frac12\left(P_{Z|Y}+Q_{Z|Y}\right)$ point-wisely. Thus,
    \[
        \begin{split}
            &\SH\PQ{XYZ}\\
            &\leq1-\int\sqrt{P_{Y}Q_{Y}}\left(\int\sqrt{P_{Z|Y}Q_{Z|Y}}\dz\right)\dy+\int\frac12\left(\sqrt{P_{XY}}-\sqrt{Q_{XY}}\right)^2\left(\int\frac12\left(P_{Z|Y}+Q_{Z|Y}\right)\dz\right)\dx\dy\\
            &=\left(1-\int\sqrt{P_{YZ}Q_{YZ}}\dy\dz\right)+\frac12\int\left(\sqrt{P_{XY}}-\sqrt{Q_{XY}}\right)^2\dx\dy\\
            &=\SH\PQ{XY}+\SH\PQ{YZ}
        \end{split}
    \]
    
    It remains to show the subadditivity with respect to product measures. If $P, Q$ are product measures over $X, Y$, then $P_{XY}=P_XP_Y$ and $Q_{XY}=Q_XQ_Y$. Since all densities are non-negative, we have  $\sqrt{P_{Y}Q_{Y}}\leq\frac12\left(P_{Y}+Q_{Y}\right)$ point-wise. Hence,
    \[
        \begin{split}
            &\SH\PQ{XY}\\
            &=1-\int\sqrt{P_{XY}Q_{XY}}\dx\dy\\
            &=1-\int\sqrt{P_{X}Q_{X}}\left(\int\sqrt{P_{Y}Q_{Y}}\dy\right)\dx\\
            &=1-\int\frac12(P_{X}+Q_{X})\left(\int\sqrt{P_{Y}Q_{Y}}\dy\right)\dx+\int\frac12\left(\sqrt{P_{X}}-\sqrt{Q_{X}}\right)^2\left(\int\sqrt{P_{Y}Q_{Y}}\dy\right)\dx\\
            &\leq1-\left(\int\frac12(P_{X}+Q_{X})\dx\right)\left(\int\sqrt{P_{Y}Q_{Y}}\dy\right)+\int\frac12\left(\sqrt{P_{X}}-\sqrt{Q_{X}}\right)^2\left(\int\frac12\left(P_{Y}+Q_{Y}\right)\dy\right)\dx\\
            &=1-\int\sqrt{P_{Y}Q_{Y}}\dy+\int\frac12\left(\sqrt{P_{X}}-\sqrt{Q_{X}}\right)^2\dx\\
            &=\SH\PQ{X}+\SH\PQ{Y}
        \end{split}
    \]
\end{proof}

\subsection{Proof of \texorpdfstring{\cref{thm:subadditivity-kl}}{Theorem \ref{thm:subadditivity-kl}}}
\label{proof:subadditivity-kl}

\begin{proof}
    The subadditivity of KL-divergence is claimed in \citep{daskalakis2017square} without proof. Here, we provide a proof for completeness.
    
    Given \cref{thm:subadditivity-markov-chain}, we only need to show the following:
    \begin{enumerate}
        \item For two Markov Chains $P, Q$ on variables $X\to Y\to Z$, it holds that $\KL\PQ{XYZ}\leq \KL\PQ{XY}+\KL\PQ{YZ}$.
        \item For two product measures $P, Q$ on variables $X, Y$, it holds that $\KL\PQ{XY}\leq \KL\PQ{X}+\KL\PQ{Y}$.
    \end{enumerate} 
    We first show the subadditivity with respect to Markov Chains. The Markov property implies $P_{XYZ}=P_{XY}P_{YZ}/P_{Y}$ (and the same holds for $Q$). Thus,
    \[
        \begin{split}
            \KL\PQ{XYZ}&=\int P_{XYZ}\log\left(\frac{P_{XY}}{Q_{XY}}\frac{P_{YZ}}{Q_{YZ}}\middle/\frac{P_{Y}}{Q_{Y}}\right)\dx\dy\dz\\
            &=\int P_{XY}\log\left(\frac{P_{XY}}{Q_{XY}}\right)\dx\dy + \int P_{YZ}\log\left(\frac{P_{YZ}}{Q_{YZ}}\right)\dy\dz-\int P_{Y}\log\left(\frac{P_{Y}}{Q_{Y}}\right)\dy\\
            &=\KL\PQ{XY} + \KL\PQ{YZ} - \KL(P_{Y}, Q_{Y})\\
        \end{split}
    \]
    The subadditivity follows from the non-negativity of KL-divergence. Additivity holds when $\KL(P_{Y}, Q_{Y})=0$.
    
    It remains to show the subadditivity with respect to product measures. We will, in fact, show additivity rather than subadditivity. If $P, Q$ are product measures over $X, Y$, then $P_{XY}=P_XP_Y$ and $Q_{XY}=Q_XQ_Y$, hence,
      \[
        \begin{split}
            \KL\PQ{XY}&=\int P_{XY}\log\left(\frac{P_{X}}{Q_{X}}\frac{P_{Y}}{Q_{Y}}\right)\dx\dy\\
            &=\int P_{X}\log\left(\frac{P_{X}}{Q_{X}}\right)\dx + \int P_{Y}\log\left(\frac{P_{Y}}{Q_{Y}}\right)\dy\\
            &=\KL\PQ{X} + \KL\PQ{Y}.
        \end{split}
    \]
\end{proof}

\subsection{Proof of \texorpdfstring{\cref{coro:linear-subadditivity-js}}{Theorem \ref{coro:linear-subadditivity-js}}}
\label{proof:linear-subadditivity-js}

\begin{proof}
    The subadditivity of Jensen-Shannon divergence follows from:
    \begin{enumerate}
        \item The subadditivity of squared Hellinger distance (\cref{thm:subadditivity-sh}).
        \item $\gf$-Divergence inequalities (Theorem 11 of \citep{sason2016f}, repeated as \cref{thm:fdiv-ineq-sh&js} in \cref{appendix:subsec:fdiv-ineq}): for any two densities $P$ and $Q$,
        \[
            (\ln2)\SH\PQ{}\leq\JS\PQ{}\leq\SH\PQ{}
        \]
    \end{enumerate}
    Combining the inequalities implies that, for any pair of Bayes-nets $P, Q$ with respect to a DAG $G$, we have,
    \[
        \JS(P, Q)\leq\SH(P, Q)\leq\sum_{i=1}^n\SH\PQ{\Pi_{i}\cup X_{i}}\leq\frac{1}{\ln2}\sum_{i=1}^n\JS\PQ{\Pi_{i}\cup X_{i}}
    \]
    This proves that Jensen-Shannon divergence satisfies $(1/\ln2)$-linear subadditivity on Bayes-nets.
    
    Note that we assume natural logarithm is used in the definition of Jensen-Shannon divergence when deriving the inequalities between $\JS\PQ{}$ and $\SH\PQ{}$ (see \cref{thm:fdiv-ineq-sh&js} for details). However, the choice of the base of the logarithm does not affect the $(1/\ln2)$-linear subadditivity of Jensen-Shannon divergence.
\end{proof}

\subsection{Proof of \texorpdfstring{\cref{thm:linear-subadditivity-tv}}{Theorem \ref{thm:linear-subadditivity-tv}}}
\label{proof:linear-subadditivity-tv}

In the following proofs, we extensively use the Integral Probability Metric (IPM) formula of Total Variation distance \citep{muller1997integral}. If $\cF$ is the set of measurable functions on $\Sp$ taking values in $[0, 1]$, then,
\[
    \TV\PQ{}=\sup_{\df\in\cF}\Big|\EE_{x\sim P}[\df(x)]-\EE_{x\sim Q}[\df(x)]\Big|
\]

\begin{lemma}
\label{lem:subadditivity-tv-markov-chain}
    Let $P$ and $Q$ be two Bayes-nets with respect to DAG $X\to Y\to Z$. Then,
    \[
        \TV\PQ{XYZ}\leq\TV\PQ{XY}+\TV\PQ{Y}+\TV\PQ{YZ}
    \]
\end{lemma}

\begin{proof}
    We do a hybrid argument. By the triangle inequality, we have:
    \[
        \TV\PQ{XYZ}\leq\TV(P_{XYZ},P_{XY}Q_{Z|Y})+\TV(P_{XY}Q_{Z|Y},Q_{XYZ})
    \]
    We bound each term on the right-hand side separately. 
    
    Let us start with the second term. Let $\cF_{xy}$ be the set of measurable functions on variables $x$ and $y$ taking values in $[0, 1]$, and $\cF_{xyz}$ be the set of measurable functions on variables $x, y, z$ taking values in $[0, 1]$, etc. Using the Markov property, we know $P_{XYZ}=P_{XY}P_{Z|Y}=P_{Y}P_{X|Y}P_{Z|Y}$ (and the same holds for $Q$). Then,
    \[
        \begin{split}
            \TV(P_{XY}Q_{Z|Y},Q_{XYZ})&=\sup_{\df\in\cF_{xyz}}\Big|\EE_{P_{XY}Q_{Z|Y}}[\df(x,y,z)]-\EE_{Q_{XYZ}}[\df(x,y,z)]\Big|\\
            &=\sup_{\df\in\cF_{xyz}}\Big|\EE_{P_{XY}}\left[\EE_{Q_{Z|Y}}[\df(x,y,z)]\right]-\EE_{Q_{XY}}\left[\EE_{Q_{Z|Y}}[\df(x,y,z)]\right]\Big|\\
            &\leq\sup_{\df\in\cF_{xy}}\Big|\EE_{P_{XY}}\left[\df(x,y)\right]-\EE_{Q_{XY}}\left[\df(x,y)\right]\Big|\\
            &\equiv\TV\PQ{XY}
        \end{split}
    \]
    Let us now bound the first term,
    \[
        \begin{split}
            \TV(P_{XYZ},P_{XY}Q_{Z|Y})&=\sup_{\df\in\cF_{xyz}}\Big|\EE_{P_{XYZ}}[\df(x,y,z)]-\EE_{P_{XY}Q_{Z|Y}}[\df(x,y,z)]\Big|\\
            &=\sup_{\df\in\cF_{xyz}}\Big|\EE_{P_{Y}P_{Z|Y}}\left[\EE_{P_{X|Y}}[\df(x,y,z)]\right]-\EE_{P_{Y}Q_{Z|Y}}\left[\EE_{P_{X|Y}}[\df(x,y,z)]\right]\Big|\\
            &\leq\sup_{\df\in\cF_{yz}}\Big|\EE_{P_{Y}P_{Z|Y}}\left[\df(y,z)\right]-\EE_{P_{Y}Q_{Z|Y}}\left[\df(y,z)\right]\Big|\\
            &\leq\sup_{\df\in\cF_{yz}}\Big|\EE_{P_{Y}P_{Z|Y}}\left[\df(y,z)\right]-\EE_{Q_{Y}Q_{Z|Y}}\left[\df(y,z)\right]\Big|\\
            &\qquad+\sup_{\df\in\cF_{yz}}\Big|\EE_{Q_{Y}Q_{Z|Y}}\left[\df(y,z)\right]-\EE_{P_{Y}Q_{Z|Y}}\left[\df(y,z)\right]\Big|\\
            &=\TV(P_{YZ},Q_{YZ})+\sup_{\df\in\cF_{yz}}\Big|\EE_{Q_{Y}}\left[\EE_{Q_{Z|Y}}\left[\df(y,z)\right]\right]-\EE_{P_{Y}}\left[\EE_{Q_{Z|Y}}\left[\df(y,z)\right]\right]\Big|\\
            &\leq\TV\PQ{YZ}+\sup_{\df\in\cF_y}\Big|\EE_{Q_{Y}}\left[\df(y)\right]-\EE_{P_{Y}}\left[\df(y)\right]\Big|\\
            &\leq\TV\PQ{YZ}+\TV\PQ{Y} 
        \end{split}
    \]
    Combining the two inequalities concludes the proof.
\end{proof}

\begin{lemma}
\label{lem:subadditivity-tv-product-measure}
    Let $P$ and $Q$ be two product measures over variables $X$ and $Y$. Then,
    \[
        \TV\PQ{XY}\leq\TV\PQ{X}+\TV\PQ{Y}
    \]
\end{lemma}

\begin{proof}
    By the triangle inequality, we have:
    \[
        \TV\PQ{XY}\leq\TV(P_{XY},P_{X}Q_{Y})+\TV(P_{X}Q_{Y},Q_{XY})
    \]
    We bound each term on the right hand side separately. Let $\cF_{xy}$ be the set of measurable functions on variables $x$ and $y$ taking values in $[0, 1]$, and $\cF_y$ be the set of measurable functions on variable $y$ taking values in $[0, 1]$, etc. Then,
    \[
        \begin{split}
            \TV(P_{XY},P_{X}Q_{Y})&=\sup_{\df\in\cF_{xy}}\Big|\EE_{P_{XY}}[\df(x,y)]-\EE_{P_{X}Q_{Y}}[\df(x,y)]\Big|\\
            &=\sup_{\df\in\cF_{xy}}\Big|\EE_{P_{Y}}\left[\EE_{P_{X}}[\df(x,y)]\right]-\EE_{Q_{Y}}\left[\EE_{P_{X}}[\df(x,y)]\right]\Big|\\
            &\leq\sup_{\df\in\cF_y}\Big|\EE_{P_{Y}}\left[\df(y)\right]-\EE_{Q_{Y}}\left[\df(y)\right]\Big|\\
            &\equiv\TV\PQ{Y}
        \end{split}
    \]
    Similarly, we get $\TV(P_{X}Q_{Y},Q_{XY})\le \TV(P_{X},Q_{X})$. Combining the two inequalities concludes the proof.
\end{proof}

\begin{prevproof}{thm:linear-subadditivity-tv}
    Similar to the proof of \cref{thm:subadditivity-markov-chain}, for a pair of Bayes-nets $P$ and $Q$ with respect to a DAG $G$, we perform induction on each nodes of $G$. Consider the topological ordering $(1,\cdots,n)$ of the nodes of $G$. Consistent with the topological ordering, consider the following Markov Chain on super-nodes: $X_{\{1,\cdots,n-1\}\setminus\Pi_n}\to X_{\Pi_n}\to X_n$, where $\Pi_n$ is the set of parents of node $n$ and $\Pi_n\subseteq\{1,\cdots,n-1\}$. We distinguish three cases:
    \begin{enumerate}
        \item $\Pi_n\neq\varnothing$ and $\Pi_n\subsetneqq \{1,\cdots,n-1\}$: In this case, we apply \cref{lem:subadditivity-tv-markov-chain} to get $\TV\PQ{}\leq \TV\PQ{\cup_{i=1}^{n-1}X_i}+\TV\PQ{X_{\Pi_n}}+\TV\PQ{X_{\Pi_n}\cup X_n}$.
        \item $\Pi_n=\{1,\cdots,n-1\}$: In this case, it is trivial that $\TV\PQ{}\equiv\TV\PQ{X_{\Pi_n}\cup X_n}\leq\TV\PQ{\cup_{i=1}^{n-1}X_i}+\TV\PQ{X_{\Pi_n}}+\TV\PQ{X_{\Pi_n}\cup X_n}$.
        \item $\Pi_n=\varnothing$: In this case, $X_n$ is independent from $(X_1,\ldots,X_{n-1})$ in both Bayes-nets. Thus we apply \cref{lem:subadditivity-tv-product-measure} to get $\TV\PQ{}\leq\TV\PQ{\cup_{i=1}^{n-1}X_i}+\TV\PQ{X_n}\equiv\TV\PQ{\cup_{i=1}^{n-1}X_i}+\TV\PQ{X_{\Pi_n}}+\TV\PQ{X_{\Pi_n}\cup X_n}$, where $\TV\PQ{X_{\Pi_n}}=0$ and $\TV\PQ{X_{\Pi_n}\cup X_n}=\TV\PQ{X_1}$ as $\Pi_n=\varnothing$.
    \end{enumerate}
    We proceed by induction. For each inductive step $k=1,\cdots,n-2$, we consider the following Markov Chain on super-nodes: $X_{\{1,\cdots,n-k-1\}\setminus\Pi_{n-k}}\to X_{\Pi_{n-k}}\to X_{n-k}$. No matter what $\Pi_{n-k}$ is, we always have: $\TV\PQ{\cup_{i=1}^{n-k}X_i}\leq\TV\PQ{\cup_{i=1}^{n-k-1}X_i}+\TV\PQ{X_{\Pi_{n-k}}}+\TV\PQ{X_{\Pi_{n-k}}\cup X_{n-k}}$. In the end of the induction, we obtain: $\TV\PQ{}\leq\TV\PQ{X_1}+\sum_{i=2}^n\big(\TV\PQ{\Pi_{i}\cup X_{i}}+\TV\PQ{\Pi_{i}}\big)$. Since $\Pi_1\equiv\varnothing$, we know $\TV\PQ{X_{\Pi_1}}=0$ and $\TV\PQ{X_{\Pi_1}\cup X_1}=\TV\PQ{X_1}$. Hence, we conclude that,
    \[
        \TV\PQ{}\leq\sum_{i=1}^n\Big(\TV\PQ{\Pi_{i}\cup X_{i}}+\TV\PQ{\Pi_{i}}\Big)
    \]
    Now we relate this inequality to the notion of linear subadditivity. For two densities $P$ and $Q$ on variables $X, Y$, it holds that,
    \[
        \begin{split}
            \TV\PQ{X}&\equiv\frac12\int\Big|P_{X}-Q_{X}\Big|\dx\\
            &=\frac12\int\Big|\int P_{XY}\dy-\int Q_{XY}\dy\Big|\dx\\
            &\leq\frac12\int\left(\int\Big|P_{XY}-Q_{XY}\Big|\dy\right)\dx\\
            &\equiv\TV\PQ{XY}
        \end{split}
    \]
    Applying this inequality to $X_{\Pi_i}$ and $X_i$, for any $i\in\{1,\cdots,n\}$, we obtain, $\TV\PQ{\Pi_{i}}\leq\TV\PQ{\Pi_{i}\cup X_{i}}$. Thus,
    \[
        \TV\PQ{}\leq2\sum_{i=1}^n\TV\PQ{\Pi_{i}\cup X_{i}}
    \]
    This concludes that Total Variation distance satisfies $2$-linear subadditivity on Bayes-nets.
\end{prevproof}

\subsection{Proof of \texorpdfstring{\cref{coro:linear-subadditivity-wp}}{Theorem \ref{coro:linear-subadditivity-wp}}}
\label{proof:linear-subadditivity-wp}

\begin{proof}
    If $\Sp$ is a finite (and therefore bounded) metric space, there exist two-way bounds between $p$-Wasserstein distance and Total Variation distance (see \cref{thm:wasserstein-ineq-tv} in \cref{appendix:subsec:wassserstein-formulas} for details), namely,
    \[
        \W_p\PQ{}^p/\diam(\Sp)^p\leq \TV\PQ{} \leq \W_p\PQ{}^p/d_{\min}^p
    \]
    where $\diam(\Sp)=\max\{d(x,y)|x,y\in\Sp\}$ is the diameter of the space $\Omega$ and $d_{\min}=\min_{x\neq y}d(x,y)$ is the smallest distance between pairs of distance points in $\Sp$. For $p\geq 1$, this directly implies the $(2^{1/p}\diam(\Sp)/d_{\min})$-linear subadditivity of $p$-Wasserstein distance on Bayes-nets on finite $\Sp$,
    \[
        \W_p(P, Q) \leq \frac{2^{1/p}\diam(\Sp)}{d_{\min}}\sum_{i=1}^n\W_p\PQ{X_i\cup X_{\Pi_i}}
    \]
    via the $2$-linear subadditivity of Total Variation distance (\cref{thm:linear-subadditivity-tv}).
\end{proof}

\subsection{Proof of \texorpdfstring{\cref{thm:subadditivity-ipm}}{Theorem \ref{thm:subadditivity-ipm}}}
\label{proof:subadditivity-ipm}

\begin{proof}
    For reference, we repeat the three conditions of the subadditivity of neural distances here:
    \begin{enumerate}[(1)]
        \item The space $\Sp$ is bounded, i.e. $\diam(\Sp)<\infty$.
        \item For any $i\in\{1,\cdots,n\}$, discriminator class $\cF_i$ is larger than the set of neural networks with a single neuron, which have ReLU activation and bounded parameters, i.e. $\cF_i\supseteq\{\max\{w^Tx+b,0\}\big|w\in\RR^{D_i}, b\in\RR, \|[w,b]\|_2=1\}$, where $D_i$ is the number of dimensions of variables $X_i\cup X_{\Pi_i}$.
        \item For any $i\in\{1,\cdots,n\}$, $\log(P_{X_i\cup X_{\Pi_i}}/Q_{X_i\cup X_{\Pi_i}})$ exists, and is bounded and Lipschitz continuous.
    \end{enumerate}
    
    For two distributions $P, Q$ and a set of discriminators $\cF$ satisfying all the three conditions, by \cref{thm:approximability-lipschitz} we know that for any $i\in\{1,\cdots,n\}$, $\log(P_{X_i\cup X_{\Pi_i}}/Q_{X_i\cup X_{\Pi_i}})$ is inside the closure of the linear span of $\cF_i$, i.e. $\log(P_{X_i\cup X_{\Pi_i}}/Q_{X_i\cup X_{\Pi_i}})\in\cl(\vspan\cF_i)$. Moreover, each $\log(P_{X_i\cup X_{\Pi_i}}/Q_{X_i\cup X_{\Pi_i}})$ is approximated by the corresponding $\cF_i$ with an error decay function, denoted by $\veps_i(r)$. Using \cref{thm:ipm-discriminative}, we upper-bound each Symmetric KL divergence between local marginals, $\SKL\PQ{X_i\cup X_{\Pi_i}}$, by a linear function of the corresponding neural distance $d_\cF\PQ{X_i\cup X_{\Pi_i}}$,
    \[
        \SKL\PQ{X_i\cup X_{\Pi_i}}\leq 2\veps_i(r)+rd_{\cF_i}\PQ{X_i\cup X_{\Pi_i}} \qquad \forall r\geq0, \forall i\in\{1,\cdots,n\}
    \]
    
    Because of the condition (3): each $\log(P_{X_i\cup X_{\Pi_i}}/Q_{X_i\cup X_{\Pi_i}})$ is bounded and Lipschitz continuous, there exists a constant $\eta_i>0$, such that,
    \[
        \left|\log(P_{X_i\cup X_{\Pi_i}}/Q_{X_i\cup X_{\Pi_i}})\right|<\eta_i
    \]
    and for any $x, y\in\Sp_i$ (which is the space of variables $X_i\cup X_{\Pi_i}$), it holds that,
    \[
        \left|\log(P_{X_i\cup X_{\Pi_i}}(x)/Q_{X_i\cup X_{\Pi_i}}(x))-\log(P_{X_i\cup X_{\Pi_i}}(y)/Q_{X_i\cup X_{\Pi_i}}(y))\right|\leq\frac{\eta_i}{\diam(\Sp_i)}\|x-y\|
    \]
    
    Again, by \cref{thm:approximability-lipschitz}, we get an efficient upper-bound on $\veps_i(r)$,
    \[
        \veps_i(r)\leq C(D_i)\eta_i\left(\frac{r}{\eta_i}\right)^{-\frac{2}{D_i+1}}\log\left(\frac{r}{\eta_i}\right) \qquad \forall r\geq R(D_i)>\ee^{\frac{D_i+1}{2}}\eta_i, \forall i\in\{1,\cdots,n\}
    \]
    where $C(D_i)$ and $R(D_i)$ are constants that only depend on the dimensionality, $D_i$, of variables $X_i\cup X_{\Pi_i}$. More specifically, $D_i=(k_i+1)d\leq (k_{\max}+1)d$, where $k_i$ is the in-degree of node $i$, $d$ is the dimensionality of each variable of the Bayes-nets, and $k_{\max}$ is the maximum in-degree of $G$.
    
    Because $C(D_i)$ and $R(D_i)$ are increasing functions of the dimensionality $D_i$, and for $r\geq R(D_i)>\ee^{\frac{D_i+1}{2}}\eta_i$, $\eta_i\left(r/\eta_i\right)^{-\frac{2}{D_i+1}}\log\left(r/\eta_i\right)$ is an increasing function of $\eta_i$, summing up the inequalities for all $i\in\{1,\cdots,n\}$ gives,
    \[
        \sum_{i=1}^n\veps_i(r)\leq nC(D_{\max})\eta_{\max}\left(\frac{r}{\eta_{\max}}\right)^{-\frac{2}{D_{\max}+1}}\log\left(\frac{r}{\eta_{\max}}\right) \qquad \forall r\geq R(D_{\max})
    \]
    where $D_{\max}=\max\{D_i\}=(k_{\max}+1)d$ and $\eta_{\max}=\max\{\eta_i\}$.
    
    Now, we sum up the inequalities $\SKL\PQ{X_i\cup X_{\Pi_i}}\leq 2\veps_i(r)+rd_\cF\PQ{X_i\cup X_{\Pi_i}}$ for $r\geq R(D_{\max})$ for all $i\in\{1,\cdots,n\}$. Because of the subadditivity of Symmetric KL divergence on Bayes-nets $P, Q$ (\cref{coro:subadditivity-jf}), we get,
    \[
        \SKL\PQ{}-2\sum_{i=1}^n\veps_i(r) \leq r\sum_{i=1}^n d_{\cF_i}\PQ{X_i\cup X_{\Pi_i}} \qquad \forall r\geq R(D_{\max})
    \]
    That is, the neural distances defined by $\cF_1, \ldots, \cF_n$ satisfy $r$-linear subadditivity for,
    \[
        r\geq R(D_{\max})
    \]
    with error,
    \[
        \eps=2\sum_{i=1}^n\veps_i(r)=\cO\left(n r^{-\frac{2}{D_{\max}+1}}\log{r}\right)
    \]
    with respect to the Symmetric KL divergence on Bayes-nets. 
    
    Note that $r$ and $\eps$ are constants independent of the Bayes-nets $P, Q$ and the sets of discriminator classes $\{\cF_1,\cdots,\cF_n\}$. And $D_{\max}=(k_{\max}+1)d$ where $k_{\max}$ is the maximum in-degree of $G$ and $d$ is the dimensionality of each variable of the Bayes-nets.
\end{proof}

\subsection{Proof of \texorpdfstring{\cref{thm:wasserstein-mrf}}{Theorem \ref{thm:wasserstein-mrf}}}
\label{proof:wasserstein-mrf}

\begin{proof}
    We first give a proof when condition (2) holds. For a pair of MRFs $P$ and $Q$ with the same factorization (thus with the same underlying graph $G$),
    \[
        P(x)=\prod_{C\in\cC}\fP^P_C(X_C) \qquad Q(x)=\prod_{C\in\cC}\fP^Q_C(X_C)
    \]
    The Symmetric KL divergence between $P$ and $Q$,
    \[
        \SKL\PQ{}\coloneqq\KL\PQ{}+\KL(Q, P)=\EE_{x\sim P}\left[\log(P/Q)\right]-\EE_{x\sim Q}\left[\log(P/Q)\right]
    \]
    can be decomposed into,
    \[
        \SKL\PQ{} = \sum_{C\in\cC}\left(\EE_{x_C\sim P_{X_C}}\left[\log(\fP^P_C/\fP^Q_C)\right]-\EE_{x_C\sim Q_{X_C}}\left[\log(\fP^P_C/\fP^Q_C)\right]\right)
    \]
    Where each term in the summation is upper-bounded by the $1$-Wasserstein distance between $P_{X_C}$ and $Q_{X_C}$ up to a constant factor,
    \[
        \begin{split}
            \EE_{x_C\sim P_{X_C}}\left[\log(\fP^P_C/\fP^Q_C)\right]&-\EE_{x_C\sim Q_{X_C}}\left[\log(\fP^P_C/\fP^Q_C)\right]\\
            &\leq \eta_C\W_1\PQ{X_C} \coloneqq \eta_C\sup_{\df\text{ $1$-Lipschitz}}\left\{\EE_{x_C\sim P_{X_C}}[\df(x)]-\EE_{x_C\sim Q_{X_C}}[\df(x)]\right\}\\
        \end{split}
    \]
    if $\log(\fP^P_C/\fP^Q_C)$ is Lipschitz continuous with Lipschitz constant $\eta_C$. Summing up the inequalities for all maximal cliques $C\in\cC$, we get,
    \[
        \SKL\PQ{} \leq \eta_{\max} \sum_{C\in\cC}\W_1\PQ{X_C} 
    \]
    where $\eta_{\max}=\max\{\eta_C|C\in\cC\}$ is the maximum Lipschitz constant. That is, $1$-Wasserstein distance satisfies $\eta_{\max}$-linear subadditivity with respect to the Symmetric KL Divergence on MRFs.
    
    We conclude the proof by showing that condition (1) implies condition (2). For a discrete and finite space $\Sp$, each $\log(\fP^P_C/\fP^Q_C)$ maps any configuration $x_C$ in $\Sp_C\subseteq\RR^{|C|d}$ (the space of variables $X_C$) to a real number, where $|C|$ is the size of clique $C$ and $d$ is the dimensionality of each variable of the MRFs. We can always extend the domain of $\log(\fP^P_C/\fP^Q_C)$ to $\RR^{|C|d}$, so that the extended function is Lipschitz continuous with Lipschitz constant,
    \[
        \eta_C = \max\left\{\frac{\left|\log(\fP^P_C(x_C^1)/\fP^Q_C(x_C^1))-\log(\fP^P_C(x_C^2)/\fP^Q_C(x_C^2))\right|}{\left\|x_C^1-x_C^2\right\|}\Bigg|x_C^1\neq x_C^2\in\Sp_C\right\}
    \]
    The rest of the proof follows from the proof above.
\end{proof}

\subsection{Proof of \texorpdfstring{\cref{coro:subadditivity-ipm-mrf}}{Theorem \ref{coro:subadditivity-ipm-mrf}}}
\label{proof:subadditivity-ipm-mrf}

\begin{proof}
    The proof is similar to the proof of \cref{thm:subadditivity-ipm} (in \cref{proof:subadditivity-ipm}) with a few differences. For a pair of MRFs $P$ and $Q$ with the same factorization (thus with the same underlying graph $G$), the Symmetric KL divergence between $P$ and $Q$ can be decomposed into,
    \[
        \SKL\PQ{} = \sum_{C\in\cC}\left(\EE_{x_C\sim P_{X_C}}\left[\log(\fP^P_C/\fP^Q_C)\right]-\EE_{x_C\sim Q_{X_C}}\left[\log(\fP^P_C/\fP^Q_C)\right]\right)
    \]
    
    For two distributions $P, Q$ and a set of discriminators $\cF$ satisfying all the three conditions, by \cref{thm:approximability-lipschitz} we know that for any $C\in\cC$, $\log(\fP^P_C/\fP^Q_C)$ is inside the closure of the linear span of $\cF_C$, i.e. $\log(\fP^P_C/\fP^Q_C)\in\cl(\vspan\cF_C)$. Moreover, each $\log(\fP^P_C/\fP^Q_C)$ is approximated by the corresponding $\cF_C$ with an error decay function, denoted by $\veps_C(r)$. Using \cref{thm:ipm-discriminative} and assign $g=\log(\fP^P_C/\fP^Q_C)$ (instead of $\log(P_{X_C}/Q_{X_C})$), we get,
    \[
        \EE_{x_C\sim P_{X_C}}\left[\log(\fP^P_C/\fP^Q_C)\right]-\EE_{x_C\sim Q_{X_C}}\left[\log(\fP^P_C/\fP^Q_C)\right] \leq 2\veps_C(r)+rd_{\cF_C}\PQ{X_C} \qquad \forall r\geq0, \forall C\in\cC
    \]
    
    Because of the condition (3): each $\log(\fP^P_C/\fP^Q_C)$ is bounded and Lipschitz continuous, there exists a constant $\eta_C>0$, such that $\left|\log(\fP^P_C/\fP^Q_C)\right|<\eta_C$, and for any $x, y\in\Sp_C$ (which is the space of variables $X_C$), it holds that $\left|\log(\fP^P_C(x)/\fP^Q_C(x))-\log(\fP^P_C(y)/\fP^Q_C(y))\right|\leq\frac{\eta_C}{\diam(\Sp_C)}\|x-y\|$.
    
    Again, by \cref{thm:approximability-lipschitz}, we get an efficient upper-bound on $\veps_C(r)$,
    \[
        \veps_C(r)\leq C(D_C)\eta_C\left(\frac{r}{\eta_C}\right)^{-\frac{2}{D_C+1}}\log\left(\frac{r}{\eta_C}\right) \qquad \forall r\geq R(D_C)>\ee^{\frac{D_C+1}{2}}\eta_C, \forall C\in\cC
    \]
    where $C(D_C)$ and $R(D_C)$ are constants that only depend on the dimensionality, $D_C$, of variables $X_C$. More specifically, $D_C=|C|d\leq c_{\max}d$, where $|C|$ is the size of clique $C$, $d$ is the dimensionality of each variable of the MRFs, and $c_{\max}=\max\{|C|\big|C\in\cC\}$ is the maximum size of the cliques in $G$.
    
    Because $C(D_C)$ and $R(D_C)$ are increasing functions of the dimensionality $D_C$, and for $r\geq R(D_C)>\ee^{\frac{D_C+1}{2}}\eta_C$, $\eta_C\left(r/\eta_C\right)^{-\frac{2}{D_C+1}}\log\left(r/\eta_C\right)$ is an increasing function of $\eta_C$, summing up the inequalities for all $C\in\cC$ gives,
    \[
        \sum_{C\in\cC}\veps_C(r)\leq |\cC|C(D_{\max})\eta_{\max}\left(\frac{r}{\eta_{\max}}\right)^{-\frac{2}{D_{\max}+1}}\log\left(\frac{r}{\eta_{\max}}\right) \qquad \forall r\geq R(D_{\max})
    \]
    where $|\cC|$ is the number of maximal cliques in $G$, $D_{\max}=\max\{D_C|C\in\cC\}=c_{\max}d$, and $\eta_{\max}=\max\{\eta_C|C\in\cC\}$.
    
    Now, we sum up the inequalities $\EE_{x_C\sim P_{X_C}}\left[\log(\fP^P_C/\fP^Q_C)\right]-\EE_{x_C\sim Q_{X_C}}\left[\log(\fP^P_C/\fP^Q_C)\right] \leq 2\veps_C(r)+rd_{\cF_C}\PQ{X_C}$ for $r\geq R(D_{\max})$ for all $C\in\cC$. Because of the decomposed form of the Symmetric KL divergence on MRFs $P, Q$, we get,
    \[
        \SKL\PQ{}-2\sum_{C\in\cC}\veps_C(r) \leq r\sum_{C\in\cC} d_{\cF_C}\PQ{X_C} \qquad \forall r\geq R(D_{\max})
    \]
    That is, the neural distances defined by $\{\cF_C|C\in\cC\}$ satisfy $r$-linear subadditivity for,
    \[
        r\geq R(D_{\max})
    \]
    with error,
    \[
        \eps=2\sum_{C\in\cC}\veps_C(r)=\cO\left(|\cC| r^{-\frac{2}{D_{\max}+1}}\log{r}\right)
    \]
    with respect to the Symmetric KL divergence on MRFs. 
    
    Note that $r$ and $\eps$ are constants independent of the MRFs $P, Q$ and the sets of discriminator classes $\{\cF_C|C\in\cC\}$. $|\cC|$ is the number of maximal cliques in $G$ and $D_{\max}=c_{\max}d$ where $c_{\max}=\max\{|C|\big|C\in\cC\}$ is the maximum size of the cliques in $G$ and $d$ is the dimensionality of each variable of the MRFs.
\end{proof}

%% file: Appendices.tex
%%%%%%%%%%%%%%%%%%%%%%%%%%%%%%%%%%%%%%%%%%%%%%%%%%%%%%%%%%%%%%%%%%%%%%%%%%%%%%%%%%%%%%%%%%%%%%%%%%%%
\section{\texorpdfstring{$\gf$}{f}-Divergences and Inequalities}
\label{appendix:fdivs}

For two probability distributions $P$ and $Q$ on the same sample space $\Omega$, the $\gf$-divergence of $P$ from $Q$, denoted $\PD(P, Q)$, is defined as,
\[
    \PD(P, Q) \coloneqq \int_\Sp \gf\left(\frac{\mathrm{d}P}{\mathrm{d}Q}\right)\mathrm{d}Q
\]
If densities exist, $\PD(P, Q)=\int_\Sp \gf\left(\frac{P(x)}{Q(x)}\right)Q(x)\dx$. In this definition, the function $\gf: \RR_+ \to \RR$ is a convex, lower-semi-continuous function satisfying $\gf(1)=0$. We can define $\gf(0)=\lim_{t\downarrow0}\gf(t) \in \RR\cup\{\infty\}$. Every convex, lower semi-continuous function $\gf$ has a convex conjugate function $\gf^*$, defined as $\gf^*=\sup_{u\in\dom_\gf}\{ut-\gf(u)\}$.

\subsection{Common \texorpdfstring{$\gf$}{f}-Divergences}
\label{appendix:subsec:common-fdivs}

All commonly-used $\gf$-divergences are listed in \cref{tab:common-f-divs}.

\begin{table}[H]
    \centering
    \renewcommand{\arraystretch}{1.4}
    \begin{tabular}{lll}
        \toprule
        \textbf{Name} & \textbf{Notation} & \textbf{Generator} $\gf(t)$ \\ \midrule
        Kullback–Leibler & $\KL$ & $t\log(t)$ \\
        Reverse KL & $R\KL$ & $-\log(t)$ \\
        Symmetric KL & $\SKL$ & $(t-1)\log(t)$ \\
        Jensen-Shannon & $\JS$ & $\frac{t}{2}\log\frac{2t}{t+1}+\frac{1}{2}\log\frac{2}{t+1}$ \\
        Squared Hellinger & $\SH$ & $\frac12\left(\sqrt{t}-1\right)^2$ \\
        Total Variation & $\TV$ & $\frac12|t-1|$ \\
        Pearson $\CS$ & $\CS$ & $(t-1)^2$ \\
        Reverse Pearson $\CS$ & $R\CS$ & $\frac{1}{t}-t$ \\
        $\ap$-Divergence & $\HA$ & \begin{tabular}[c]{@{}l@{}}$\begin{cases} \frac{t^\ap-1}{\ap(\ap-1)} & \ap\neq0,1\\ t\ln t & \ap=1\\ -\ln t & \ap=0\\ \end{cases}$\end{tabular} \\
        \bottomrule
    \end{tabular}
    \vspace{10pt}
    \caption{List of common $\gf$-divergences with generator functions.}
    \label{tab:common-f-divs}
\end{table}

We always adopt the most widely-accepted definitions. Note the $\frac12$ coefficients in the definitions of squared Hellinger distance and Total Variation distance, in the spirit of normalizing their ranges to $[0, 1]$. 

The $\ap$-divergences $\HA$ ($\ap\in\RR$), popularized by \citep{liese2006divergences}, generalize many $\gf$-divergences including KL divergence, reverse KL divergence, $\CS$ divergence, reverse $\CS$ divergence, and Hellinger distances. More specifically, they satisfy the following relations: $\cH_1=\KL$, $\cH_0=R\KL$, $\cH_2=\frac12\CS$, $\cH_{-1}=\frac12R_{\CS}$, and $\cH_{\frac12}=4\SH$.

\subsection{Inequalities between \texorpdfstring{$\gf$}{f}-Divergences}
\label{appendix:subsec:fdiv-ineq}

First, we show a general approach to obtain inequalities between $\gf$-divergences. Then, we prove the inequalities between squared Hellinger distance and Jensen-Shannon divergence. %, and the inequalities between squared Hellinger distance and Total Variation distance independently. 
We also list the well-known Pinsker's inequality for completeness.

\begin{lemma}
\label{lem:fdiv-general-ineq}
    Consider two $\gf$-divergences $D_{\gf_1}$ and $D_{\gf_2}$ with generator functions $\gf_1(\cdot)$ and $\gf_2(\cdot)$. If there exist two positive constants $0<A<B$, such that for any $t\in[0,\infty)$, it holds that,
    \[
        A\gf_2(t)\leq \gf_1(t)\leq B\gf_2(t)
    \]
    Then, for any two densities $P$ and $Q$ (such that $P\ll Q$), we have,
    \[
        AD_{\gf_2}\PQ{}\leq D_{\gf_1}\PQ{}\leq BD_{\gf_2}\PQ{}
    \]
\end{lemma}

\begin{proof}
    Note that we extend the domain of $\gf_1$ and $\gf_2$ by defining $\gf_1(0)=\lim_{t\downarrow0}\gf_1(t)$ (and similar for $\gf_2$). We require $P\ll Q$ so that $\gf$-divergences are well-defined. In this sense, for any $x\in\Sp$, $P(x)/Q(x)\in[0,\infty)$ is defined, and we have $A\gf_2(P(x)/Q(x))\leq\gf_1(P(x)/Q(x))\leq B\gf_2(P(x)/Q(x))$. Multiply non-negative $Q(x)$ and integrate over $\Sp$. We obtain the desired inequality: $AD_{\gf_2}\PQ{}\leq D_{\gf_1}\PQ{}\leq BD_{\gf_2}\PQ{}$.
\end{proof}

\begin{theorem}[Theorem 11 of \citep{sason2016f}]
\label{thm:fdiv-ineq-sh&js}
    For any two densities $P$ and $Q$, (assume natural logarithm is used in the definition of Jensen-Shannon divergence), we have
    \[
        (\ln2)\SH\PQ{}\leq\JS\PQ{}\leq\SH\PQ{}
    \]
\end{theorem}

\begin{proof}
    Given \cref{lem:fdiv-general-ineq}, we only need to prove that for any $t\in[0,\infty)$, the following inequality holds,
    \[
        (\ln2)\gf_{\SH}(t)\leq \gf_{\JS}(t)\leq \gf_{\SH}(t)
    \]
    where the definitions of $\gf_{\SH}$ and $\gf_{\JS}$ can be found in \cref{tab:common-f-divs}. 
    
    Note that when $t=1$, all terms are $0$ and the inequalities hold trivially. For $t\neq1$, as $\gf_{\SH}(t)>0$, we define,
    \[
        \xi(t)=\frac{\gf_{\JS}(t)}{\gf_{\SH}(t)}=\frac{t\ln\frac{2t}{t+1}+\ln\frac{2}{t+1}}{\left(\sqrt{t}-1\right)^2}
    \]
    $\xi(t)$ is defined on $[0,1)\cup(1,\infty)$, We want to prove that $\ln2\leq\xi(t)\leq1$ always holds. Its derivative is,
    \[
        \xi'(t)=\frac{\sqrt{t}\ln\frac{2t}{t+1}+\ln\frac{2}{t+1}}{\sqrt{t}\left(1-\sqrt{t}\right)^3}
    \]
    Denote the numerator above by $\xi_{(1)}(t)$. Its derivative is,
    \[
        \xi_{(1)}'(t)=\frac{(t+1)\ln\frac{2t}{t+1}+2\left(1-\sqrt{t}\right)}{2\sqrt{t}(t+1)}
    \]
    Again, denote the numerator above by $\xi_{(2)}(t)$. Its derivative is,
    \[
        \xi_{(2)}'(t)=\frac{1}{t}-\frac{1}{\sqrt{t}}+\ln\frac{2t}{t+1}
    \]
    Using the well-known logarithm inequality: for any $x>0$, $\ln{x}>1-\frac{1}{x}$, we have,
    \[
        \xi_{(2)}'(t)\geq\frac{1}{t}-\frac{1}{\sqrt{t}}+1-\frac{t+1}{2t}=\frac{\left(\sqrt{t}-1\right)^2}{2t}\geq0
    \]
    Also, since $\xi_{(2)}(1)=0$, and the denominator of $\xi_{(1)}'(t)$ is always positive, hence,
    \[
        \xi_{(1)}'(t)
        \begin{cases}
            <0 & t\in[0,1)\\
            >0 & t\in(1,\infty)\\
        \end{cases}
    \]
    Because $\xi_{(1)}(1)=0$, this implies $\xi_{(1)}(t)\geq0$. Thus,
    \[
        \xi'(t)
        \begin{cases}
            >0 & t\in[0,1)\\
            <0 & t\in(1,\infty)\\
        \end{cases}
    \]
    That is, $\xi(t)$ is strictly increasing on $[0,1)$, and is strictly decreasing on $(1,\infty)$. To determine its range, we only need to compute these limits: $\lim_{t\downarrow0}\xi(t)$, $\lim_{t\uparrow1}\xi(t)$, $\lim_{t\downarrow1}\xi(t)$, and $\lim_{t\to+\infty}\xi(t)$:
    \[
        \begin{split}
            &\lim_{t\downarrow0}\xi(t)=\ln2\\
            &\lim_{t\uparrow1}\xi(t)=\lim_{t\downarrow1}\xi(t)=\lim_{t\to1}\frac{\sqrt{t}\ln\frac{2t}{t+1}}{\sqrt{t}-1}=\lim_{t\to1}\frac{2\sqrt{t^3}}{t+1}=1\\
            &\lim_{t\to+\infty}\xi(t)=\lim_{t\to+\infty}\frac{t\ln\frac{2t}{t+1}}{\left(\sqrt{t}-1\right)^2}=\lim_{t\to+\infty}\frac{\ln\frac{2t}{t+1}+\frac{1}{t+1}}{\frac{\sqrt{t}-1}{\sqrt{t}}}=\ln2\\
        \end{split}
    \]
    Together with the monotonic properties of $\xi(t)$, we know
    \[
        \ln2\leq \xi(t)\leq1
    \]
\end{proof}

\begin{theorem}[Pinsker's Inequality, Eq. (1) of \citep{sason2016f}]
\label{thm:fdiv-pinskers-ineq}
    For any two densities $P$ and $Q$, we have,
    \[
        \TV\PQ{}\leq\sqrt{\frac12\KL\PQ{}}
    \]
\end{theorem}

It is a well-known result. See for example Theorem 2.16 of \citep{massart2007concentration} for a proof.

%%%%%%%%%%%%%%%%%%%%%%%%%%%%%%%%%%%%%%%%%%%%%%%%%%%%%%%%%%%%%%%%%%%%%%%%%%%%%%%%%%%%%%%%%%%%%%%%%%%%
\section{Wasserstein Distances: Formulas and Inequalities}
\label{appendix:wasserstein}

Suppose $\Omega$ is a metric space with distance $d(\cdot,\cdot)$. The $p$-Wasserstein distance $\W_p$ is defined as,
\[
    \W_p\PQ{} \coloneqq \left(\inf_{\gamma\in\Gamma\PQ{}}\int_{\Sp\times\Sp}d(x,y)^p\mathrm{d}\gamma(x,y)\right)^{\frac{1}{p}}
\]
where $\gamma\in\Gamma\PQ{}$ denotes the set of all possible couplings of $P$ and $Q$.

\subsection{Formulas for Wasserstein Distances}
\label{appendix:subsec:wassserstein-formulas}

We list the algorithm and the formula to calculate the Wasserstein distance when space $\Sp$ is finite or the distributions $P$ and $Q$ are Gaussians.

\begin{theorem}
\label{thm:wp-formula-discrete}
    For any two discrete distributions $P, Q$ on a finite space $\Sp=\{\mathbf{x}_1,\cdots,\mathbf{x}_n\}$, the $p$-Wasserstein distance $\W_p$ can be computed by the following linear program:
    \begin{equation*}
        \begin{array}{@{}lr@{}l@{}l@{}}
        \W_p\PQ{}^p= & \min\text{ } & \sum_{i=1}^n\sum_{j=1}^n d^p(\mathbf{x}_i, \mathbf{x}_j)\pi_{ij} & \\
                     &\text{subject to } & \sum_{j=1}^n\pi_{ij}=P(\mathbf{x}_i) & \quad i=1,\cdots,n \\
                     &                   & \sum_{i=1}^n\pi_{ij}=Q(\mathbf{x}_j) & \quad j=1,\cdots,n \\
                     &\text{and }        & \pi_{ij}>0 & \quad i=1,\cdots,n \text{ and } j=1,\cdots,n \\
        &\end{array}
    \end{equation*}
\end{theorem}

Useful discussions can be found in \citep{oberman2015efficient}.

\begin{theorem}
\label{thm:w2-formula-gaussian}
    For any two non-degenerate Gaussians $P=\cN(m_1, C_1)$ and $Q=\cN(m_2, C_2)$ on $\RR^n$, with respective means $m_1, m_2\in\RR^n$ and (symmetric positive semi-definite) covariance matrices $C_1, C_2\in\RR^{n\times n}$. The square of 2-Wasserstein distance $\W_2$ between $P, Q$ is,
    \[
        \W_2\PQ{}^2=\|m_1-m_2\|_2^2+\trace\left(C_1+C_2-2\left(C_2^{1/2}C_1C_2^{1/2}\right)^{1/2}\right)
    \]
    where $\|\cdot\|_2$ is the Euclidean norm.
\end{theorem}

See \citep{olkin1982distance} for a proof.

\subsection{Inequalities between \texorpdfstring{$p$}{p}-Wasserstein Distance and Total Variation Distance}
\label{appendix:subsec:wasserstein-ineq-tv}

Both Wasserstein distances and Total Variation distance can be regarded as optimal transportation costs. More specifically, 
\[
    \begin{split}
        \W_p\PQ{}&\coloneqq\left(\inf_{\gamma\in\Gamma\PQ{}}\int_{\Sp\times\Sp}d(x,y)^p\mathrm{d}\gamma(x,y)\right)^{\frac{1}{p}}\\
        \TV\PQ{}&\coloneqq\inf_{\gamma\in\Gamma\PQ{}}\int_{\Sp\times\Sp}\mathbf{1}_{x\neq y}\mathrm{d}\gamma(x,y)
    \end{split}
\]
where $\Gamma\PQ{}$ denotes the set of all measures on $\Sp\times\Sp$ with marginals $P$ and $Q$ on variable $x$ and $y$ respectively, (also called the set of all possible couplings of $P$ and $Q$). Bounding the distance $d(x,y)$ directly leads to inequalities between $p$-Wasserstein distance and Total Variation distance.

\begin{theorem}
\label{thm:wasserstein-ineq-tv}
    For any two distributions $P$ and $Q$ on a space $\Sp$, if $\Sp$ is bounded with diameter $\diam(\Sp)=\max\{d(x,y)|x,y\in\Sp\}$, then,
    \[
        \W_p\PQ{}^p\leq \diam(\Sp)^p \TV\PQ{}
    \]
    Moreover, if $\Sp$ is finite, let $d_{\min}=\min_{x\neq y}d(x,y)$ be the minimum mutual distance between pairs of distinct points in $\Sp$, then,
    \[
        \W_p\PQ{}^p\geq d^p_{\min}\TV\PQ{}
    \]
\end{theorem}

\begin{proof}
    This theorem is a generalization of Theorem 4 of \citep{gibbs2002choosing}. Since $d(\cdot,\cdot)$ is a metric of space $\Sp$, $d(x,y)=0$ if and only if $x=y$. Thus $d(x,y)\equiv d(x,y)\mathbf{1}_{x\neq y}$, and we have,
    \[
        \W_p\PQ{}^p=\inf_{\gamma\in\Gamma\PQ{}}\int_{\Sp\times\Sp}d(x,y)^p\mathbf{1}_{x\neq y}\mathrm{d}\gamma(x,y)
    \]
    If $\Sp$ is bounded, then for any $x, y$ in $\Sp$, it holds that $d(x,y)\leq\diam(\Sp)$. Applying this inequality to the formula above leads to $\W_p\PQ{}^p\leq \diam(\Sp)^p \TV\PQ{}$.
    
    Similarly, if $\Sp$ is finite, then for any distinct $x\neq y$ in $\Sp$, it holds that $d(x,y)\geq d_{\min}$. We can generalize it to: for any $x, y$ in $\Sp$, we have $d(x,y)\mathbf{1}_{x\neq y}\geq d_{\min}\mathbf{1}_{x\neq y}$. Applying this inequality to the formula above leads to $\W_p\PQ{}^p\geq d^p_{\min}\TV\PQ{}$.
\end{proof}

%%%%%%%%%%%%%%%%%%%%%%%%%%%%%%%%%%%%%%%%%%%%%%%%%%%%%%%%%%%%%%%%%%%%%%%%%%%%%%%%%%%%%%%%%%%%%%%%%%%%
\section{``Breadth First Search''-Subadditivity on MRFs}
\label{appendix:random-field}

Most of our theoretical results in this paper are for the subadditivity of divergences on Bayes-nets. However, following the same recursive approach as in the proof of \cref{thm:subadditivity-markov-chain}, we can develop a different version of subadditivity on MRFs that depends on a Breadth-First Search (BFS) ordering $(1, \ldots, n)$ on the undirected graph $G$, which we call {\em BFS-Subadditivity on MRFs} (to distinguish it from the version we defined in \cref{def:subadditivity}).

For BFS-Subadditivity on MRFs, each local neighborhood is the union of a node $k\in\{1, \ldots, n\}$ and a subset $\Sigma_k=\cup_{i=1}^k N_i\setminus\{1, \ldots ,k\}$, where $N_i$ is the set of nodes adjacent to node $i$, and $\Sigma_k$ is a separating subset between $\{1, \ldots, k\}$ and $\{k+1, \ldots, n\}\setminus\Sigma_k$. The construction of BFS-Subadditivity of a divergence $\D$ requires exactly the same two properties as in \cref{thm:subadditivity-markov-chain}, i.e. $\D$ is subadditive with respect to product measures and length-$3$ Markov Chains. In this sense, it is not hard to verify that all the divergences we prove to satisfy subadditivity on Bayes-net in the paper, satisfy BFS-Subadditivity on MRFs as well.

\subsection{Constructing Subadditivity Upper-Bound on Generic Graphical Models}
\label{appendix:subsec:generic-graphical-model}

From the proof of \cref{thm:subadditivity-markov-chain} in \cref{proof:subadditivity-markov-chain}, we obtain the subadditivity upper-bound on Bayes-nets by repeatedly applying the subadditivity inequality on Markov Chain $X\to Y\to Z$. Moreover, we allow $X=\varnothing$ or $Y=\varnothing$ (i.e., $X$ and $Z$ are conditional independent), as addressed by the second and third cases in the proof. In general, for a generic probability graphical model with an underlying graph $G$ (there may be directed and undirected edges in $G$), let $P$ and $Q$ be two distributions characterized by such graphical model. If $\D$ satisfy subadditivity on Markov Chain $X\to Y\to Z$ with conditionally independent variables $X$ and $Y$, we can obtain a subadditivity upper-bound on $\D\PQ{}$ by the following procedure:
\begin{enumerate}
    \item Choose an ordering of nodes $(1,\cdots,n)$. The ordering is valid if the induction can be proceeded form start to end.
    \item For node $k=1,\cdots,n-1$, let $\Sigma_k$ be the smallest set of nodes such that $\Sigma_k\subsetneqq\{k+1,\cdots,n\}$ and $X_k$ is conditionally independent of $\cup_{i=k+1}^n X_i$ given $X_{\Sigma_k}$, which can be written as $X_k\bigCI\cup_{i=k+1}^n X_i\:|\:X_{\Sigma_k}$. If we cannot find such $\Sigma_k$, the ordering $(1,\cdots,n)$ is invalid and the induction cannot be proceeded. Applying the subadditivity of $\delta$ on the Markov Chain of super-nodes $X_{\{k+1,\cdots,n\}\setminus\Sigma_k}\to X_{\Sigma_k}\to X_k$ gives an inequality $\D\PQ{\cup_{i=k}^n X_i}\leq\D\PQ{\cup_{i=k+1}^n X_i}+\D\PQ{X_{\Sigma_k}\cup X_k}$.
    \item By combining all the inequalities obtained, we get a subadditivity upper-bound $\sum_{i=1}^n\D\PQ{X_{\Sigma_i}\cup X_i}\geq\D\PQ{}$.
\end{enumerate}
This process is identical to the proof of \cref{thm:subadditivity-markov-chain} for Bayes-nets, except that 
\begin{inparaenum}[(1)]
    \item we have to manually choose a valid ordering of nodes, and 
    \item the set of parents $\Pi_k$ is replaced by the smallest set of nodes $X_{\Sigma_k}\subsetneqq\{k+1,\cdots,n\}$ such that $X_k\bigCI\cup_{i=k+1}^n X_i\:|\:X_{\Sigma_k}$, which depends on the ordering we choose.
\end{inparaenum}
For Bayes-nets, the ordering we use is the reversed topological ordering, and for each $k$, we have $\Sigma_k=\Pi_k$.

\subsection{BFS-Subadditivity on MRFs and its Application to Sequences of Words}
\label{appendix:subsec:bfs-subadditivity-mrf}

\begin{figure}[ht]
    \centering
    
        \begin{subfigure}{.40\linewidth}
        \centering
        \includegraphics[width=\linewidth]{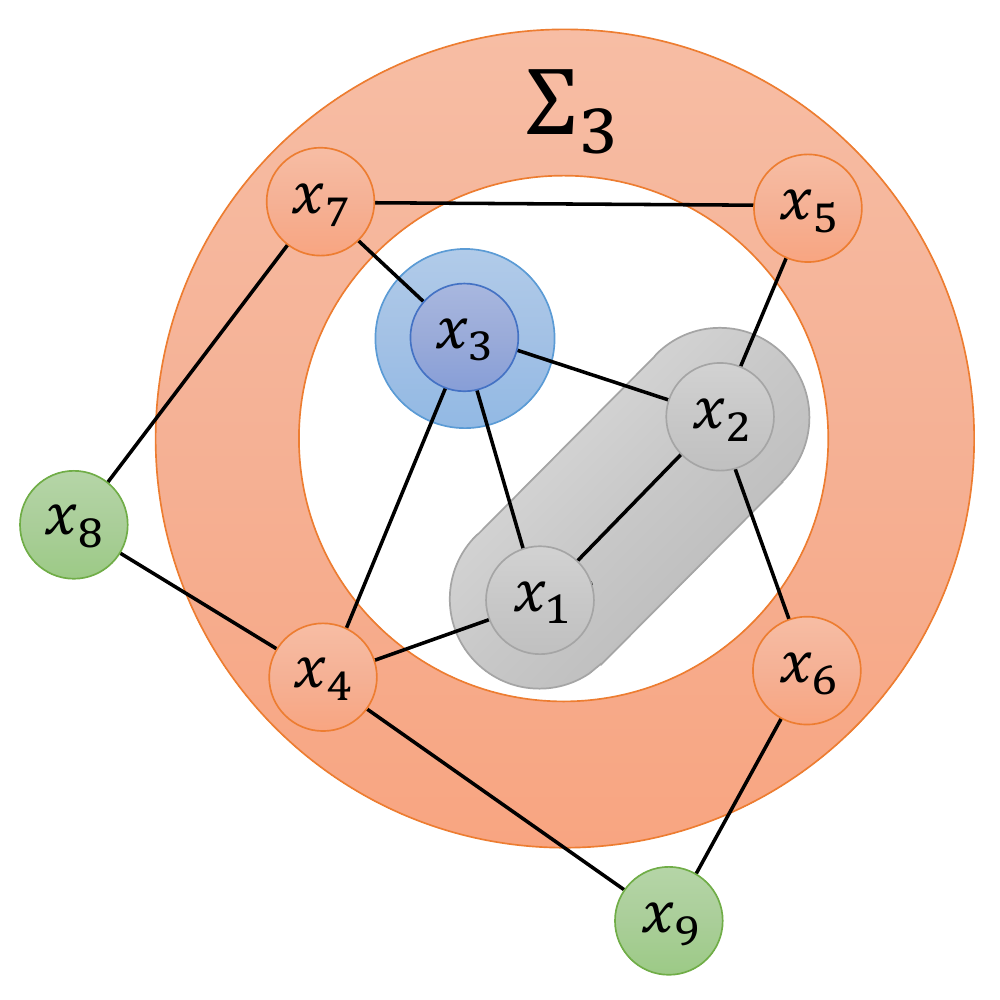}
        \caption{A MRF with $9$ variables.}
        \label{fig:mrf-graph}
    \end{subfigure}
    \hspace{.05\linewidth}
    \begin{subfigure}{.45\linewidth}
        \centering
        \includegraphics[width=\linewidth]{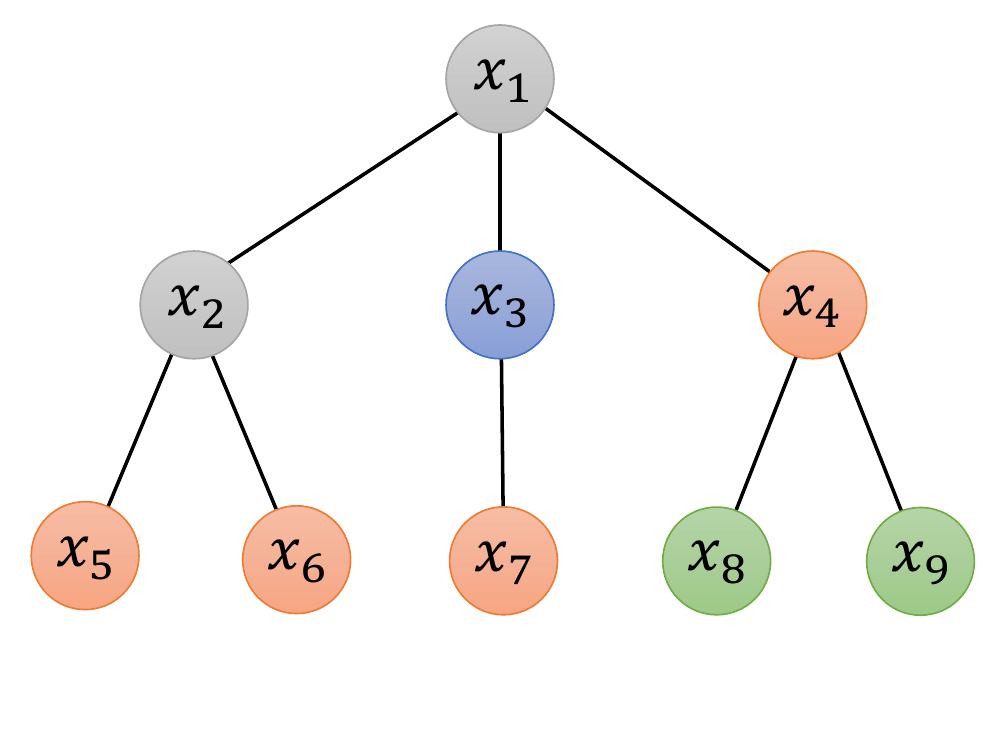}
        \vspace{13pt}
        \caption{A BFS tree of the MRF.}
        \label{fig:mrf-bfs-tree}
    \end{subfigure}
    \caption{A local neighborhood according to BFS-subadditivity, $\{3\}\cup\Sigma_3$, of a MRF with $9$ variables, if the BFS ordering $(1,\cdots,9)$ is used. Where (a) is the MRF and (b) is the corresponding BFS tree. It is a snapshot of the induction process at $k=3$. Where the gray nodes have been processed, the blue node is the current focus, the orange nodes represent the separating subset $\Sigma_3$, which is the smallest subset such that $X_3\protect\bigCI\cup_{i=4}^9 X_i\:|\:X_{\Sigma_3}$, and the green nodes are the rest.}
    \label{fig:mrf}
\end{figure}

Let us now illustrate this process on MRFs, whose underlying probability structure is described by undirected graphs. An enumeration of the nodes of a graph $G$ is said to be a BFS ordering if it is a possible output of the BFS algorithm on this graph. If we use a BFS ordering $(1,\cdots,n)$, then it is not hard to prove that for any $k\in\{1,\cdots,n\}$, we have $\Sigma_k=\cup_{i=1}^k N_i\setminus\{1,\cdots,k\}$, where $N_i$ is the set of nodes adjacent to node $i$ (i.e. the set of nearest neighbors). As shown in \cref{fig:mrf}, if we choose a BFS ordering, $\Sigma_k$ is actually the smallest set of nodes that surround the current and processed nodes $\{1,\cdots,k\}$. $\Sigma_k$ is called a separating subset between $\{1,\cdots,k\}$ and $\{k+1,\cdots,n\}\setminus\Sigma_k$, as every path from a node in $\{1,\cdots,k\}$ to a node in $\{k+1,\cdots,n\}\setminus\Sigma_k$ passes through $\Sigma_k$. By the global Markov property of MRFs, we indeed have $X_k\bigCI\cup_{i=k+1}^n X_i\:|\:X_{\Sigma_k}$.

As an example, we may consider a particular type of MRFs: sequences with local dependencies but no natural directionality, e.g., sequences of words. If we assume the distribution of a word depends on both the pre- and post- context, and consider up to $(2p+1)$-grams (i.e. consider the distribution of up to $2p+1$ consecutive words), the corresponding MRF is an undirected graph $G$, where each node $i$ is connected to its $p$ previous nodes and $p$ subsequent nodes. Let $(1,\cdots,n)$ be the natural ordering of these $n$ words. Clearly, both $(1,\cdots,n)$ and $(n,\cdots,1)$ are valid BFS orderings. Following the method above, and if we truncate the induction at step $k=n-p$ (see \cref{appendix:subsec:lod-truncation} for details), these two orderings result in an identical subadditivity upper bound $\sum_{k=1}^{n-p}\D\PQ{\cup_{i=k}^{k+p}X_i}$. Each local neighborhoods contains $p+1$ consecutive words. Equipped with this theoretical-justified subadditivity upper-bound, we can use a set of local discriminators in GANs, each on a subsequence of $p+1$ consecutive words. This is how we apply local discriminators to sequences of words.

%%%%%%%%%%%%%%%%%%%%%%%%%%%%%%%%%%%%%%%%%%%%%%%%%%%%%%%%%%%%%%%%%%%%%%%%%%%%%%%%%%%%%%%%%%%%%%%%%%%%
\section{A Counter-Example for the Subadditivity of \texorpdfstring{$2$}{2}-Wasserstein Distance}
\label{appendix:counter-cexp-wasserstein}

In this section, we report a counter-example for the subadditivity of $2$-Wasserstein distance using Gaussian distributions in $\RR^3$. Note that as we shown in \cref{coro:linear-subadditivity-wp}, in a finite space $\Sp$, $2$-Wasserstein distance satisfies $(\sqrt{2}\diam(\Sp)/d_{\min})$-linear subadditivity on Bayes-nets, where $\diam(\Sp)$ is the diameter and $d_{\min}$ is the smallest distance between pairs of distinct points in $\Sp$. However the counter-example in this section shows that, in an arbitrary metric space $\Sp$, $2$-Wasserstein distance does not satisfy subadditivity (with linear coefficient $\ap=1$) on Bayes-nets and MRFs.

Consider an non-degenerate 3-dimensional Gaussian with zero mean $P=\cN(\mathbf{0}, C)$ on variables $(X, Y, Z)$ ($C\in\RR^{3\times3}$ is the covariance matrix), which are also Bayes-nets with structure $X\to Y\to Z$. From the definition of Bayes-nets: each variable is conditionally independent of its non-descendants given its parents, we know $P$ is a Bayes-net if and only if for any $x, y, z\in\RR$, it holds that $P_{Z|X,Y}(z|x,y)=P_{Z|Y}(z|y)$. Let $C_{ij}$ denote the element of $C$ at the $i$-th row and $j$-th column. It is not hard to compute that,
\[
    \begin{split}
        P_{Z|Y}(z|y)&=\cN\left(\frac{C_{32}}{C_{22}}y,C_{33}-\frac{C_{32}C_{23}}{C_{22}}\right)\\
        P_{Z|X,Y}(z|x,y)&=\cN\left(\begin{bmatrix}C_{31}&C_{32}\end{bmatrix}\begin{bmatrix}C_{11}&C_{12}\\C_{21}&C_{12}\end{bmatrix}^{-1}\begin{bmatrix}x\\y\end{bmatrix}, C_{33}-\begin{bmatrix}C_{31}&C_{32}\end{bmatrix}\begin{bmatrix}C_{11}&C_{12}\\C_{21}&C_{12}\end{bmatrix}^{-1}\begin{bmatrix}C_{13}\\C_{23}\end{bmatrix}\right)\\
    \end{split}
\]
Matching the means and variances of these two 1-dimensional Gaussians of $z$, we know that the two conditional distributions coincide, and therefore $P$ is a Bayes-net, if and only if $C_{32}C_{21}=C_{31}C_{22}$, i.e. the $2\times2$ upper-right (or equivalently, the lower-left) sub-matrix of $C$ has zero determinant. This condition can also be written as $\var[Y]\cov[X,Z]=\cov[X,Y]\cov[Y,Z]$.

It is clear that this condition on the covariance matrix $C$ is symmetric under switching variables $X$ and $Z$. This means $P_{X|Y,Z}(x|y,z)=P_{X|Y}(x|y)$ holds simultaneously, and the most appropriate graphical model to describe $P$ is the MRF. However, as long as the Markov property $P_{Z|X,Y}(z|x,y)=P_{Z|Y}(z|y)$ holds, $P$ is a valid Bayes-net. These 3-dimensional Gaussians are special, as they satisfy the definitions of both Bayes-nets and MRFs.

Based on the discussions above, we construct two 3-dimensional Gaussians $P$ and $Q$ that are valid Bayes-nets and MRFs, as follows.

\begin{counterexample}
\label{cexp:w2-gaussian}
    Consider two 3-dimensional Gaussians $P^x=\cN(\zV,C_1)$ and $Q^{xy}=\cN(\zV,C_2)$ in $\Sp=\RR^3$ parametrized by $(x,y)\in\{(x,y)\in\RR^2|0<x, y<1\}$, where,
    \[
    C_1=
        \begin{bmatrix}
        1 & x & 0 \\
        x & 1 & 0 \\
        0 & 0 & 1
        \end{bmatrix}
    \qquad
    C_2=
        \begin{bmatrix}
        1 & x & xy \\
        x & 1 & y \\
        xy & y & 1
        \end{bmatrix}
    \]
    and $\zV\in\RR^3$ is the zero vector. The two distributions are valid Bayes-nets and MRFs with structure $X\to Y\to Z$ (when considered as Bayes-nets) or $X$--$Y$--$Z$ (when considered as MRFs), since the $2\times2$ upper-right (or lower-left) sub-matrices of $C_1$ and $C_2$ has zero determinants. The $2$-Wasserstein distance between them, $W_2(P^x,Q^{xy})$, depends on parameters $(x, y)$. For any $(x,y) \in \{(x,y)\in\RR^2|0<x, y<1\}$, it holds that $\W_2(P^x_{XYZ}, Q^{xy}_{XYZ})>\W_2(P^x_{XY}, Q^{xy}_{XY}) + \W_2(P^x_{YZ}, Q^{xy}_{YZ})$, which violets the subadditivity inequality (with linear coefficient $\ap=1$) of $2$-Wasserstein distance on Bayes-nets and MRFs.
\end{counterexample}

\begin{figure}[ht]
    \centering
    \includegraphics[trim=0.3in 0.2in 0.3in 0.5in, clip, width=.50\linewidth]{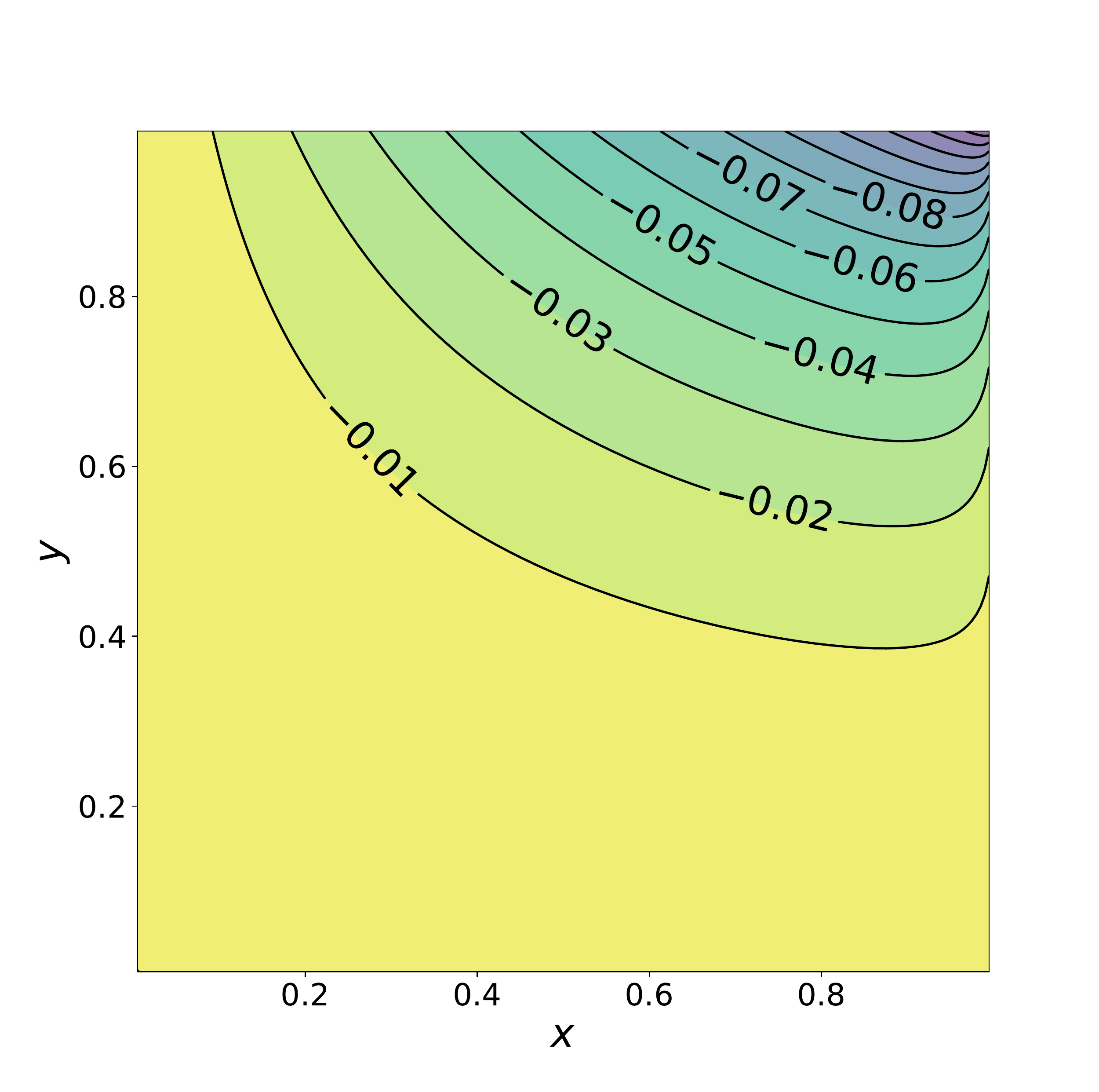}
    \caption{Contour maps showing the counter-example for the subadditivity of $2$-Wasserstein distance. The two distributions $P^x, Q^{xy}$ are 3-dimensional Gaussians $P^x=\cN(\zV,C_1)$, $Q^{xy}=\cN(\zV,C_2)$ which are valid Bayes-nets and MRFs. The contours and colors indicate the subadditivity gap $\Delta=\W_2(P^x_{XY}, Q^{xy}_{XY})+\W_2(P^x_{YZ}, Q^{xy}_{YZ})-\W_2(P^x_{XYZ}, Q^{xy}_{XYZ})$.}
    \label{fig:cexp-gaussian-w2}
\end{figure}

\cref{cexp:w2-gaussian} can be numerically verified, as the $2$-Wasserstein distance between Gaussians can be exactly computed using the formula in \cref{thm:w2-formula-gaussian} in \cref{appendix:subsec:wassserstein-formulas}. As shown in \cref{fig:cexp-gaussian-w2}, the subadditivity gap $\Delta=\W_2(P^x_{XY}, Q^{xy}_{XY})+\W_2(P^x_{YZ}, Q^{xy}_{YZ})-\W_2(P^x_{XYZ}, Q^{xy}_{XYZ})$ is negative for any $(x,y) \in \{(x,y)\in\RR^2|0<x, y<1\}$, thus the subadditivity inequality is violated.

This straightforward but fundamental counter-example shows that Wasserstein's subadditivity does not hold even if all distributions are Gaussians. For many common divergences including Jensen-Shannon divergence, Total Variation distance, and $p$-Wasserstein distance, the best we can prove is linear subadditivity.

%%%%%%%%%%%%%%%%%%%%%%%%%%%%%%%%%%%%%%%%%%%%%%%%%%%%%%%%%%%%%%%%%%%%%%%%%%%%%%%%%%%%%%%%%%%%%%%%%%%%
\section{Local Subadditivity}
\label{appendix:local}

In this section, we consider the case when two distributions $P, Q$ are close to each other. This can happen after some training steps in a GAN. We consider two notions of ``closeness'' for distributions.

\begin{definition}[One- and Two-Sided $\eps$-Close Distributions]
%\label{def:close-distributions}
    Distributions $P, Q$ are one-sided $\eps$-close for some $0<\eps<1$, if $\forall x\in\Sp\subseteq\RR^{nd}$, $P(x)/Q(x)<1+\eps$. Moreover, $P, Q$ are two-sided $\eps$-close, if $\forall x$, $1-\eps<P(x)/Q(x)<1+\eps$. Note this requires $P\ll\gg Q$.
\end{definition}

\subsection{Local Subadditivity under Perturbation}
\label{appendix:subsec:local-perturbation}

For the sake of theoretical simplicity, we consider the limit $\eps\to0$ for two-sided $\eps$-close distributions. We call $Q$ a perturbation of $P$ \citep{makur2015study}.

\begin{theorem}
\label{thm:local-perturbation-suadditivity}
    For two-sided $\eps$-close distributions $P, Q$ with $\eps\to 0$ on a common Bayes-net $G$, any $\gf$-divergence $\PD\PQ{}$ such that $\gf''(1)>0$ has subadditivity up to $\cO(\eps^3)$. That is,
    \[
        \PD\PQ{}\leq \sum_{i=1}^n \PD\PQ{X_i\cup X_{\Pi_i}} + \cO(\eps^3)
    \]
    Moreover, the subadditivity gap is proportional to the sum of $\CS$ divergences between marginals on the set of parents of each node, up to $\cO(\eps^3)$. That is,
    \[
        \Delta = \sum_{i=1}^n \PD\PQ{X_i\cup X_{\Pi_i}} - \PD\PQ{} = \frac{\gf''(1)}{2}\sum_{i=1}^n \CS(P_{\Pi_i},Q_{\Pi_i}) + \cO(\eps^3)
    \]
\end{theorem}

\cref{thm:local-perturbation-suadditivity} indicates that when $P, Q$ are very close, the focus of the set of local discriminators falls on the differences between the marginals on the set of parents. We make use of the Taylor expansion of $f(\cdot)$ in the proof. To prove \cref{thm:local-perturbation-suadditivity}, we first prove the following lemma describing the approximation behavior of nearly all $\gf$-divergences when $P, Q$ are perturbations with respect to each other.

\begin{lemma}
\label{lem:local-fdiv-approximation}
    For two-sided $\eps$-close distributions $P, Q$ with $\eps\to 0$, any $\gf$-divergence $\PD\PQ{}$ with $\gf(t)$ twice differentiable at $t=1$ and $\gf''(1)>0$, is proportional to $\CS\PQ{}$ up to $\cO(\eps^3)$, i.e.
    \[
        \PD\PQ{}=\frac{\gf''(1)}{2}\CS\PQ{}+\cO(\eps^3)
    \]
    And $\CS$ is now symmetric up to $\cO(\eps^3)$, i.e. $\CS\PQ{}=\CS(Q,P)+\cO(\eps^3)$.
\end{lemma}

\begin{proof}
    Since $\gf(t)$ twice differentiable at $t=1$, and $P(x)/Q(x)\in(1-\eps,1+\eps)$ with $0<\eps\ll 1$, by Taylor's theorem we get,
    \[
        f(\frac{P}{Q})=\gf'(1)\left(\frac{P}{Q}-1\right)+\frac12\gf''(1)\left(\frac{P}{Q}-1\right)^2+\cO(\eps^3)
    \]
    Multiply by $Q$ and integrate over $\Sp\in\RR^{nd}$ gives,
    \[
        \begin{split}
            \PD\PQ{}&=\frac{\gf''(1)}{2}\int Q\left(\frac{P}{Q}-1\right)^2\dx+\cO(\eps^3)\\
            &=\frac{\gf''(1)}{2}\CS\PQ{}+\cO(\eps^3)\\
        \end{split}
    \]
    Where the first order term vanishes because $\int P\dx=\int Q\dx=1$. This equation implies that all $\gf$-divergences such that $\gf''(1)>0$ behave similarly when the two distributions $P$ and $Q$ are sufficiently close.
    
    Meanwhile, because $P/Q=1+\cO(\eps)$, we have,
    \[
        \begin{split}
            \CS\PQ{}&=\int\frac{(P-Q)^2}{P}\frac{P}{Q}\dx\\
            &=\int\frac{(P-Q)^2}{P}(1+\cO(\eps))\dx\\
            &=\CS(Q,P)+\cO(\eps^3)
        \end{split}
    \]
    Thus we can exchange $P$ and $Q$ freely in any $\cO(\eps^2)$ terms (e.g. $(P-Q)^2/Q$), while preserving the equality up to $\cO(\eps^3)$.
\end{proof}

Based on \cref{lem:local-fdiv-approximation}, \cref{thm:local-perturbation-suadditivity} can be proved by comparing an $\gf$-divergence with the squared Hellinger distance.

\begin{prevproof}{thm:local-perturbation-suadditivity}
    We first prove that the subadditivity inequality holds using \cref{lem:local-fdiv-approximation}. Define $R(x)=\frac12\left(\sqrt{PQ}+\frac{P+Q}{2}\right)$ as the average of the geometric and arithmetic means of $P$ and $Q$. Clearly for any $x\in\Sp$, it holds that $|R(x)-Q(x)|<|P(x)-Q(x)|<\eps$. Thus $R/Q=1+\cO(\eps)$, and by \cref{lem:local-fdiv-approximation}, we have,
    \[
        \begin{split}
            \PD\PQ{}&=\frac{\gf''(1)}{2}\CS\PQ{}+\cO(\eps^3)\\
            &=\frac{\gf''(1)}{2}\int\frac{(P-Q)^2}{R}\frac{R}{Q}\dx+\cO(\eps^3)\\
            &=\frac{\gf''(1)}{2}\int\frac{(P-Q)^2}{R}\dx+\cO(\eps^3)\\
            &=2\gf''(1)\int\left(\sqrt{P}-\sqrt{Q}\right)\dx+\cO(\eps^3)\\
            &=4\gf''(1)\SH\PQ{}+\cO(\eps^3)
        \end{split}
    \]
    Since $\gf''(1)>0$, we can re-write this equation as $\SH\PQ{}=\frac{1}{4\gf''(1)}\PD\PQ{}+\cO(\eps^3)$. Applying this formula to both sides of the subadditivity inequality of $\SH$ (\cref{thm:subadditivity-sh}): $\SH\PQ{}\leq \sum_{i=1}^n \SH\PQ{X_i\cup X_{\Pi_i}}$, we conclude that the subadditivity inequality holds up to $\cO(\eps^3)$:
    \[
        \PD\PQ{}\leq \sum_{i=1}^n \PD\PQ{X_i\cup X_{\Pi_i}} + \cO(\eps^3)
    \]
    
    Then, we prove that the subadditivity gap $\Delta\coloneqq\sum_{i=1}^n \PD\PQ{X_i\cup X_{\Pi_i}} - \PD\PQ{}$ is proportional to $\sum_{i=1}^n \CS(P_{\Pi_i},Q_{\Pi_i})$ up to $\cO(\eps^3)$ using a different approach. Let us start from the simple case when $P, Q$ are Markov Chains with structure $X\to Y\to Z$. The Markov property $P_{Z|XY}=P_{Z|Y}$ holds (and the same for $Q$). Since the joint distributions $P_{XYZ}$ and $Q_{XYZ}$ are two-sided $\eps$-close, so are the marginal and conditional distributions. We define the differences between the marginals and conditionals of $P$ and $Q$ as follows,
    \[
        \begin{split}
            &Q_{X|Y} = P_{X|Y} + \eps J_{X|Y}\\
            &Q_{Y} = P_{Y} + \eps J_{Y}\\
            &Q_{Z|Y} = P_{Z|Y} + \eps J_{Z|Y}\\
        \end{split}
    \]
    Clearly $\int J_{X|Y}\dx=\int J_{Y}\dy=\int J_{Z|Y}\dz=0$.
    Using \cref{lem:local-fdiv-approximation}, we have,
    \[
        \begin{split}
            &\frac{2}{\eps^2\gf''(1)}\PD\PQ{XYZ}+\cO(\eps)\\
            &=\frac{1}{\eps^2}\int\frac{(P_{XYZ}-Q_{XYZ})^2}{P_{XYZ}}\dx\dy\dz\\
            &=\frac{1}{\eps^2}\int\frac{\left(P_{X|Y}P_{Y}P_{Z|Y}-Q_{X|Y}Q_{Y}Q_{Z|Y}\right)^2}{P_{X|Y}P_{Y}P_{Z|Y}}\dx\dy\dz\\
            &=\int\Bigg(\frac{J_{Y}^2P_{X|Y}P_{Z|Y}}{P_{Y}}+\frac{J_{X|Y}^2P_{Y}P_{Z|Y}}{P_{X|Y}}+\frac{J_{Z|Y}^2P_{X|Y}P_{Y}}{P_{Z|Y}}\\
            &\qquad+2J_{X|Y}J_{Y}P_{Z|Y}+2J_{Y}J_{Z|Y}P_{X|Y}+2J_{X|Y}J_{Z|Y}P_{Y}\Bigg)\dx\dy\dz\\
            &=\int\frac{J_{Y}^2}{P_{Y}}\dy+\int\frac{J_{X|Y}^2P_{Y}}{P_{X|Y}}\dx\dy+\int\frac{J_{Z|Y}^2P_{Y}}{P_{Z|Y}}\dy\dz\\
        \end{split}
    \]
    Similarly,
    \[
        \begin{split}
            \frac{2}{\eps^2\gf''(1)}\PD\PQ{XY}+\cO(\eps)&=\frac{1}{\eps^2}\int\frac{(P_{X|Y}P_{Y}-Q_{X|Y}Q_{Y})^2}{P_{X|Y}P_{Y}}\dx\dy\\
            &=\int\left(\frac{J_{X|Y}^2P_{Y}}{P_{X|Y}}+\frac{J_{Y}^2P_{X|Y}}{P_{Y}}+2J_{X|Y}J_{Y}\right)\dx\dy\\
            &=\int\frac{J_{X|Y}^2P_{Y}}{P_{X|Y}}\dx\dy+\int\frac{J_{Y}^2}{P_{y}}\dy
        \end{split}
    \]
    And,
    \[
        \frac{2}{\eps^2\gf''(1)}\PD\PQ{YZ}+\cO(\eps)=\int\frac{J_{Z|Y}^2P_{Y}}{P_{Z|Y}}\dy\dz+\int\frac{J_{Y}^2}{P_{Y}}\dy
    \]
    Thus, the subadditivity gap on the Markov Chain $X\to Y\to Z$ is,
    \[
        \begin{split}
            \Delta_\text{Markov Chain}&=\PD\PQ{XY}+\PD\PQ{YZ}-\PD\PQ{XY}\\
            &=\frac{\gf''(1)}{2}\int\frac{J_{Y}^2}{P_{Y}}\dy+\cO(\eps^3)\\
            &=\frac{\gf''(1)}{2}\CS(P_{Y},Q_{Y})+\cO(\eps^3)
        \end{split}
    \]
    Moreover, consider the special case when $Y=\varnothing$, thus $P, Q$ are product measures on conditionally independent variables $X$ and $Z$. Similarly, we have,
    \[
        \frac{2}{\eps^2\gf''(1)}\PD\PQ{XZ}+\cO(\eps)=\CS\PQ{X}+\CS\PQ{Z}
    \]
    Hence the subadditivity gap is,
    \[
        \Delta_\text{Product Measure}=\PD\PQ{X}+\PD\PQ{Z}-\PD\PQ{XZ}=0+\cO(\eps^3)\\
    \]
    
    Now, for any pair of generic Bayes-nets $P$ and $Q$, following the approach in the proof of \cref{thm:subadditivity-markov-chain} in \cref{proof:subadditivity-markov-chain}, we repeatedly apply the subadditivity inequality on Markov Chains of super-nodes $X_{\{1,\cdots,n-k-1\}\setminus\Pi_{n-k}}\to X_{\Pi_{n-k}}\to X_{n-k}$, for $k=0,1,\cdots,n-2$. Consider three cases:
    \begin{enumerate}
        \item $\Pi_{n-k}\neq\varnothing$ and $\Pi_{n-k}\subsetneqq \{1,\cdots,n-k-1\}$: In this case, the subadditivity gap is $\frac{\gf''(1)}{2}\CS(P_{\Pi_{n-k}},Q_{\Pi_{n-k}})+\cO(\eps^3)$.
        \item $\Pi_{n-k}=\{1,\cdots,n-k-1\}$: In this case, as discussed in \cref{proof:subadditivity-markov-chain}, we add a redundant term $\D\PQ{\cup_{i=1}^{n-k-1}X_i}\equiv\D\PQ{\Pi_{n-k}}$ into the subadditivity upper-bound. Thus, by \cref{lem:local-fdiv-approximation}, the subadditivity gap is $\frac{\gf''(1)}{2}\CS(P_{\Pi_{n-k}},Q_{\Pi_{n-k}})+\cO(\eps^3)$
        \item $\Pi_{n-k}=\varnothing$: In this case, $X_{n-k}$ is independent from $(X_1,\ldots,X_{n-k-1})$ in both Bayes-nets. Thus the subadditivity gap is $0$.
    \end{enumerate}
    For all the three cases, the subadditivity gap at an induction step $k$ is $\frac{\gf''(1)}{2}\CS(P_{\Pi_{n-k}},Q_{\Pi_{n-k}})+\cO(\eps^3)$ (note that $\CS(P_{\Pi_{n-k}},Q_{\Pi_{n-k}})=0$ when $\Pi_{n-k}=\varnothing$). Along with the induction process for $k=0,1,\cdots,n-2$, the subadditivity gaps accumulate, and we finally get,
    \[
        \Delta\coloneqq\sum_{i=1}^n \PD\PQ{X_i\cup X_{\Pi_i}} - \PD(P, Q)=\frac{\gf''(1)}{2}\sum_{i=1}^n \CS(P_{\Pi_i},Q_{\Pi_i}) + \cO(\eps^3)
    \]
\end{prevproof}

\subsection{Linear Subadditivity for Close Distributions}
\label{appendix:subsec:local-linear}

Now, we consider distributions that are one or two-sided $\eps$-close with a non-infinitesimal $\eps>0$. This is a more realistic setup compared to the setup in \cref{appendix:subsec:local-perturbation}. The Taylor expansion approach used there is no longer applicable. However, using the methodology to prove general $\gf$-divergence inequalities (\cref{lem:fdiv-general-ineq}), and a technique of equivalent $\gf$-divergences, we are able to obtain linear subadditivity for both cases, under very mild conditions.

We first prove a lemma which reveals the connection between the notion of closeness and linear subadditivity.

\begin{lemma}
\label{lem:local-linear-subadditivity-fdiv}
    Consider two $\gf$-divergences $D_{\gf_1}$ and $D_{\gf_2}$ with generator functions $\gf_1(t)$ and $\gf_2(t)$, where $\gf_2$ has subadditivity on Bayes-nets with respect to \cref{def:subadditivity}. Let $I\subseteq(0,\infty)$ be an interval. If there exists two positive constants $A<B$, such that for any $t\in I$, it holds that $\gf_2(t)\geq 0$ and $A\leq\gf_1(t)/\gf_2(t)\leq B$. Then, for any pair of distributions $P$ and $Q$, such that for any $x\in\Sp$, $P(x)/Q(x)\in I$, the linear subadditivity inequality of $D_{\gf_1}$ holds with coefficient $0<\ap=A/B<1$.
\end{lemma}

\begin{proof}
    For any $t\in I$, multiplying $\gf_2(t)\geq 0$ to the inequalities $A\leq\gf_1(t)/\gf_2(t)\leq B$ gives,
    \[
        A\gf_2(t)\leq \gf_1(t)\leq B\gf_2(t) \qquad \forall t\in I
    \]
    Similar to the proof of \cref{lem:fdiv-general-ineq} in \cref{appendix:subsec:fdiv-ineq}, since for any $x\in\Sp$, it holds that $P(x)/Q(x)\in I$, we have,
    \[
        A\gf_2(P(x)/Q(x))\leq \gf_1(P(x)/Q(x))\leq B\gf_2(P(x)/Q(x)) \qquad \forall x\in\Sp
    \]
    Multiply non-negative $Q(x)$ and integrate over $\Sp$. Thus, for such pairs of $P, Q$, we obtain, 
    \[
        AD_{\gf_2}\PQ{}\leq D_{\gf_1}\PQ{}\leq BD_{\gf_2}\PQ{}
    \]
    
    Now consider $P, Q$ are Bayes-nets such that for any $x\in\Sp$, $P(x)/Q(x)\in I=[a, b]$, i.e. $a\leq P(x)/Q(x)\leq b$. For any non-empty set $S\subsetneqq\{X_1,\cdots,X_n\}$, let $\Sp_{\{X_1,\cdots,X_n\}\setminus S}$ be the space of the variables not in $S$. Then, multiplying non-negative $Q(x)$ to $a\leq P(x)/Q(x)\leq b$ and integrating over $\Sp_{\{X_1,\cdots,X_n\}\setminus S}$ gives $aQ_S\leq P_S\leq bQ_S$. Moreover, $Q_S$ is positive because $Q$ is positive. Thus, for any pair of marginal distributions $P_S$ and $Q_S$ of such distributions, they also satisfy that for any $x\in\Sp_S$, $P_S(x)/Q_S(x)\in I=[a, b]$.
    
    Applying the first inequality to pairs of marginals $P_{X_i\cup X_{\Pi_i}}$ and $Q_{X_i\cup X_{\Pi_i}}$ gives,
    \[
        D_{\gf_2}\PQ{X_i\cup X_{\Pi_i}}\leq \frac{1}{A} D_{\gf_1}\PQ{X_i\cup X_{\Pi_i}} \qquad \forall i\in\{1,\cdots,n\}
    \]
    Similarly, applying the second inequality to $P$ and $Q$ gives,
    \[
        \frac{1}{B}D_{\gf_1}\PQ{}\leq D_{\gf_2}\PQ{}
    \]
    Combine them with the subadditivity inequality of $D_{\gf_2}$, i.e. $D_{\gf_2}\PQ{}\leq \sum_{i=1}^n D_{\gf_2}\PQ{X_i\cup X_{\Pi_i}}$, we have,
    \[
        \frac{A}{B}D_{\gf_1}\PQ{}\leq \sum_{i=1}^n D_{\gf_1}\PQ{X_i\cup X_{\Pi_i}}
    \]
    This proves that $D_{\gf_1}$ satisfy $A/B$-linear subadditivity for such pairs of Bayes-nets $P$ and $Q$.
\end{proof}

Now, we list the two theorems characterizing the linear subadditivity of $\gf$-divergences when the distributions are one- or two-sided $\eps$-close.

\begin{theorem}
%\label{thm:local-linear-subadditivity-two-sided-close}
    An $\gf$-divergence whose $\gf(\cdot)$ is continuous on $(0,\infty)$ and twice differentiable at $1$ with $\gf''(1)>0$, satisfies $\ap$-linear subadditivity, when $P, Q$ are two-sided $\eps(\ap)$-close with $\eps>0$, where $\eps(\ap)$ is a non-increasing function and $\lim_{\eps\downarrow0}\ap=1$.
\end{theorem}

\begin{proof}
    Following \cref{lem:local-linear-subadditivity-fdiv}, we consider the quotient $\gf(t)/\gf_{\SH}(t)$, where $\gf_{\SH}$ is the generator function of squared Hellinger distance, and $\gf_{\SH}(t)\coloneqq\frac12\left(\sqrt{t}-1\right)^2\geq0$ is always non-negative. If we can bound this quotient by positive numbers on an interval $t\in(1-\eps,1+\eps)$ for some $0<\eps<1$, then by \cref{lem:local-linear-subadditivity-fdiv}, we prove that $\PD$ satisfies linear subadditivity when the distributions $P$ and $Q$ are two-sided $\eps$-close.

    Because $\gf(t)$ and $\gf_{\SH}(t)$ are continuous functions on $(0, \infty)$, the quotient $\gf(t)/\gf_{\SH}(t)$ is also continuous on $(0,\infty)$. To bound the quotient in the neighborhood around $t=1$, we need to prove $\lim_{t\to1}\gf(t)/\gf_{\SH}(t)$ exists and is positive. For $\gf_{\SH}$, we know $\gf'_{\SH}(1)=\frac12\left(1-1/\sqrt{t}\right)\big|_{t=1}=0$ and $\gf''_{\SH}(1)=\frac14 t^{-3/2}\Big|_{t=1}=\frac14>0$. Thus, since $\gf(t)$ is twice differentiable at $t=1$, the limit of the quotient at $t=1$ exists and is positive if and only if $\gf'(1)=0$ and $\gf''(1)>0$. That is,
    \[
        0<\lim_{t\to1}\gf(t)/\gf_{\SH}(t)<\infty \iff \gf'(1)=0 \text{ and } \gf''(1)>0
    \]
    The latter condition is given, but the former condition, $\gf'(1)=0$, does not hold even for some $\gf$-divergences which satisfy subadditivity on any Bayes-nets, e.g. for $\KL$ divergence, $\gf'_{\KL}(1)=1+\log(t)\big|_{t=1}=1\neq 0$.
    
    However a trick can be used to rewrite the generator function $\gf(t)$ without changing the definition of $\PD$, so that the modified generator function satisfies the desired condition. For any $k\in\RR$, the modified generator $\hp(t) = \gf(t)+k(t-1)$ defines the same $\gf$-divergence,
    \[
        \begin{split}
            D_{\hp}\PQ{}&=\int Q\hp\left(\frac{P}{Q}\right)\dx=\int Q\left(\gf \left(\frac{P}{Q}\right)+k\left(\frac{P}{Q}-1\right)\right)\dx\\
            &=\int Q\gf\left(\frac{P}{Q}\right)\dx+k\int(P-Q)\dx=\PD\PQ{}
        \end{split}
    \]
    Thus, for any $\gf(t)$ twice differentiable at $t=1$ with $\gf''(1)>0$, we can define $\hp(t)\coloneqq\gf(t)-\gf'(1)(t-1)$. It is easy to verify that $\hp(t)$ has zero first derivative $\hp'(1)=0$ and positive second derivative  $\hp''(1)>0$ at $t=1$. The modified generator satisfies the two required conditions. As a consequence, we have $0<\lim_{t\to1}\hp(t)/\gf_{\SH}(t)<\infty$, and the quotient can be bounded by positive numbers in the neighborhood of $t=1$, because of the continuity of $\gf(t)$. Applying \cref{lem:local-linear-subadditivity-fdiv} to interval $I=(1-\eps, 1+\eps)$ concludes the proof.
\end{proof}

\cref{thm:local-linear-subadditivity-two-sided-close} applies to all practical $f$-divergences, including $\KL$, reverse $\KL$, $\CS$, reverse $\CS$, and squared Hellinger $\SH$ divergences.

In addition to the requirements of \cref{thm:local-linear-subadditivity-two-sided-close}, if $\gf(\cdot)$ is also strictly convex and $\gf(0)=\lim_{t\downarrow0}\gf(t)$ is finite, $\forall t\in[0,1)$, we have the following subadditivity result for one-sided close distributions.

\begin{theorem}
%\label{thm:local-linear-subadditivity-one-sided-close}
    An $\gf$-divergence whose $\gf(\cdot)$ is continuous and strictly convex on $(0,\infty)$, twice differentiable at $t=1$, and has finite $\gf(0)=\lim_{t\downarrow0}\gf(t)$, has linear subadditivity with coefficient $\ap>0$, when $P, Q$ are one-sided $\eps(\ap)$-close with $\eps>0$, where $\eps(\ap)$ is an non-increasing function and $\lim_{\eps\downarrow0}\ap>0$.
\end{theorem}

\begin{proof}
    From the proof of \cref{thm:local-linear-subadditivity-two-sided-close}, let $\hp(t)\coloneqq\gf(t)-\gf'(1)(t-1)$ be the modified generator function. We know the quotient $\hp(t)/\gf_{\SH}(t)$ can be bounded by positive numbers for any $t\in (1-\eps, 1+\eps)$ for some $0<\eps<1$. It remains to prove that $\hp(t)/\gf_{\SH}(t)$ can be bounded by positive numbers on the interval $[0, 1-\eps)$.
    
    The generator $\gf(t)$ is a strictly convex function on $(0,\infty)$, so is the modified generator $\hp(t)$, since their difference is a linear function of $t$. Because $\hp'(1)=0$, the tangent line of the curve of $\hp(t)$ at $t=1$ coincides with the x-axis. Since $\hp(t)$ is strictly convex on $(0,\infty)$, the graph of $\hp(t)$ lies above the x-axis, i.e. for any $t\in(0,\infty)$ we have $\hp(t)\geq 0$, where the equality holds if and only if $t=1$. Hence, for any $t\in [0,1-\eps)$, it holds that $\hp(t)>0$. Moreover, $\hp(0)=\gf(0)+\gf'(1)$ and we know $\gf(0)=\lim_{t\downarrow0}\gf(t)$ is finite. In this sense, $\hp(0)$ is finite and positive. By the continuity of the modified generator $\hp(t)$, we know $\hp(t)$ can be bounded by positive numbers on $[0,1-\eps)$. Moreover, clearly $\gf_{\SH}(t)\coloneqq\frac12\left(\sqrt{t}-1\right)^2$ can be bounded by positive numbers $[0,1-\eps)$. This implies that the quotient $\hp(t)/\gf_{\SH}(t)$ can be bounded by positive numbers on $[0,1-\eps)$. Applying \cref{lem:local-linear-subadditivity-fdiv} to the combined interval $I=[0, 1+\eps)=[0, 1-\eps)\cup\{\eps\}\cup(1-\eps, 1+\eps)$ concludes the proof.
\end{proof}

Using \cref{thm:local-linear-subadditivity-one-sided-close}, we can relax the condition $P\gg Q$, as long as $\gf(0)<\infty$ and $\gf(\cdot)$ is strictly convex. A broad class of $\gf$-divergences satisfy this; see \cref{appendix:local-exp} below.

%%%%%%%%%%%%%%%%%%%%%%%%%%%%%%%%%%%%%%%%%%%%%%%%%%%%%%%%%%%%%%%%%%%%%%%%%%%%%%%%%%%%%%%%%%%%%%%%%%%%
\section{Examples of Local Subadditivity}
\label{appendix:local-exp}

In this section, we discuss a notable class of $\gf$-divergences that satisfy local subadditivity, namely the $\ap$-divergences. $\ap$-Divergences are $\gf$-divergences whose generator functions $\gf_{\HA}(\cdot)$ generalize power functions (see \cref{tab:common-f-divs} in \cref{appendix:subsec:common-fdivs}). We show that all $\ap$-divergences satisfy linear subadditivity when the distributions are two-sided close, and $\ap$-divergences with $\ap>0$ satisfy linear subadditivity when the distributions are only one-sided close.

Since for any $\ap\in\RR$, $\gf_{\HA}(t)$ is continuous with respect to $t$, and its second order derivative at $t=1$, i.e. $\gf''_{\HA}(1)=t^{\ap-2}\big|_{t=1}=1$ is positive, by \cref{thm:local-linear-subadditivity-two-sided-close} we conclude the following result.

\begin{example}
\label{exp:local-two-sided-close-alpha-div}
    $\ap$-divergences,
    \[
        \HA\PQ{}\coloneqq
        \begin{cases}
            \frac{1}{\ap(\ap-1)}\int Q\left(\left(P/Q\right)^\ap-1\right)\dx & \ap\neq0, 1\\
            \KL\PQ{} & \ap=1\\
            \KL(Q,P) & \ap=0\\
        \end{cases}
    \]
    which generalize KL and reverse KL divergences, $\CS$ and reverse $\CS$ divergences, and squared Hellinger distance (see \cref{appendix:subsec:common-fdivs} for details), satisfy linear subadditivity when the two distributions $P$ and $Q$ are two-sided $\eps$-close for some $\eps>0$.
\end{example}

For $\ap$-divergences with $\ap>0$, apart from the above-mentioned properties, $\gf_{\HA}(t)$ is strictly convex since for any $t\in(0, \infty)$, we have $\gf''_{\HA}(t)=t^{\ap-2}>0$. And $\gf(0)=\lim_{t\downarrow0}$ is always finite, because when $\ap=1$, we have $\lim_{t\downarrow0}\gf(t)=0$, and when $\ap>0$ and $\ap\neq1$, the limit $\lim_{t\downarrow0}\gf(t)=-\frac{1}{\ap(\ap-1)}$ exists. By \cref{thm:local-linear-subadditivity-one-sided-close}, we obtain the following.

\begin{example}
\label{exp:local-one-sided-close-alpha-div}
    $\ap$-divergences with $\ap>0$, which generalize KL divergence, $\CS$ divergence, and squared Hellinger distance, satisfy linear subadditivity when the two distributions $P$ and $Q$ are one-sided $\eps$-close for some $\eps>0$.
\end{example}

%%%%%%%%%%%%%%%%%%%%%%%%%%%%%%%%%%%%%%%%%%%%%%%%%%%%%%%%%%%%%%%%%%%%%%%%%%%%%%%%%%%%%%%%%%%%%%%%%%%%
\section{Prior Work on Bounding the IPMs}
\label{appendix:ipm}

We list some of the prior work on bounding the Integral Probability Metrics (IPMs). All the concepts and theorems introduced here are used to prove the generalized subadditivity of neural distances on Bayes-nets (\cref{thm:subadditivity-ipm}) and on MRFs (\cref{coro:subadditivity-ipm-mrf}).

\subsection{Preliminaries and Notations}
\label{appendix:subsec:ipm-preliminaries}

Firstly, we introduce some concepts that help us characterize the set of discriminators $\cF$. Consider $\cF$ as a set of some functions $\df:\Sp\to\RR$, where $\Sp\subseteq\RR^D$. The Banach space of bounded continuous functions is denoted by $C_b(\Sp)\coloneqq\{\df:\Sp\to\RR|\df\text{ is continuous and } \|\df\|_{\infty}<\infty\}$, where $\|\df\|_\infty=\sup_{x\in X}|\df(x)|$ is the uniform norm. The linear span of $\cF$ is defined as,
\[
    \vspan\cF\coloneqq\left\{\ap_0+\sum_{i=1}^n\ap_i\df_i\bigg|\ap_i\in\RR, \df_i\in\cF, n\in\NN\right\}
\]
For a function $g\in\vspan\cF$, we define the $\cF$-variation norm $\|g\|_\cF$ as the infimum of the $L_1$ norm of the expansion coefficients of $g$ over $\cF$, that is,
\[
    \|g\|_\cF=\inf\Bigg\{\sum_{i=1}^n|\ap_i|\bigg|g=\ap_0+\sum_{i=1}^n\ap_i\df_i,\forall\ap_i\in\RR, \df_i\in\cF, n\in\NN\Bigg\}
\]
Let $\cl(\vspan\cF)$ be the closure of the linear span of $\cF$. We say $g\in\cl(\vspan\cF)$ is approximated by $\cF$ with an error decay function $\veps(r)$ for $r\geq 0$, if there exists a $\df_r\in\vspan\cF$, such that $\|\df_r\|_\cF\leq r$ and $\|\df-\df_r\|_{\infty}\leq\veps(r)$. In this sense, it is not hard to show that $g\in\cl(\vspan\cF)$ if and only if $\inf_{r\geq 0}\veps(r)=0$.

\subsection{The Universal Approximation Theorems}
\label{appendix:subsec:ipm-universal-approximation}

From Theorem 2.2 of \citep{zhang2018on}, we know that $d_\cF(P, Q)$ is discriminative, i.e. $d_\cF\PQ{}=0 \iff P=Q$, if and only if $C_b(X)$ is contained in the closure of $\vspan\cF$, i.e. $C_b(X)\subseteq\cl(\vspan\cF)$. In other words, it means that we require $\vspan\cF$ to be dense in $C_b(X)$, so that $d_\cF\PQ{}\to0$ implies the weak converge of the fake distribution $Q$ to the real distribution $P$.

By the famous universal approximator theorem (e.g. Theorem 1 of \citep{leshno1993multilayer}), the discriminative criteria $C_b(X)\subseteq\cl(\vspan\cF)$ can be satisfied by small discriminator sets such as the neural networks with only a single neuron, $\cF=\{\sigma(w^Tx+b)|w\in\RR^D, b\in\RR\}$, if the activation function $\sigma:\RR\to\RR$ is continuous but not a polynomial. Later, \citep{bach2017breaking} proves that the set of single-neuron neural networks with rectified linear unit (ReLU) activation also satisfies the criteria.

\begin{theorem}[Theorem 1 of \citep{leshno1993multilayer}, \citep{bach2017breaking}]
\label{thm:neural-network-approximator}
    For the set of neural networks with a single neuron, i.e. $\cF=\{\sigma(w^Tx+b)|w\in\RR^D, b\in\RR\}$. The linear span of $\cF$ is dense in the Banach space of bounded continuous functions $C_b(X)$, i.e. $C_b(x)\subseteq\cl(\vspan\cF)$, if the activation function $\sigma(\cdot)$ is continuous but not a polynomial, or if
    $\sigma(u)=\max\{u,0\}^\ap$ for some $\ap\in\NN$ (when $\ap=1$, $\sigma(u)=\max\{u,0\}$ is the ReLU activation).
\end{theorem}

See \citep{leshno1993multilayer} and \citep{bach2017breaking} for further details and the proofs.

\subsection{IPMs Upper-Bounding the Symmetric KL Divergence}
\label{appendix:subsec:ipm-lower-bound}

\citep{zhang2018on} explains how IPMs can control the likelihood function, so that along with the training of an IPM-based GAN, the training likelihood should generally increase. More specifically, they prove that if the densities $P$ and $Q$ exist, and $\log(P/Q)$ is inside the closure of the linear span of $\cF$, i.e. $\log(P/Q)\in\cl(\vspan\cF)$, a function of the IPM $d_\cF\PQ{}$ can upper-bound the Symmetric KL divergence $\SKL\PQ{}$. In this sense, minimizing the IPM leads to the minimization of Symmetric KL divergence (and thus KL divergence), which is equivalent to the maximization of the training likelihood.
    
\begin{theorem}[Proposition 2.7 and 2.9 of \citep{zhang2018on}]
\label{thm:ipm-discriminative}
    Any function $g$ inside the closure of the linear span of $\cF$, i.e. $g\in\cl(\vspan\cF)$, is approximated by $\cF$ with an error decay function $\veps(r)$. It satisfies,
    \[
        \Big|\EE_{x\sim P}[g(x)]-\EE_{x\sim Q}[g(x)]\Big|\leq2\veps(r)+r d_{\cF}\PQ{} \qquad \forall r\geq0
    \]
    Moreover, consider two distributions with positive densities $P$ and $Q$, if $g=\log(P/Q)\in\cl(\vspan\cF)$, we have,
    \[
        \SKL\PQ{}\equiv\Big|\EE_{x\sim P}[\log(P(x)/Q(x))]-\EE_{x\sim Q}[\log(P(x)/Q(x))]\Big|\leq2\veps(r)+r d_{\cF}\PQ{} \qquad \forall r\geq0
    \]
\end{theorem}

\begin{proof}
    The proof is in Appendix C of \citep{zhang2018on}. We repeat the proof here for completeness.
    
    Since $g$ is approximated by $\cF$ with error decay function $\veps(r)$, for any $r\geq0$, there exist some $\df_r\in\vspan\cF$, which can be represented as $\df_r=\sum_{i=1}^n\ap_i\df_i+\ap_0$ with some $\ap_i\in\RR$ and $\df_i\in\cF$, such that $\sum_{i=1}^n|\ap_i|=\|\df_r\|_\cF\leq r$ and $\|g-\df_r\|_\infty<\veps(r)$. In this sense, we have,
    \[
        \begin{split}
            &\Big|\EE_{x\sim P}[g(x)]-\EE_{x\sim Q}[g(x)]\Big|\\
            &=\Big|\big(\EE_{x\sim P}[g(x)]-\EE_{x\sim P}[\df_r(x)]\big)-\big(\EE_{x\sim Q}[g(x)]-\EE_{x\sim Q}[\df_r(x)]\big)+\big(\EE_{x\sim P}[\df_r(x)]-\EE_{x\sim Q}[\df_r(x)]\big)\Big|\\
            &\leq\Big|\EE_{x\sim P}[g(x)-\df_r(x)]\Big|+\Big|\EE_{x\sim Q}[g(x)-\df_r(x)]\Big|+\Big|\EE_{x\sim P}[\df_r(x)]-\EE_{x\sim Q}[\df_r(x)]\Big|\\
            &\leq\EE_{x\sim P}\big|g(x)-\df_r(x)\big|+\EE_{x\sim Q}\big|g(x)-\df_r(x)\big|+\Big|\sum_{i=1}^n\ap_i\big(\EE_{x\sim P}[\df_i(x)]-\EE_{x\sim Q}[\df_i(x)]\big)\Big|\\
            &\leq2\veps(r)+\sum_{i=1}^n|\ap_i|\Big|\EE_{x\sim P}[\df_i(x)]-\EE_{x\sim Q}[\df_i(x)]\Big|\\
            &\leq2\veps(r)+r d_\cF\PQ{}\\
        \end{split}
    \]
    Applying this inequality to $g=\log(P/Q)$ proves that, for any $r\geq0$, this linear function of IPM $2\veps(r)+r d_\cF\PQ{}$ upper-bounds the Symmetric KL divergence $\SKL\PQ{}$.
\end{proof}

The upper-bounds obtained by \cref{thm:ipm-discriminative} are a set linear functions of the IPM, $\{2\veps(r)+r d_\cF\PQ{}\big|r\geq0\}$. In order to prove that the IPM $d_\cF\PQ{}$ can upper-bound the Symmetric KL divergence $\SKL\PQ{}$ up to some constant coefficient and additive error, i.e. $\ap\SKL\PQ{}-\eps\leq d_\cF\PQ{}$ for some constants $\ap, \veps>0$, we have to control both $\veps(r)$ and $r$ simultaneously. Because $\lim_{r\to\infty}\veps(r)=0$, all we need is an efficient upper-bound on $\veps(r)$ for large enough $r$, which is provided in \citep{bach2017breaking}.

\begin{theorem}[Proposition 6 of \citep{bach2017breaking}]
\label{thm:approximability-lipschitz}
    For a bounded space $\Sp$, let $g:\Sp\to\RR$ be a bounded and Lipschitz continuous function (i.e. there exists a constant $\eta>0$ such that $\|g\|_\infty<\eta$ and for any $x, y\in\Sp\subseteq\RR^{D}$, it holds that $\|g(x)-g(y)\|_\infty\leq\frac{1}{\diam\Sp}\eta\|x-y\|_2$), and let $\cF$ be a set of neural networks with a single neuron, which have ReLU activation and bounded parameters (i.e. $\cF=\{\max\{w^Tx+b, 0\}\big|w\in\RR^{D}, b\in\RR, \|[w,b]\|_2=1\}$). Then, we have $g\in\cl(\vspan\cF)$, and $g$ is approximated by $\cF$ with error decay function $\veps(r)$, such that,
    \[
        \veps(r)\leq C(D)\eta\left(\frac{r}{\eta}\right)^{-\frac{2}{D+1}}\log\left(\frac{r}{\eta}\right) \qquad \forall r\geq R(D)
    \]
    where $C(D), R(D)$ are constants which only depend on the number of dimensions, $D$.
\end{theorem}

See Proposition 3, Appendix C.3, and Appendix D.4 of \citep{bach2017breaking} for the proof.

%%%%%%%%%%%%%%%%%%%%%%%%%%%%%%%%%%%%%%%%%%%%%%%%%%%%%%%%%%%%%%%%%%%%%%%%%%%%%%%%%%%%%%%%%%%%%%%%%%%%
\section{Subadditivity Upper-Bounds at Different ``Levels of Detail'' on Sequences}
\label{appendix:levels-of-detail}

The subadditivity upper-bound on a Bayes-net, $\sum_{i=1}^n \D\PQ{X_i\cup X_{\Pi_i}}$, depends on the structure of the Bayes-net. More specifically, the underlying DAG $G$ determines the set of local neighborhoods $\{\{1\}\cup\Pi_1,\cdots,\{n\}\cup\Pi_n\}$, and consequently, determines how we construct the set of local discriminators. In this section, we discuss that the set of local neighborhoods can be change either by truncating the induction process when deriving the subadditivity upper-bound (see the proof of \cref{thm:subadditivity-markov-chain} in \cref{proof:subadditivity-markov-chain} for example), or by contracting the neighboring nodes of the Bayes-net. Both methods result in a tighter subadditivity upper-bound at a coarser level-of-detail (i.e., with larger local neighborhoods). For the sake of simplicity, we limit the scope to sequences describing auto-regressive time series. For such graph $G$, there are $T$ nodes ($\{1,\cdots,T\}$), and each node depends on its $p$ previous nodes; see \cref{fig:lod-truncation} for an example.

\begin{figure}[ht]
    \centering
    \begin{subfigure}{.46\linewidth}
        \centering
        \includegraphics[width=\linewidth]{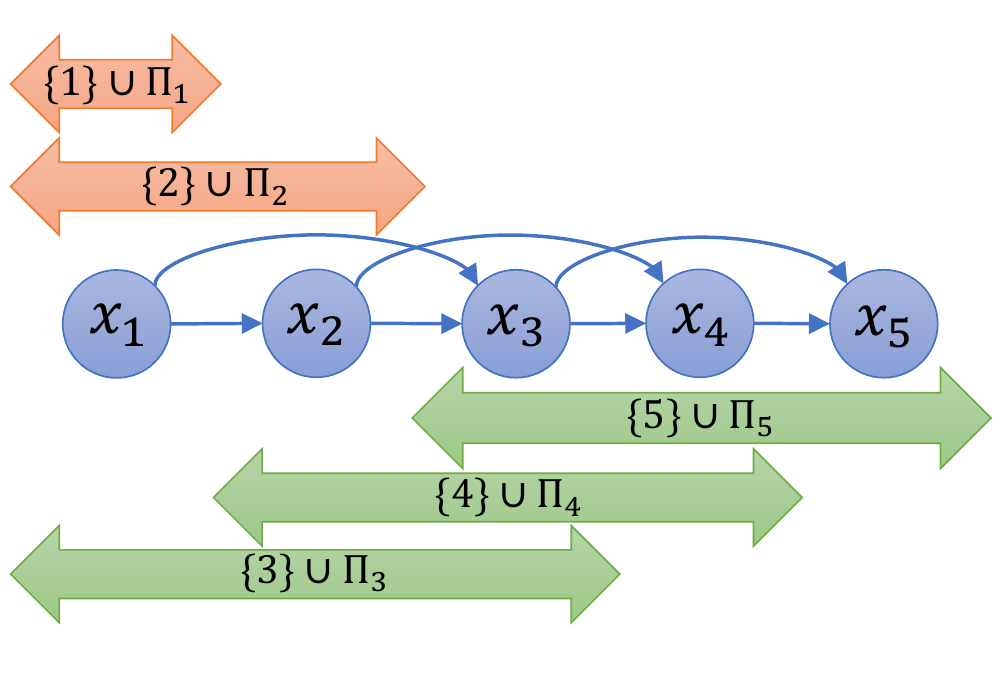}
        \caption{Local neighborhoods of auto-regressive time series with $T=5$ and $p=2$. The original set of local neighborhoods is represented by the red and green bars. Two local neighborhoods $\{1\}\cup\Pi_1$ and $\{2\}\cup\Pi_2$ (red bars) can be safely removed by truncating the induction process.}
        \label{fig:lod-truncation}
    \end{subfigure}
    \hspace{.05\linewidth}
    \begin{subfigure}{.47\linewidth}
        \centering
        \includegraphics[width=\linewidth]{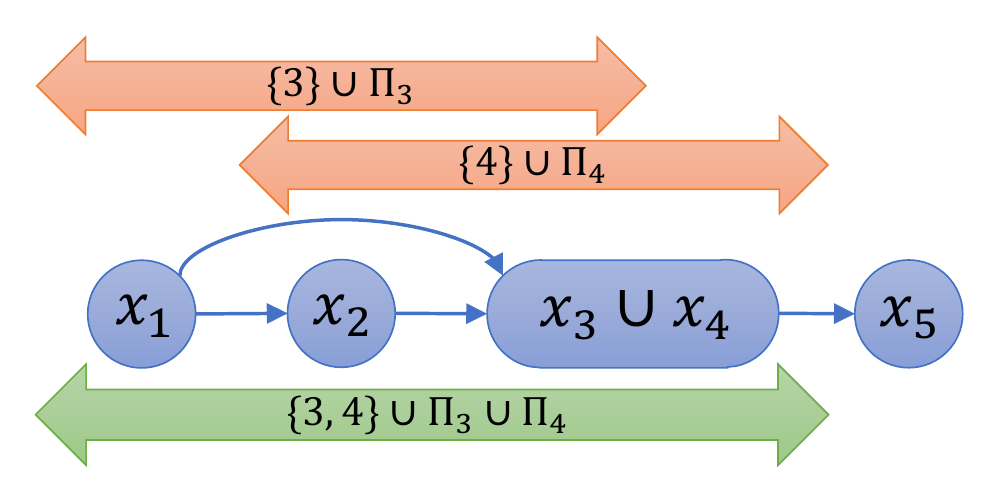}
        \vspace{18.5pt}
        \caption{Change of the local neighborhoods of auto-regressive time series with $T=5$ and $p=2$, if contracting neighboring nodes $3$ and $4$ to form a super-node $\{3, 4\}$. Two local neighborhoods $\{3\}\cup\Pi_3$ and $\{4\}\cup\Pi_4$ (red bars) are replaced by a new neighborhood $\{3, 4\}\cup\Pi_3\cup\Pi_4$ (green bar).}
        \label{fig:lod-contraction}
    \end{subfigure}
    \caption{Changes of the local neighborhoods of a Bayes-net representing auto-regressive time series with $T=2$ and $p=2$, if we (a) truncate the induction process, or (b) contract a pair of neighboring nodes. In each case, the subadditivity upper-bound becomes tighter and characterize the Bayes-net at a coarser level-of-detail.}
    \label{fig:lod}
\end{figure}

\subsection{Truncation of Induction}
\label{appendix:subsec:lod-truncation}

For a probability divergence $\D$ which satisfies subadditivity on Bayes-nets, the subadditivity upper-bound $\sum_{t=1}^T \D\PQ{X_t\cup X_{\Pi_t}}$ of $\D\PQ{}$ is obtained by repeatedly applying the subadditivity of $\D$ on Markov Chains of super-nodes $X_{\{1,\cdots,s\}\setminus\Pi_{s+1}}\to X_{\Pi_{s+1}}\to X_{s+1}$, for $s=T-1,T-2,\cdots,1$. We can truncate the induction process and get an alternative upper-bound: $\D\PQ{}<\D\PQ{\cup_{t=1}^{s}X_t}+\sum_{t=s+1}^T \D\PQ{X_t\cup X_{\Pi_t}}$. This new upper-bound is tighter, but it does not encode the conditional independence information of the sub-sequence $(X_1,\cdots,X_s)$. However, this alternative upper-bound is preferable if we choose $s$ to be the largest number where its set of parents is exactly its previous nodes, i.e. $\Pi_s=\{1,\cdots,s-1\}$. The subadditivity inequality that we combined at induction step $s$ is $\D\PQ{\cup_{t=1}^{s}X_t}\equiv\D\PQ{X_{\Pi_s}\cup X_s}\leq\D\PQ{\cup_{t=1}^{s-1}X_t}+\D\PQ{X_{\Pi_s}\cup X_s}$ (corresponding to the second case in the proof of \cref{thm:subadditivity-markov-chain} in \cref{proof:subadditivity-markov-chain}). Truncating at such $s$ avoids introducing the redundant term $\D\PQ{\cup_{t=1}^{s-1}X_t}$ into the upper-bound. As shown in \cref{fig:lod-truncation}, for this specific example $s=p+1=3$ is the largest number such that $\Pi_3=\{1, 2\}$. Truncating at $s=3$ removes $\{1\}\cup\Pi_1$ and $\{2\}\cup\Pi_2$ from the set of local neighborhoods, resulting in a more efficient subadditivity upper-bound $\sum_{t=3}^5\D\PQ{X_t\cup X_{\Pi_t}}$. This is helpful for time series data, since it makes all local neighborhoods have the same number of dimensions. If all $X_t\in\RR^d$, then for $t=3, 4$ and $5$, $X_t\cup X_{\Pi_t}\in\RR^{3d}$. In this sense, we can share the same neural network architecture among all the local discriminators.

\subsection{Neighboring Nodes Contraction}
\label{appendix:subsec:lod-contraction}

The set of local neighborhoods is determined by the structure $G$ of the Bayes-net. Network contraction not only simplifies the Bayes-net but also leads to a tighter subadditivity upper-bound at a lower level-of-detail. Here, we only consider the contraction of neighboring nodes in a time series $(X_1,\cdots,X_T)$. If we merge node $s$ with $s+1$ ($s=1,\cdots,T-1$), and form a super-node $\{s. s+1\}$, local neighborhoods $\{s\}\cup\Pi_s$ and $\{s+1\}\cup\Pi_{s+1}$ are replaced by $\{s, s+1\}\cup\Pi_s\cup\Pi_{s+1}$, and the total number of neighborhoods decreases by one. As shown in \cref{fig:lod-contraction}, when nodes $3$ and $4$ are merged, local neighborhoods $\{3\}\cup\Pi_3$ and $\{4\}\cup\Pi_4$ are replaced by $\{3, 4\}\cup\Pi_3\cup\Pi_4$. We omit the conditional dependence between nodes $3$ and $4$, but reduce one local discriminator in the GAN. Neighboring nodes contraction allows us to control the level-of-detail that the subadditivity upper-bound encodes flexibly. This can be useful when the variables in the Bayes-net have non-uniform dimensionalities.

%%%%%%%%%%%%%%%%%%%%%%%%%%%%%%%%%%%%%%%%%%%%%%%%%%%%%%%%%%%%%%%%%%%%%%%%%%%%%%%%%%%%%%%%%%%%%%%%%%%%
\section{More Experiment Results}
\label{appendix:more-experiment-results}

\subsection{Experiments on Synthetic MRFs}

We perform some additional experiments to demonstrate the benefits of exploiting the underlying MRF structure of the data in the design of model-based GANs. We generate a synthetic Gaussian MRF dataset, the graph of which is a $4$-cycle (namely a graph with $4$ nodes, each of which has degree $2$). We train our model-based GAN (we call them MRF GANs) and the standard model-free GAN on five thousand samples. The network architecture (except for the number of inputs) and the hyper-parameters are the same as what we used on the EARTHQUAKE dataset (see \cref{appendix:experimental-setups}). We evaluate the performance of the GANs by the Energy Statistics score (for the definition, see \cref{sec:experiments}). We observe that the average Energy Statistics of MRF GAN is $(1.3\pm0.2)\times 10^{-3}$, which is smaller (thus better) than the standard GAN's, $(7.9\pm1.2)\times 10^{-3}$. This simple experiment confirms that the benefits of exploiting the conditional independence structure apply to MRFs as well. This is consistent with our theory. Such benefits can be reaped even in low dimensions.

\subsection{Sensitivity Analysis of Bayes-net GANs}
In this part, we analyze how sensitive Bayes-net GAN is to its causal structure. In \cref{subsec:experiment-bayesnet}, we know from the experiments on two real Bayes-nets that the Bayes-net GANs with the ground truth causal DAGs consistently outperform the standard GANs in terms of the generation quality and the convergence speed. Here, we perturb the ground truth DAG of {\it the EARTHQUAKE dataset} by randomly rewiring $1$ of its $4$ edges. The graph editing distance between the resulting noisy causal DAG and the ground truth is therefore $1$. We train a Bayes-net GAN with local discriminators constructed using this noisy DAG, and evaluate it by the four metrics as in \cref{subsec:experiment-bayesnet} (\cref{tab:bayesnets-scores-noisy}), as well as using the causal structure predicted from its generated data (\cref{fig:bayesnet-structure-noisy}). In \cref{tab:bayesnets-scores-noisy}, we observe that the energy statistics and the detection AUC scores are not sensitive to the noise in the casual structure used by the Bayes-net GAN. The BIC score and the predicted structure are more affected. The Bayes-net GAN with the noisy causal graph learns some redundant indirect dependence (\cref{fig:bayesnet-structure-local-noisy}).

\begin{table}[H]
    \centering
    \resizebox{\textwidth}{!}{%
        \begin{tabular}{c|cccc}
        \toprule
        DAG used      & \begin{tabular}[c]{@{}c@{}}{\bf Energy Stats.} ($\times 10^{-2}$)\\{\small (smaller is better)}\end{tabular} & \begin{tabular}[c]{@{}c@{}}{\bf Detection AUC}\\{\small (smaller is better)}\end{tabular} & \begin{tabular}[c]{@{}c@{}}{\bf Rel. BIC} ($\times 10^2$)\\{\small (larger is better)}\end{tabular} & \begin{tabular}[c]{@{}c@{}}{\bf Rel. GED}\\{\small (smaller is better)}\end{tabular} \\ \midrule
        Ground truth & $0.24\pm0.04$                                                                                                           & $0.523\pm0.005$                                                             & $+1.68\pm0.17$                                                                        & $0.4\pm0.7$                                                            \\ \hline
        Noisy        & $0.27\pm0.11$                                                                                                           & $0.528\pm0.002$                                                             & $-2.02\pm0.15$                                                                        & $2.4\pm1.1$                                                            \\ \bottomrule                           
        \end{tabular}%
    }
    \captionof{table}{Quality metrics of samples generated by the Bayes-net GANs with the ground truth and the noisy causal structure.}
    \label{tab:bayesnets-scores-noisy}
\end{table}

\begin{figure}[H]
    \centering
    \begin{subfigure}[t]{.35\linewidth}
        \centering
        \makebox[1.4\width]{\centering\includegraphics[width=.7\linewidth]{figs/bayesnet-structure-local-true.pdf}}
        \caption{Bayes-net GAN with ground truth DAG}
        \label{fig:bayesnet-structure-local-true-comp}
    \end{subfigure}%
    \hspace{4em}
    \begin{subfigure}[t]{.35\linewidth}
        \centering
        \makebox[1.4\width]{\centering\includegraphics[width=.7\linewidth]{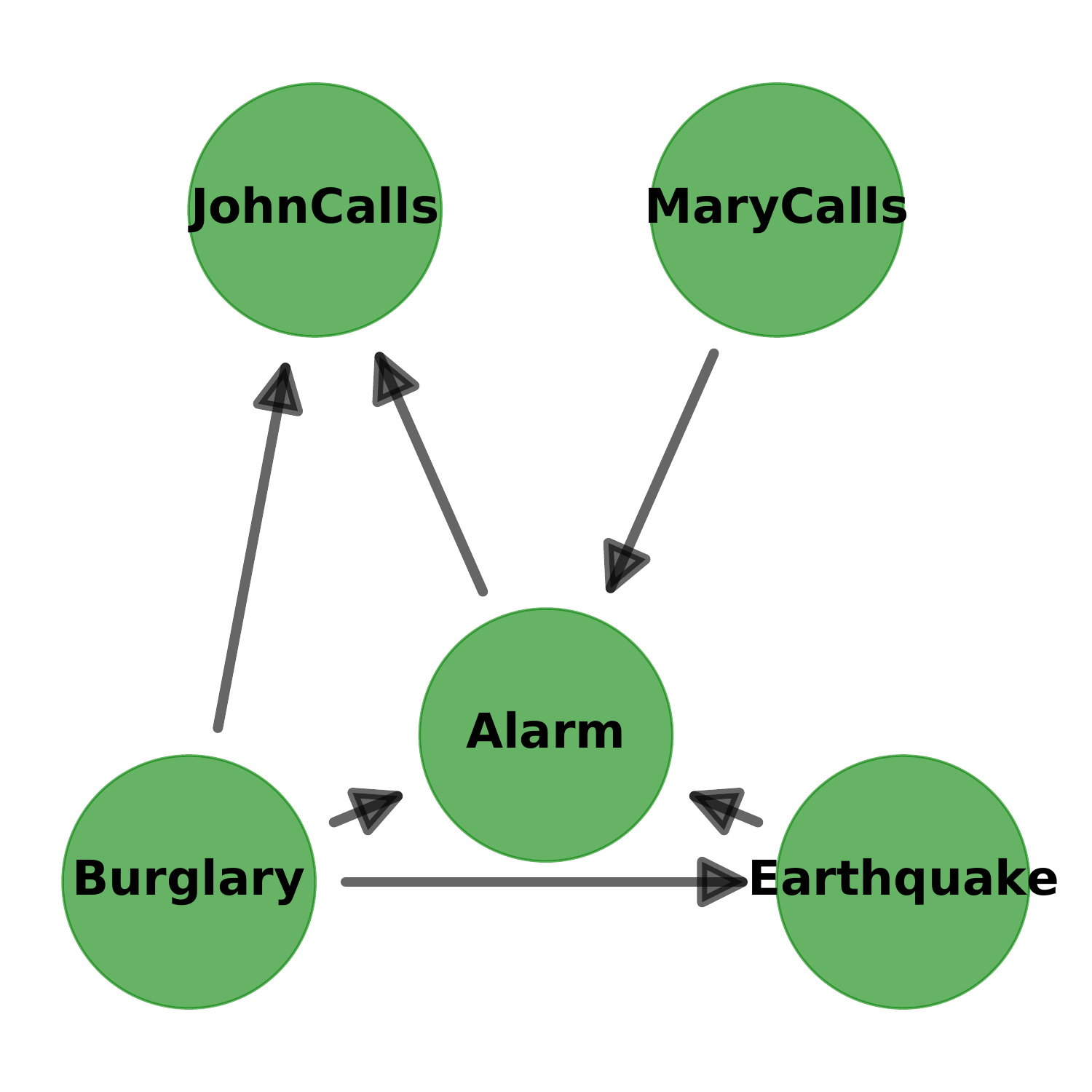}}
        \caption{Bayes-net GAN with noisy DAG}
        \label{fig:bayesnet-structure-local-noisy}
    \end{subfigure}%
    \caption{Causal structures predicted from the data generated by the Bayes-net GAN with \subref{fig:bayesnet-structure-local-true-comp} the ground truth causal graph and \subref{fig:bayesnet-structure-local-noisy} the noisy causal graph.}
    \label{fig:bayesnet-structure-noisy}
\end{figure}

%%%%%%%%%%%%%%%%%%%%%%%%%%%%%%%%%%%%%%%%%%%%%%%%%%%%%%%%%%%%%%%%%%%%%%%%%%%%%%%%%%%%%%%%%%%%%%%%%%%%
\section{Experimental Setups}
\label{appendix:experimental-setups}

In this section, we report the detailed setups of the experiments in \cref{sec:experiments}. Unless otherwise stated, model comparisons between our model-based GANs and the standard GANs are conducted using exactly the same set of training hyper-parameters. All experiments are repeated five or ten times.

\subsection{Datasets}
\label{appendix:subsec:datasets}

\begin{itemize}[leftmargin=*]
    \item \textbf{Experiment 1: synthetic ball throwing trajectories}: {\it The ball throwing trajectory dataset} is synthetic, consists of single-variate time-series data $(y_1, \ldots, y_{15})$ representing the $y$-coordinates of ball throwing trajectories lasting $1$ second, where $y_t = v_0*(t/15)-g(t/15)^2/2$. $v_0$ is a Gaussian random variable and $g=9.8$ is the gravitational acceleration.
    \item \textbf{Experiment 2: real Bayes-nets}: The {\it EARTHQUAKE dataset} is a small Bayes-net with $5$ nodes and $4$ edges characterizing the alarm system against burglary which can get occasionally set off by an earthquake~\citep{korb2010bayesian}. {\it The CHILD dataset} is a Bayes-net with $20$ nodes and $25$ edges for diagnosing congenital heart disease in a newborn ``blue baby''~\citep{spiegelhalter1992learning}. The underlying causal DAGs of two Bayes-nets are known, and are obtained from \url{https://www.bnlearn.com/bnrepository/}. Both Bayes-nets have categorical features. We simulate samples from the Bayes-nets and then train GANs on them.
\end{itemize}

\subsection{Local Discriminators}
\label{appendix:subsec:local-discriminators}

\begin{itemize}[leftmargin=*]
    \item \textbf{Experiment 1: synthetic ball throwing trajectories}: Each local discriminator measures the Jensen-Shannon divergence in the local neighborhood, following~\citep{nowozin2016f}. The use of local discriminators is justified by the $(1/\ln2)$-linear subadditivity of Jensen-Shannon divergence on Bayes-nets (\cref{coro:linear-subadditivity-js}).
    \item \textbf{Experiment 2: real Bayes-nets}: Each local discriminator measures the Wasserstein distance in the local neighborhood, following~\citep{arjovsky2017wasserstein}. The use of local discriminators is justified by the generalized subadditivity of neural distances on Bayes-nets (\cref{thm:subadditivity-ipm}). We use Gumbel-Softmax~\citep{jang2016categorical} in the output layer of the generator, so that the generator produces categorical data while allowing (approximate) back propagation.
\end{itemize}

\subsection{Network Architectures}
\label{appendix:subsec:network-architectures}

\begin{itemize}[leftmargin=*]
    \item \textbf{Experiment 1: synthetic ball throwing trajectories}: For the GANs on {\it the ball throwing trajectory} dataset, we use a $5$-layer fully connected network (FCN) for the generator where the number of hidden dimensions is set to $8$ for all layers. We take a hybrid design for each local discriminator. Each local discriminator is a combination of a $4$-layer FCN and a $3$-layer convolutional neural network (CNN) so that it can penalize both inaccurate global distributions (via FCN) and inaccurate local dynamics (via CNN). The discriminator of the standard GAN on {\it the ball throwing trajectory} dataset is the local discriminator with localization width $15$ (which is equal to the length of the time-series). The local discriminators share the same architecture (except the input layer) even if the localization width varies, but they do not share parameters. We make the discriminators' architecture powerful enough such that adding more neurons/layers cannot bring us any further performance gain.
    \item \textbf{Experiment 2: real Bayes-nets}: For the GANs on {\it the EARTHQUAKE dataset}, we use a $5$-layer FCN for the generator where the number of hidden dimensions is set to $32$ for all layers. We also use a $4$-layer FCN for each local discriminator where the number of hidden dimensions is set to $8$. For the GANs on {\it the CHILD dataset}, we use a $7$-layer FCN with $256$ hidden dimensions for the generator, and a $6$-layer FCN with $32$ hidden dimensions for the discriminator. There is no parameter sharing among the local discriminators. We apply Batch Normalization after each hidden layer in the generator, but not in the discriminator. All ReLUs are leaky, with slope $0.2$. In the experiments, we keep the architecture of the generator and the other hyper-parameters the same, and compare our model-based GANs with the standard GANs on Bayes-nets.
\end{itemize}

\subsection{Training Setups and Hyper-Parameters}
\label{appendix:subsec:training-setups}

The networks are implemented using the {\it PyTorch} framework. All networks are trained from scratch on one {\it NVIDIA RTX 2080 Ti} GPU with 11GB memory.

\begin{itemize}[leftmargin=*]
    \item \textbf{Experiment 1: synthetic ball throwing trajectories}: We train GANs with local discriminators (with localization width equals to $1, 2, 3, 5, 8, 11, 15$) for $500$ epochs, with learning rate $0.0001$ and batch size $128$. We repeat each experiment $10$ times and report the averages with uncertainties.
    \item \textbf{Experiment 2: real Bayes-nets}: We train the standard GANs and the Bayes-net GANs for $100$ epochs, with learning rate $0.001$ and batch size $128$. We repeat each experiment $5$ times and report the averages with uncertainties.
\end{itemize}

\subsection{Evaluation Setups}
\label{appendix:subsec:evaluation-setups}

\begin{itemize}[leftmargin=*]
    \item \textbf{Experiment 1: ball throwing trajectories}: We estimate the gravitational acceleration $g$ learned by the GANs, via degree-$2$ polynomial regression on the generated trajectories.
    \item \textbf{Experiment 2: real Bayes-nets}: The energy statistics are calculated using the standard {\it torch-two-sample} package (available at \url{https://github.com/josipd/torch-two-sample}). The fake detection AUC scores are obtained by training binary classifiers to distinguish the fake samples from the real ones. The binary classifier is a $3$-layer FCN with hidden dimensions $16$ on {\it the EARTHQUAKE dataset} and a $5$-layer FCN with hidden dimension $32$ on {\it the CHILD dataset}. We train the classifiers for $100$ epochs, with learning rate $0.001$ and batch size $128$.
\end{itemize}

%%%%%%%%%%%%%%%%%%%%%%%%%%%%%%%%%%%%%%%%%%%%%%%%%%%%%%%%%%%%%%%%%%%%%%%%%%%%%%%%%%%%%%%%%%%%%%%%%%%%
\section{Empirical Verification of Subadditivity}
\label{appendix:verification-subadditivity}

In this section, we verify the subadditivity of squared Hellinger distance, KL divergence, Symmetric KL divergence, and the linear subadditivity of Jensen-Shannon divergence, Total Variation distance, $1$-Wasserstein distance, and $2$-Wasserstein distance on binary auto-regressive sequences in a finite space $\Sp$.

To construct a simple Bayes-net $P$ on a sequence of bits $(X_1, \cdots, X_n)\in\{0,1\}^n$, consider the auto-regressive sequence defined by,
\[
    P(X_t=1|X_{t-1},\cdots,X_{t-p})=\sigma(\sum_{i=1}^{p}\varphi_i X_{t-i})
\]
where $p\in\NN$ such that $0<p<n$ is called the order of this auto-regressive sequence, and $[\varphi_1,\cdots,\varphi_n]$ are the coefficients. The marginal distributions of the initial variables $X_1,\cdots,X_p$ have to be pre-defined. We assume they are conditionally independent, and define,
\[
    P(X_i=1)=\psi_i \qquad \forall i\in\{1,\cdots,p\}
\]
where for any $i\in\{1,\cdots,p\}$, $\psi_i\in[0,1]$. If the distribution of a binary sequence $(X_1,\cdots,X_n)$ follows the definitions above, we say it is a binary auto-regressive sequence of order $p$ with coefficients $[\varphi_1,\cdots,\varphi_n]$ and initials $[\psi_1,\cdots,\psi_n]$.

Binary auto-regressive sequences are Bayes-nets, because each variable $X_t$ is conditionally independent of its non-descendants given its parent variables $X_{t-1},\cdots,X_{t-p}$. The probabilistic graph $G$ is determined by the length $n$ and the order $p$. For a statistical divergence $\D$ satisfying subadditivity, as described in \cref{appendix:subsec:lod-truncation}, we truncate the induction process and get a subadditivity upper-bound $\sum_{t=p+1}^n\D\PQ{\cup_{i=t-p}^t X_i}$. We verify that the subadditivity inequality (or linear subadditivity inequality) holds for various statistical divergences, on two specific examples.

\begin{example}[Binary Auto-Regressive Sequences with Different Local Dependencies]
\label{exp:verify-subadditivity-discrete-dependency}
    Consider binary auto-regressive sequences $(X_1, X_2, X_3, X_4)\in\{0,1\}^4$ of order $p=2$ with initials $[\psi_1,\psi_2]=[\frac12,\frac12]$. Two distributions $P^x$ (with coefficients $[\varphi_1, \varphi_2]=[0,x]$) and $Q^y$ (with coefficients $[\varphi_1, \varphi_2]=[0,y]$) are Bayes-nets with identical underlying structure. Divergence $\D(P^x,Q^y)$ is a function of the parameters $(x, y)$. For all $(x,y)\in\{(x,y)\in\RR^2|x\neq y\}$, we have $\D(P^x, Q^y)<\sum_{t=p+1}^n \D(P^x_{\cup_{i=t-p}^t X_i}, Q^y_{\cup_{i=t-p}^t X_i})$ if $\D$ satisfies subadditivity, or $\ap\cdot\D(P^x, Q^y)<\sum_{t=p+1}^n \D(P^x_{\cup_{i=t-p}^t X_i}, Q^y_{\cup_{i=t-p}^t X_i})$ if $\D$ satisfies $\ap$-linear subadditivity.
\end{example}

\begin{example}[Binary Auto-Regressive Sequences with Different Initial Distributions]
\label{exp:verify-subadditivity-discrete-initial}
    Consider binary auto-regressive sequences $(X_1, X_2, X_3, X_4)\in\{0,1\}^4$ of order $p=2$ with coefficients $[\varphi_1, \varphi_2]=[1,-1]$. Two distributions $P^x$ (with initials $[\psi_1,\psi_2]=[\frac12, x]$) and $Q^y$ (with initials $[\psi_1,\psi_2]=[\frac12, y]$) are Bayes-nets with identical underlying structure. Divergence $\D(P^x,Q^y)$ is a function of the parameters $(x, y)$. For all $(x,y)\in\{(x,y)\in\RR^2|0<x\neq y<1\}$, we have $\D(P^x, Q^y)<\sum_{t=p+1}^n \D(P^x_{\cup_{i=t-p}^t X_i}, Q^y_{\cup_{i=t-p}^t X_i})$ if $\D$ satisfies subadditivity, or $\ap\cdot\D(P^x, Q^y)<\sum_{t=p+1}^n \D(P^x_{\cup_{i=t-p}^t X_i}, Q^y_{\cup_{i=t-p}^t X_i})$ if $\D$ satisfies $\ap$-linear subadditivity.
\end{example}

\begin{figure}[bp]
    \centering
    \hspace{-43pt}
    \large
    \begin{tabular}{lcc}
        & \cref{exp:verify-subadditivity-discrete-dependency} & \cref{exp:verify-subadditivity-discrete-initial} \\
        $\SH$ & \begin{tabular}[c]{@{}l@{}}\includegraphics[trim=0.4in 0.4in 0.4in 0.4in, clip, width=0.30\linewidth]{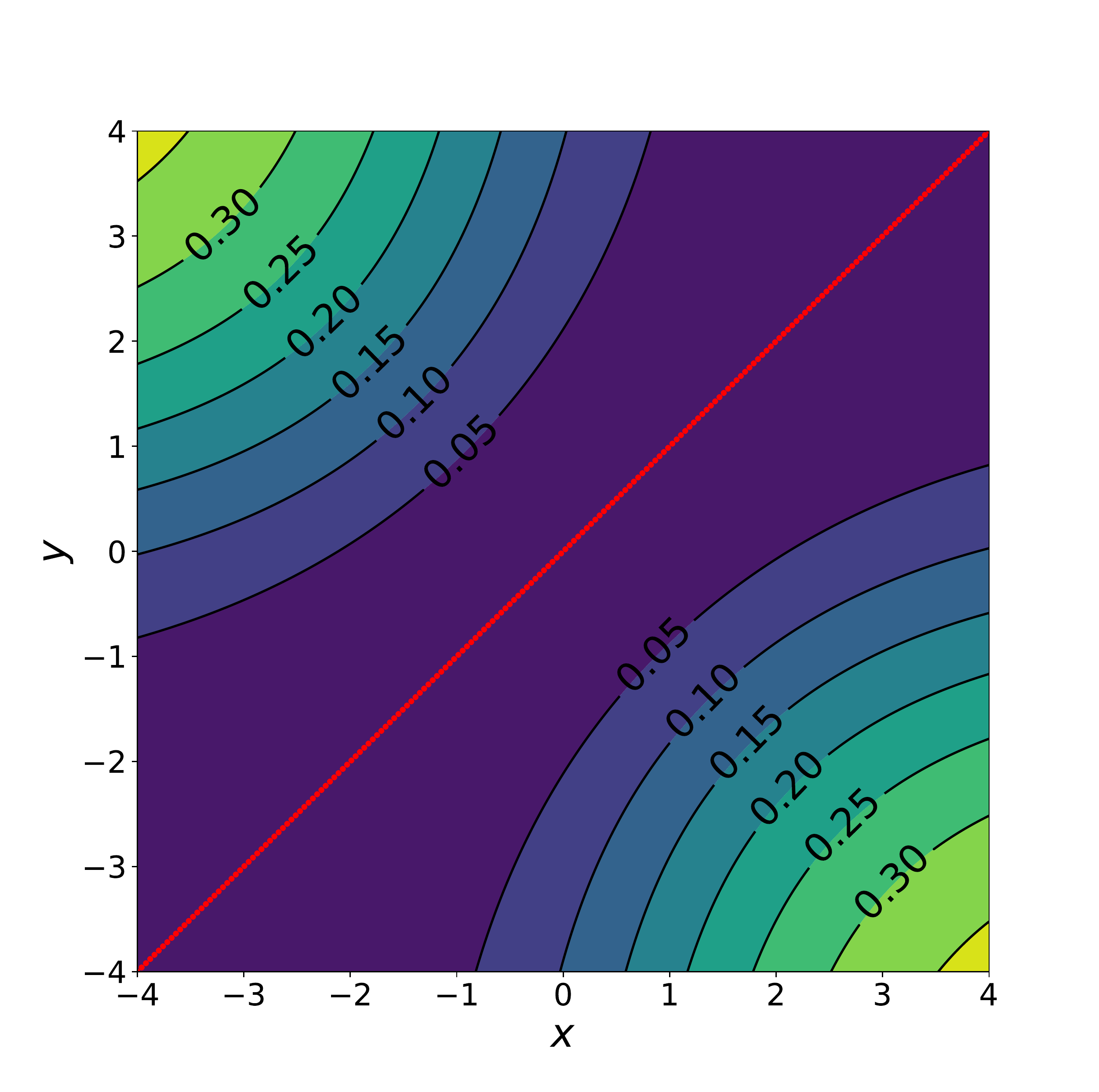}\end{tabular} & \begin{tabular}[c]{@{}l@{}}\includegraphics[trim=0.4in 0.4in 0.4in 0.4in, clip, width=0.30\linewidth]{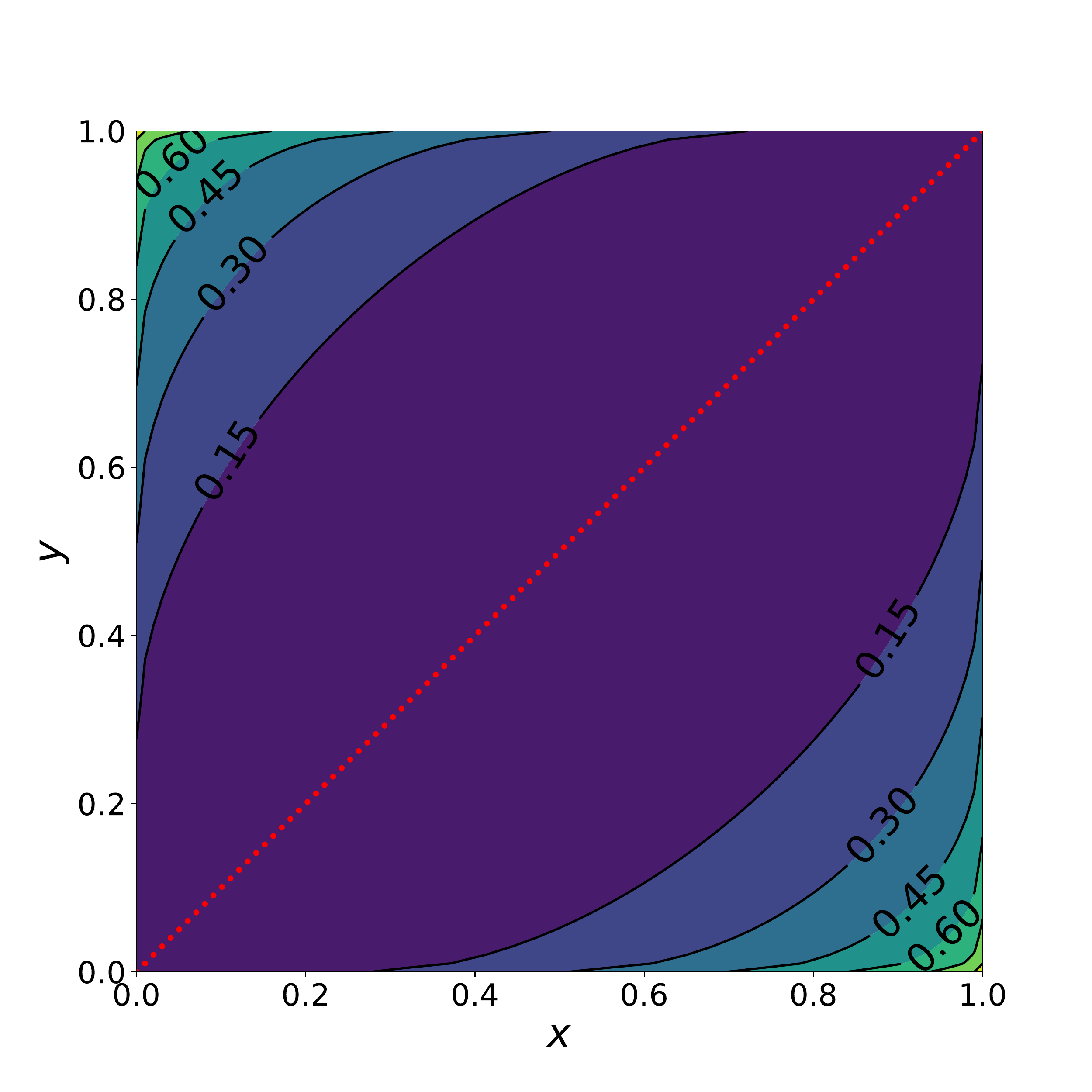}\end{tabular} \\
        $\KL$ & \begin{tabular}[c]{@{}l@{}}\includegraphics[trim=0.4in 0.4in 0.4in 0.4in, clip, width=0.30\linewidth]{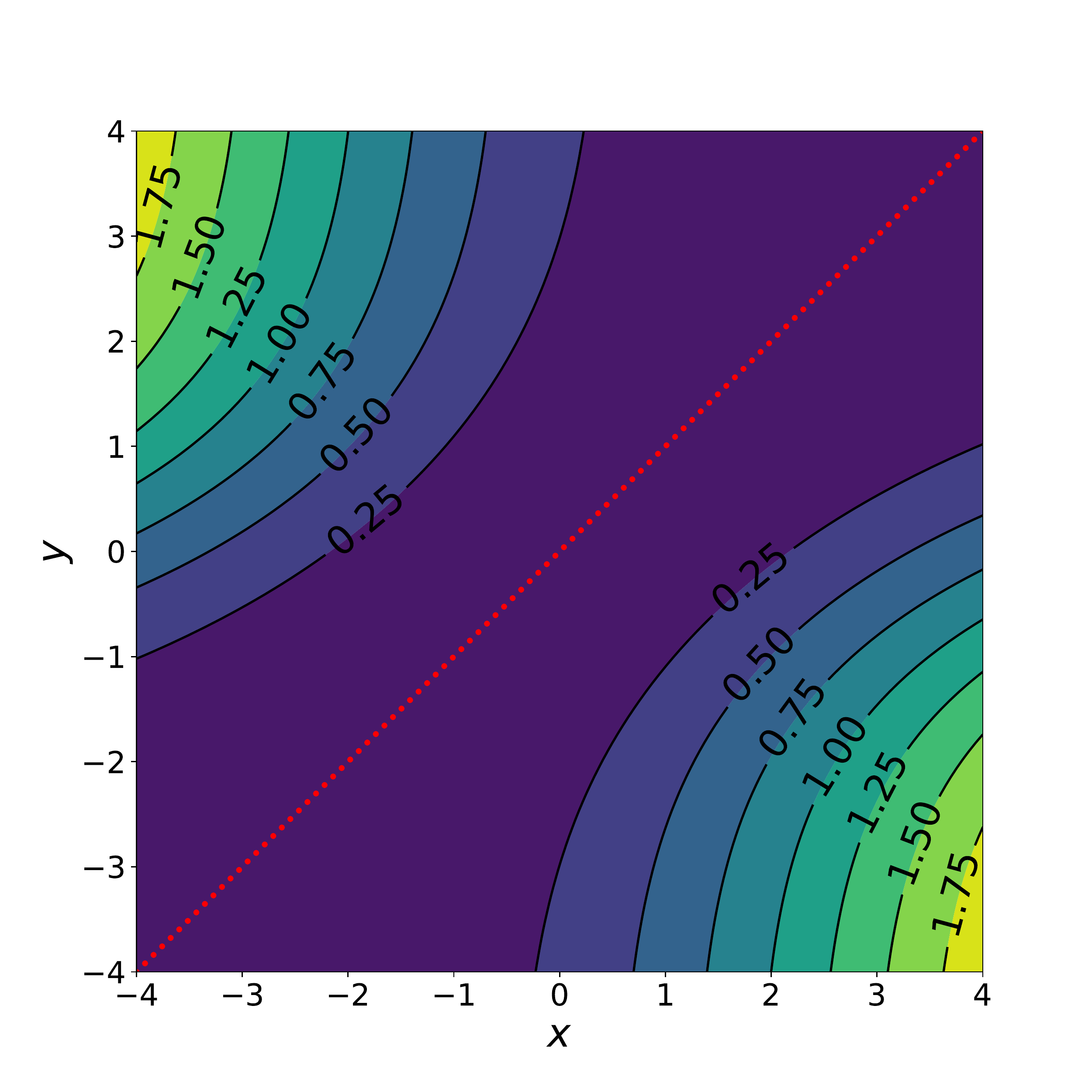}\end{tabular} & \begin{tabular}[c]{@{}l@{}}\includegraphics[trim=0.4in 0.4in 0.4in 0.4in, clip, width=0.30\linewidth]{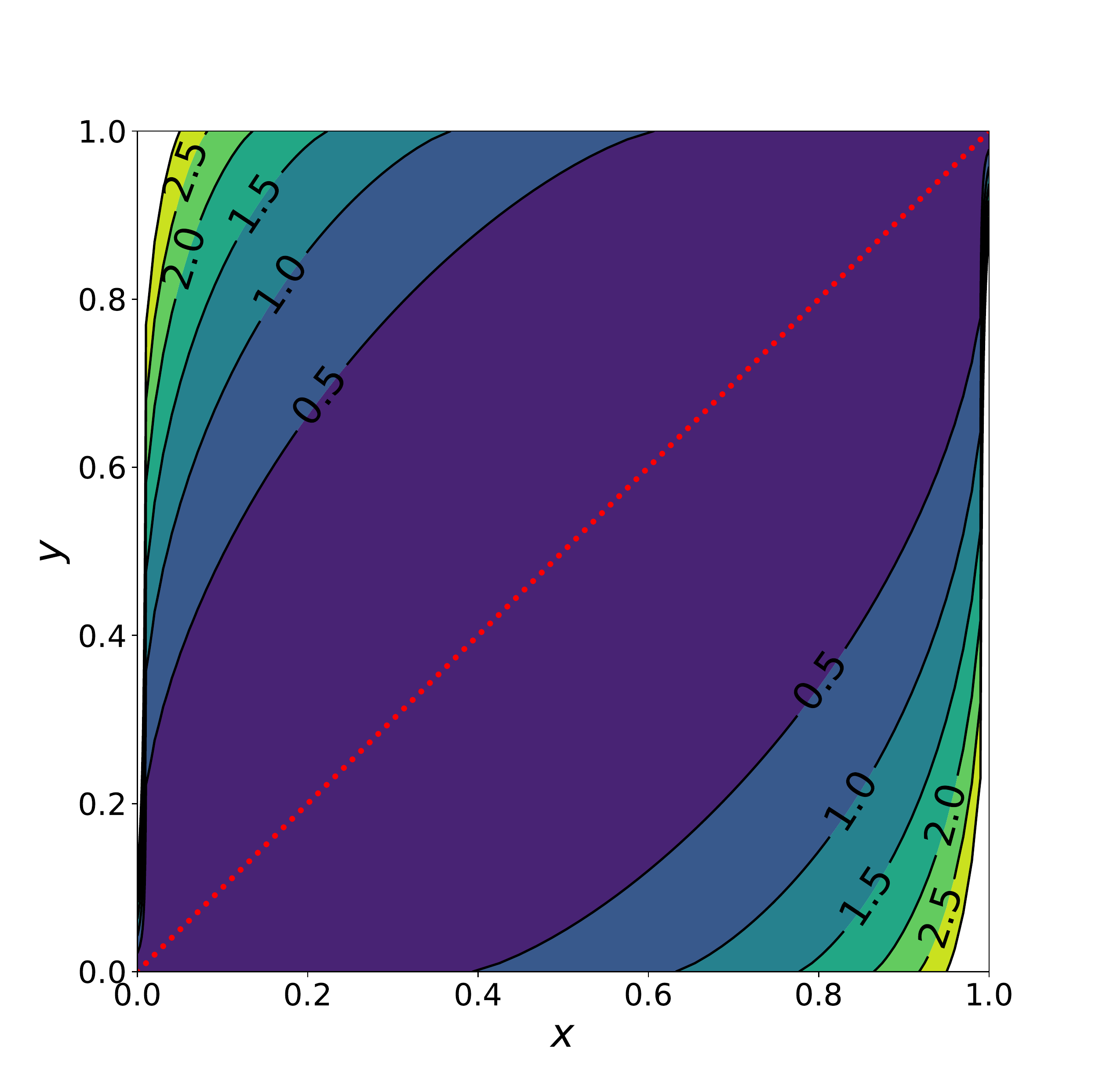}\end{tabular} \\
        $\SKL$ & \begin{tabular}[c]{@{}l@{}}\includegraphics[trim=0.4in 0.4in 0.4in 0.4in, clip, width=0.30\linewidth]{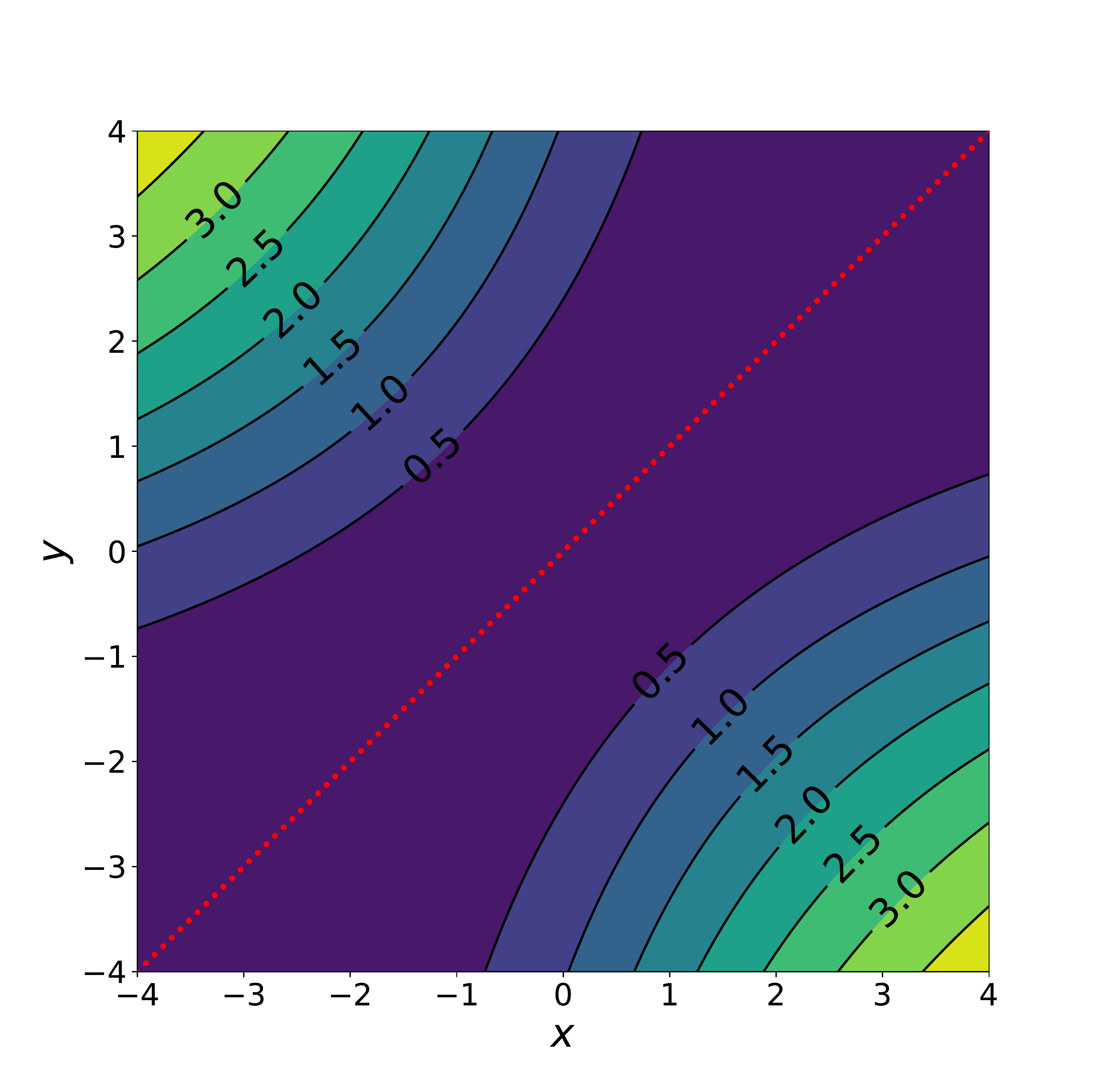}\end{tabular} & \begin{tabular}[c]{@{}l@{}}\includegraphics[trim=0.4in 0.4in 0.4in 0.4in, clip, width=0.30\linewidth]{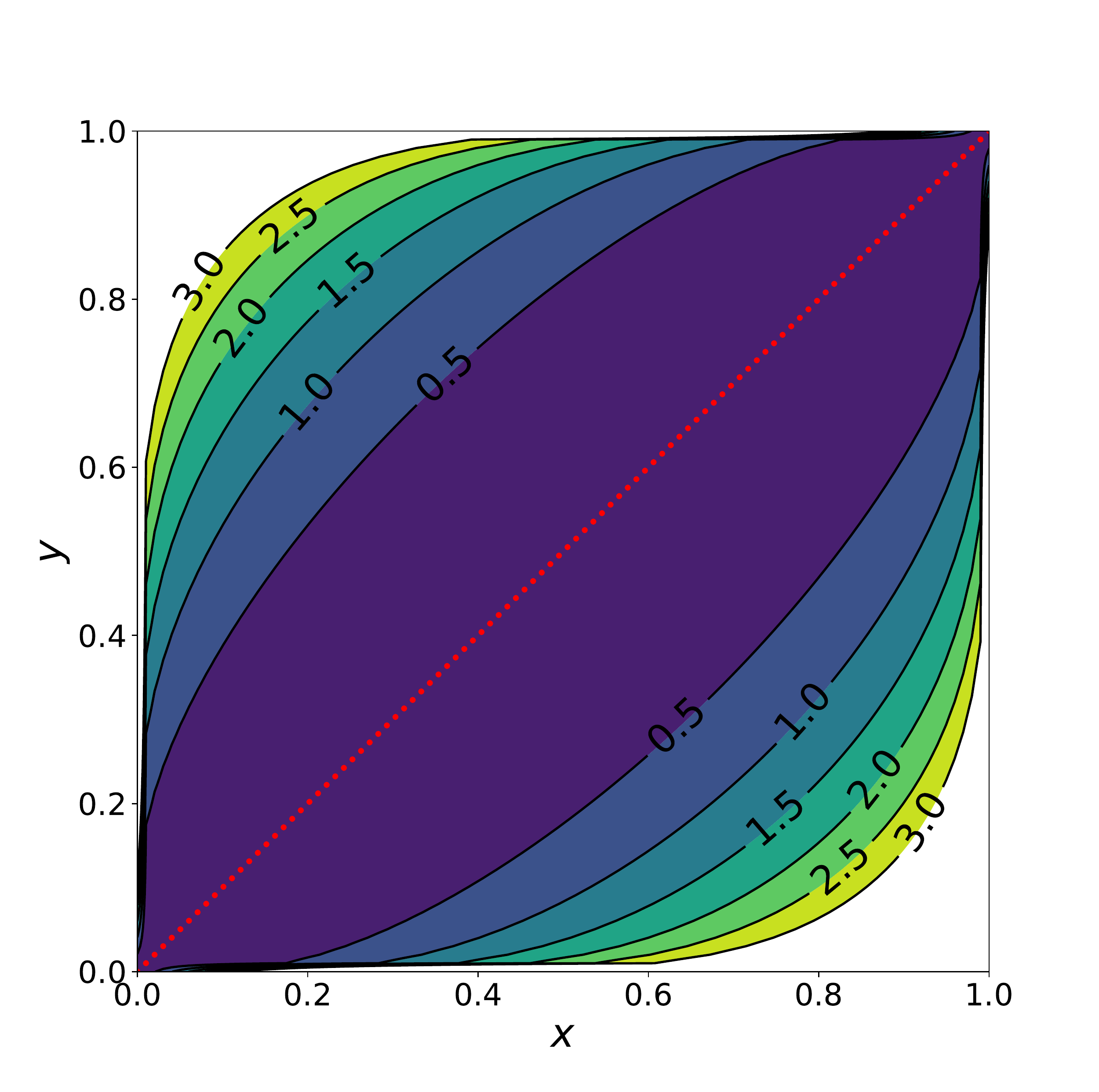}\end{tabular} \\
        $\JS$ & \begin{tabular}[c]{@{}l@{}}\includegraphics[trim=0.4in 0.4in 0.4in 0.4in, clip, width=0.30\linewidth]{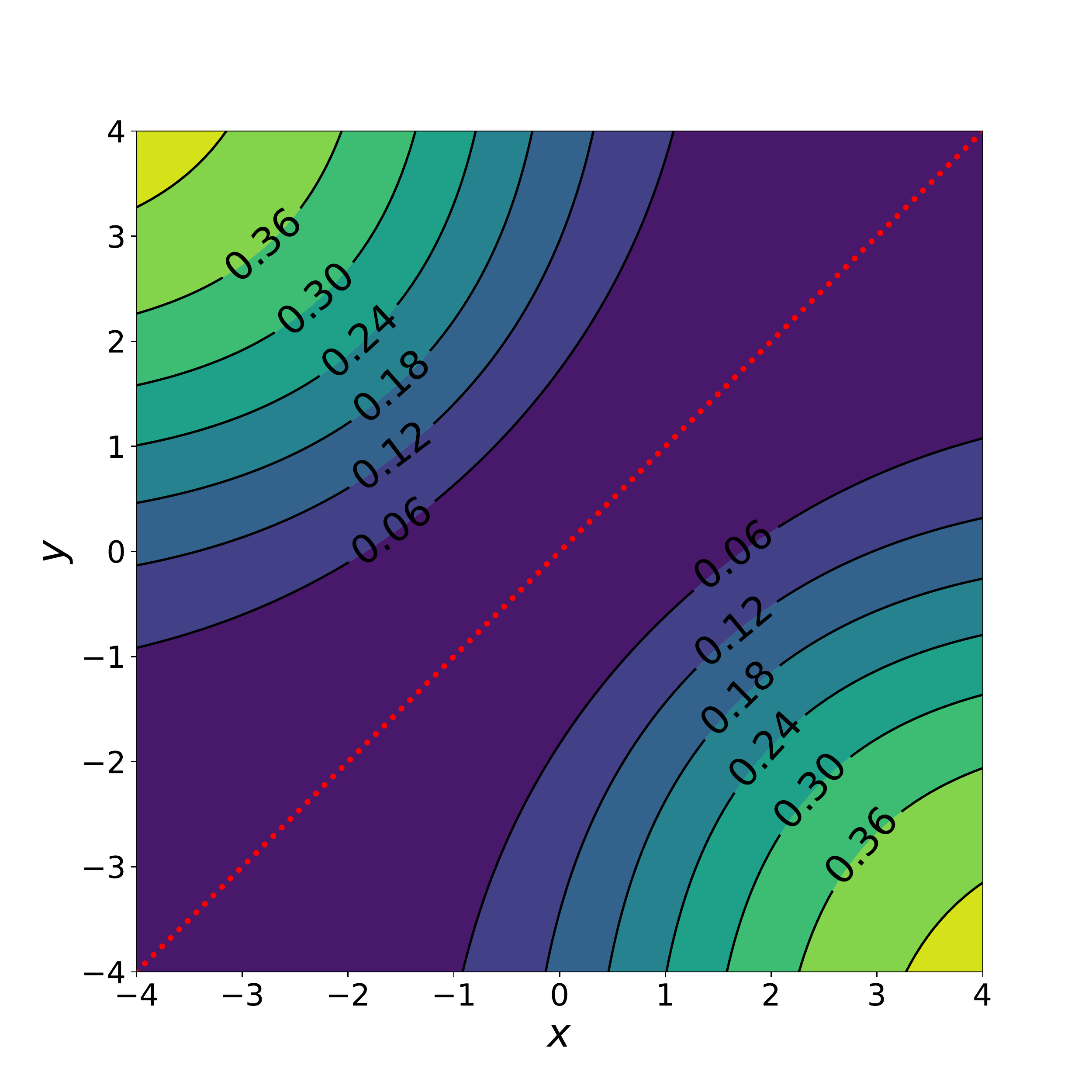}\end{tabular} & \begin{tabular}[c]{@{}l@{}}\includegraphics[trim=0.4in 0.4in 0.4in 0.4in, clip, width=0.30\linewidth]{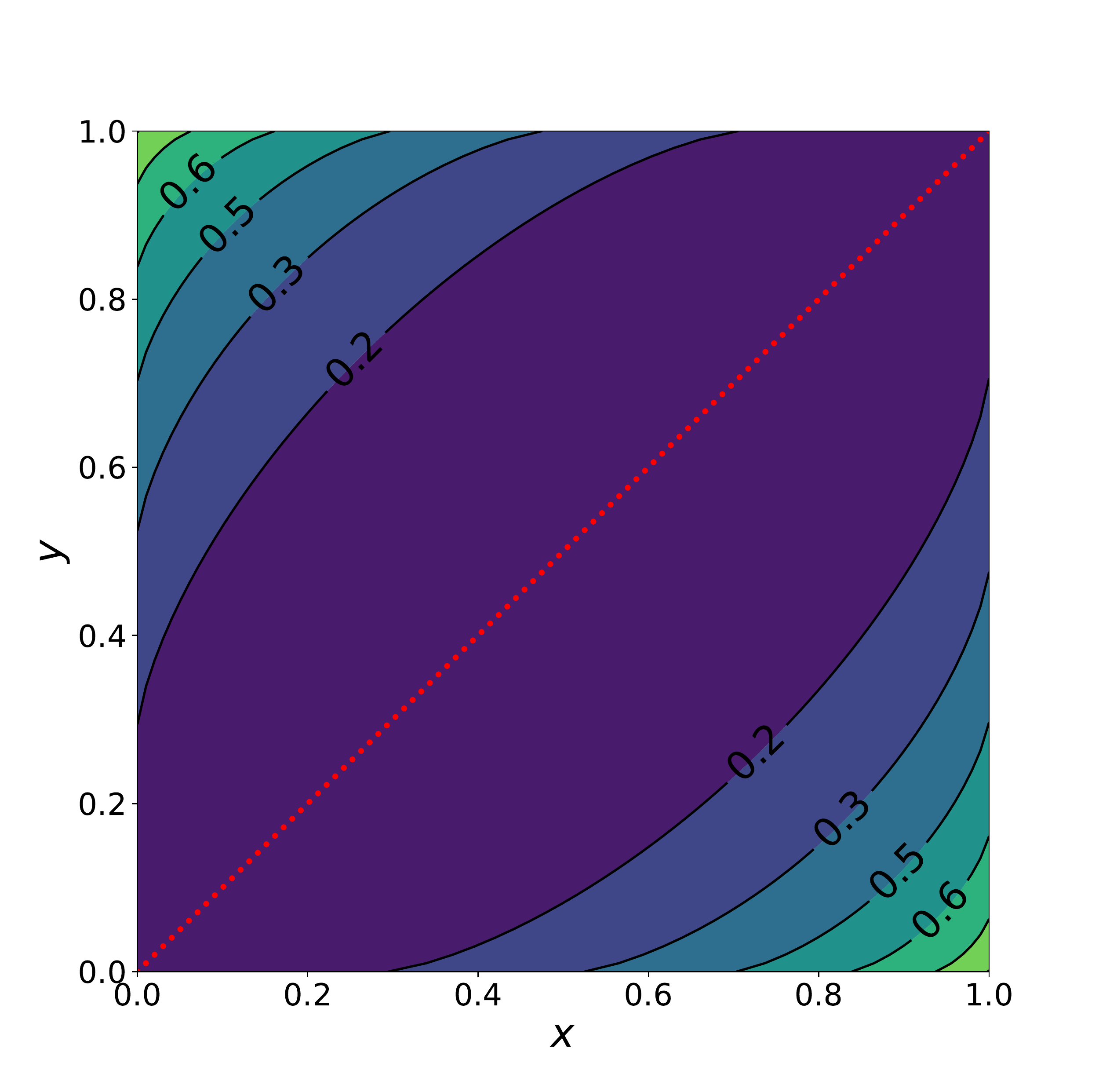}\end{tabular} \\
    \end{tabular}
\end{figure}

\begin{figure}[tp]
    \centering
    \hspace{-35pt}
    \large
    \begin{tabular}{lcc}
        & \cref{exp:verify-subadditivity-discrete-dependency} & \cref{exp:verify-subadditivity-discrete-initial} \\
        $\TV$ & \begin{tabular}[c]{@{}l@{}}\includegraphics[trim=0.4in 0.4in 0.4in 0.4in, clip, width=0.30\linewidth]{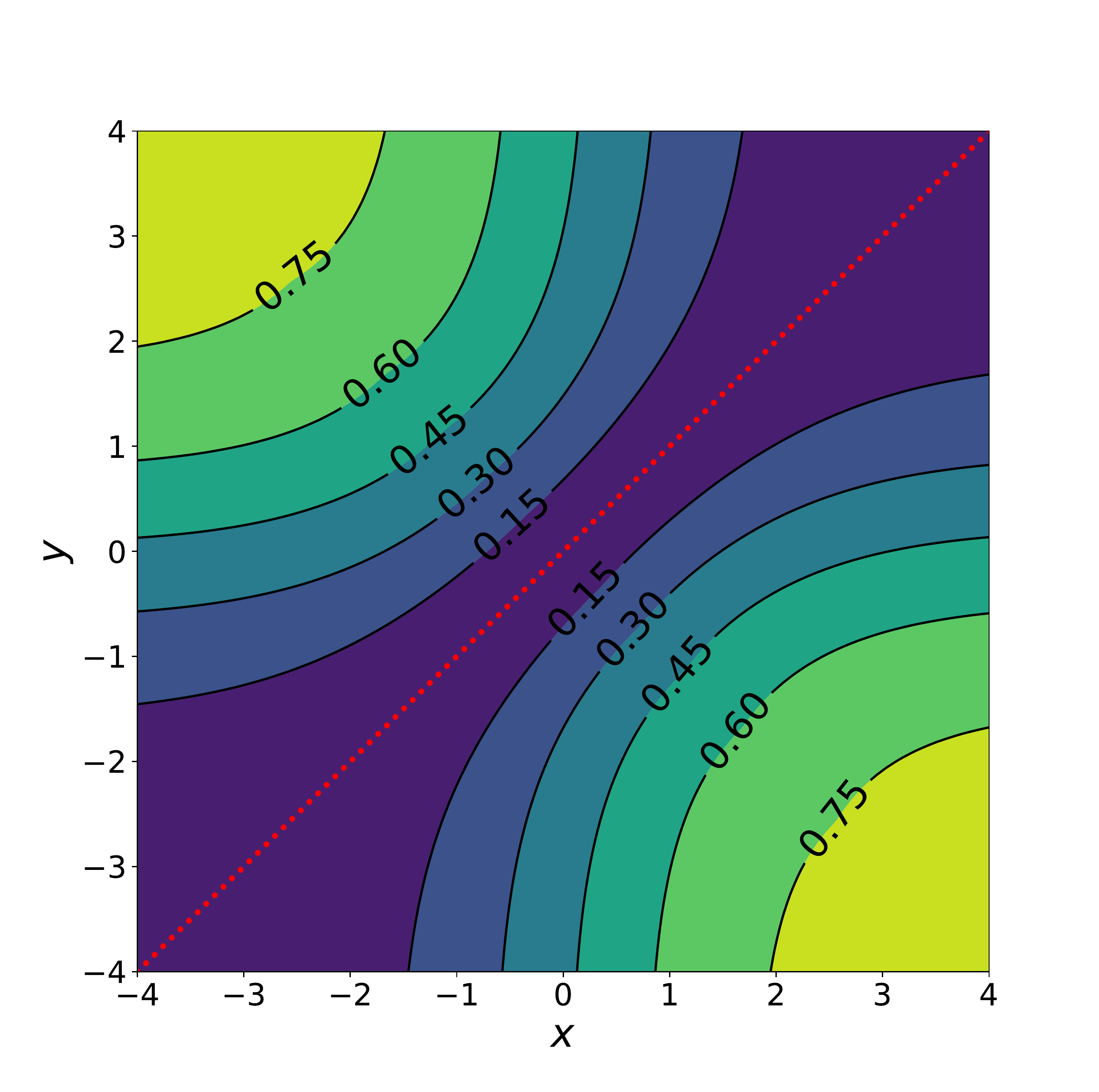}\end{tabular} & \begin{tabular}[c]{@{}l@{}}\includegraphics[trim=0.4in 0.4in 0.4in 0.4in, clip, width=0.30\linewidth]{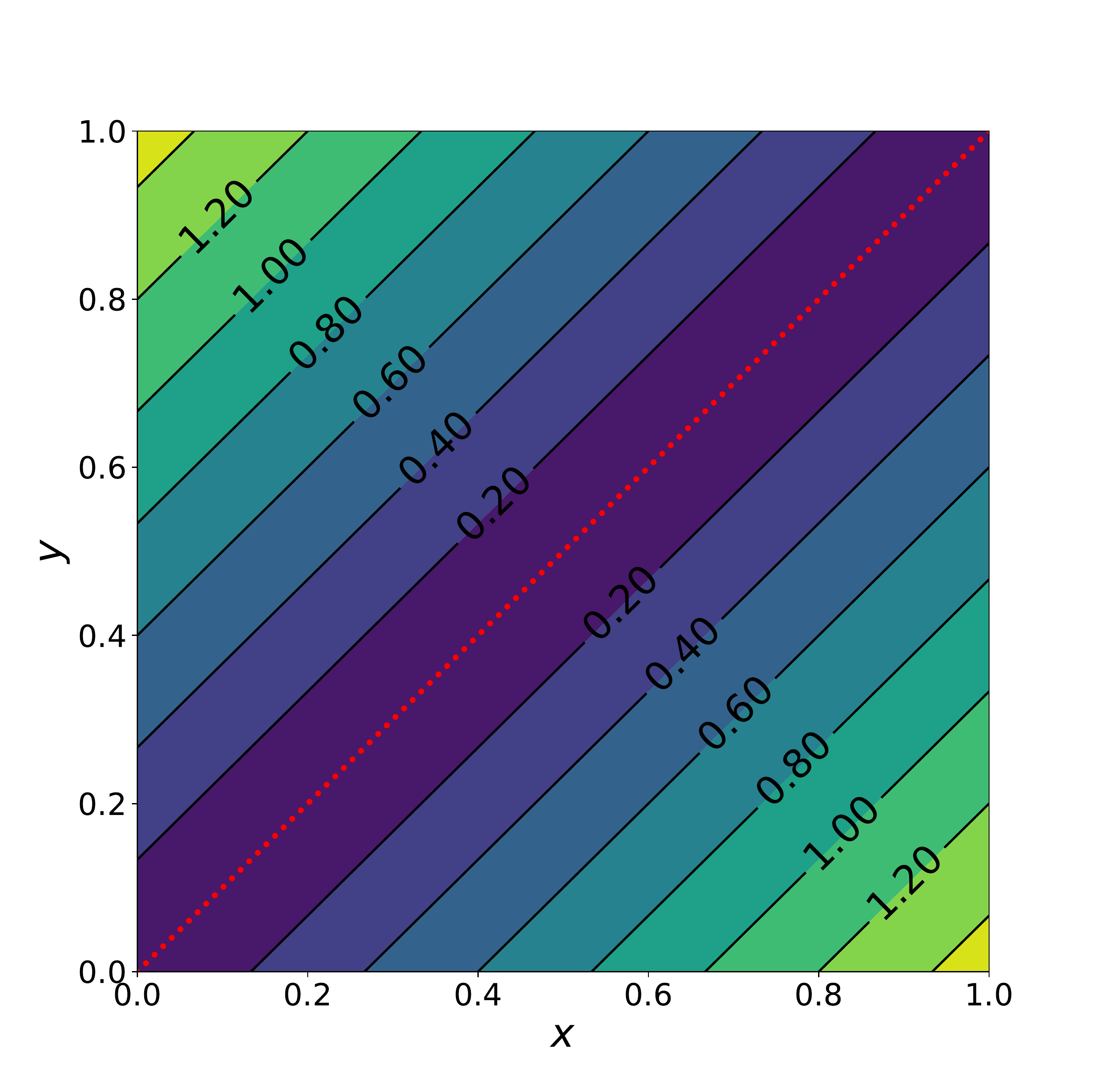}\end{tabular} \\
        $\W_1$ & \begin{tabular}[c]{@{}l@{}}\includegraphics[trim=0.4in 0.4in 0.4in 0.4in, clip, width=0.30\linewidth]{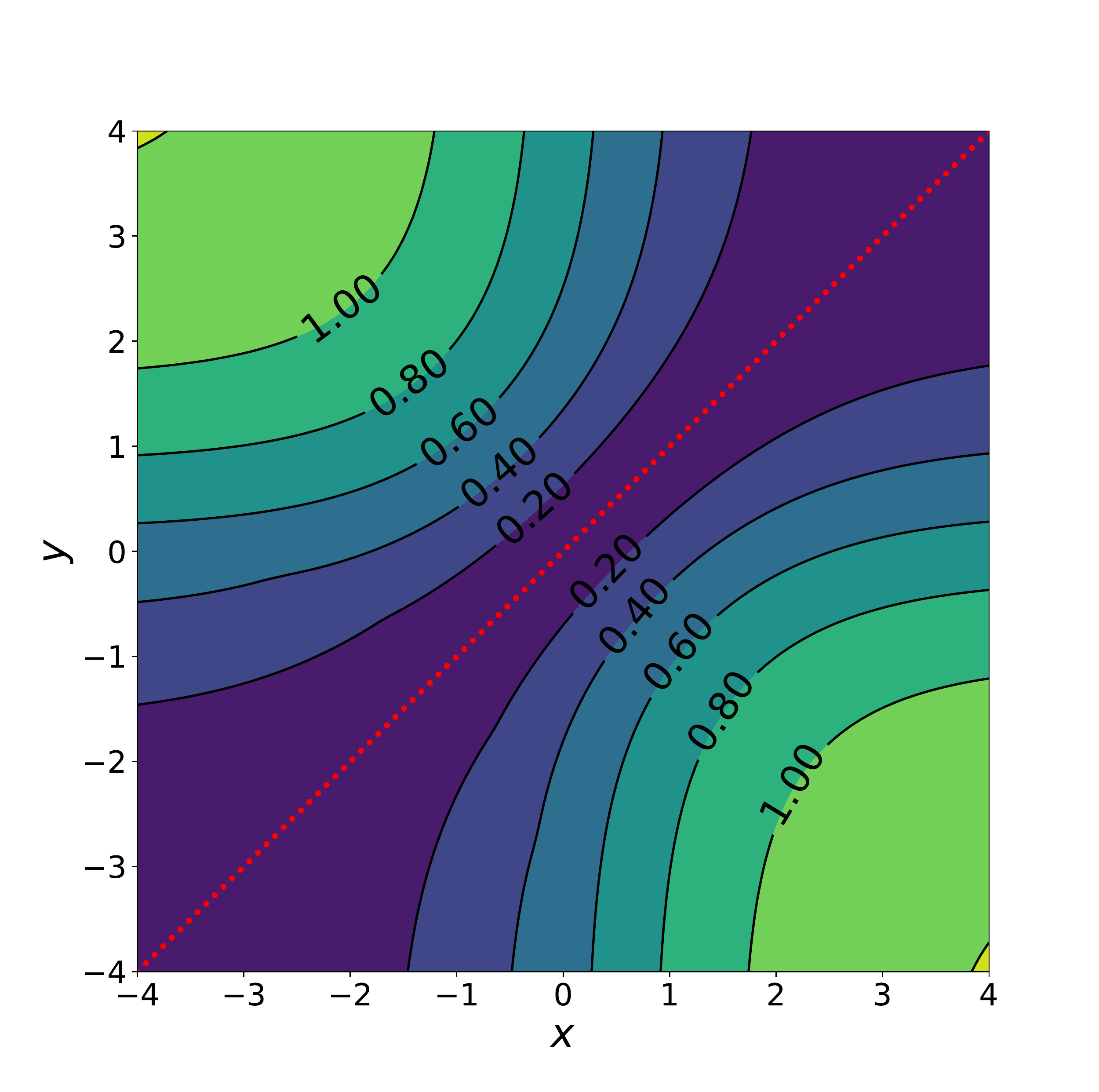}\end{tabular} & \begin{tabular}[c]{@{}l@{}}\includegraphics[trim=0.4in 0.4in 0.4in 0.4in, clip, width=0.30\linewidth]{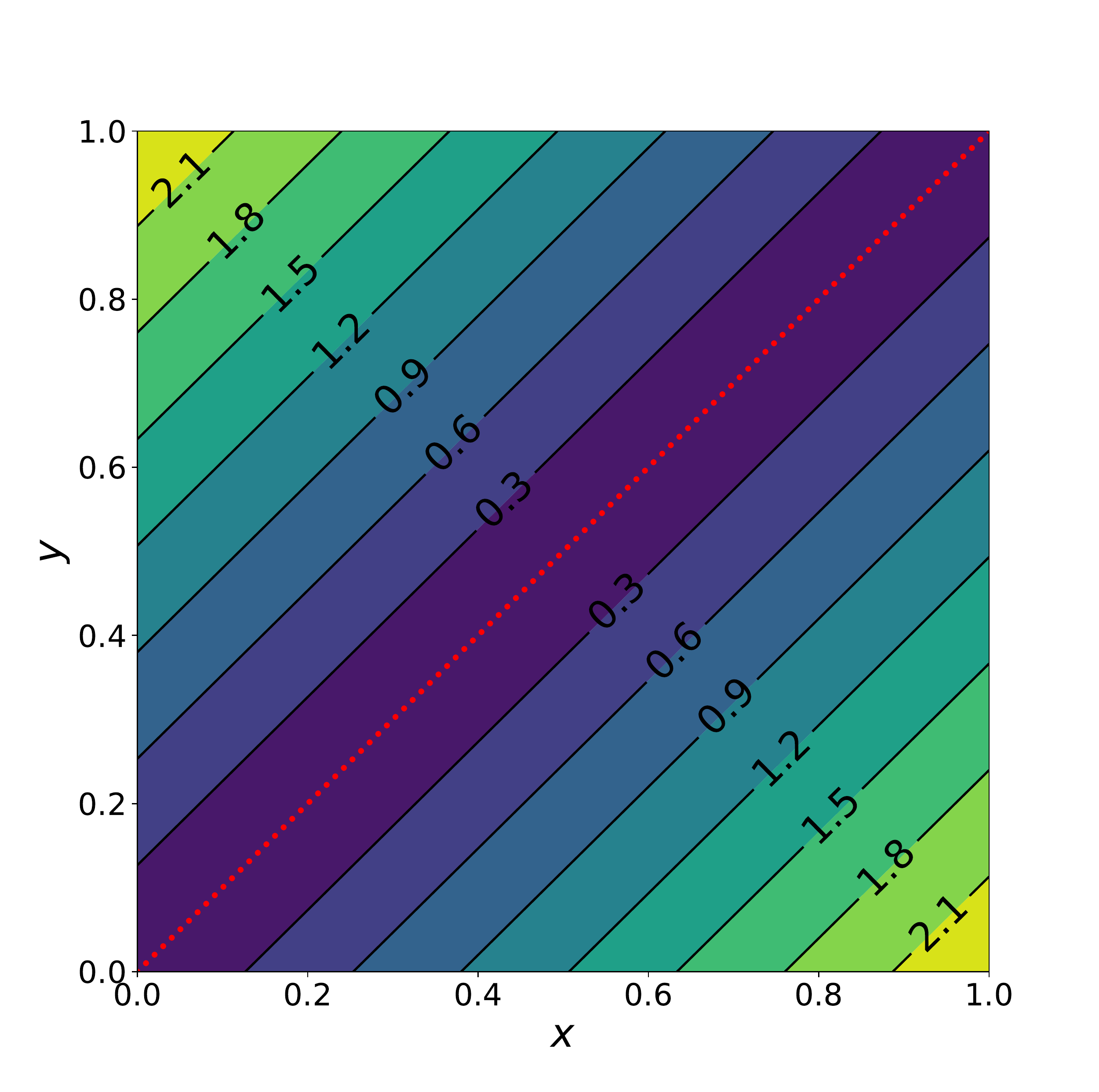}\end{tabular} \\
        $\W_2$ & \begin{tabular}[c]{@{}l@{}}\includegraphics[trim=0.4in 0.4in 0.4in 0.4in, clip, width=0.30\linewidth]{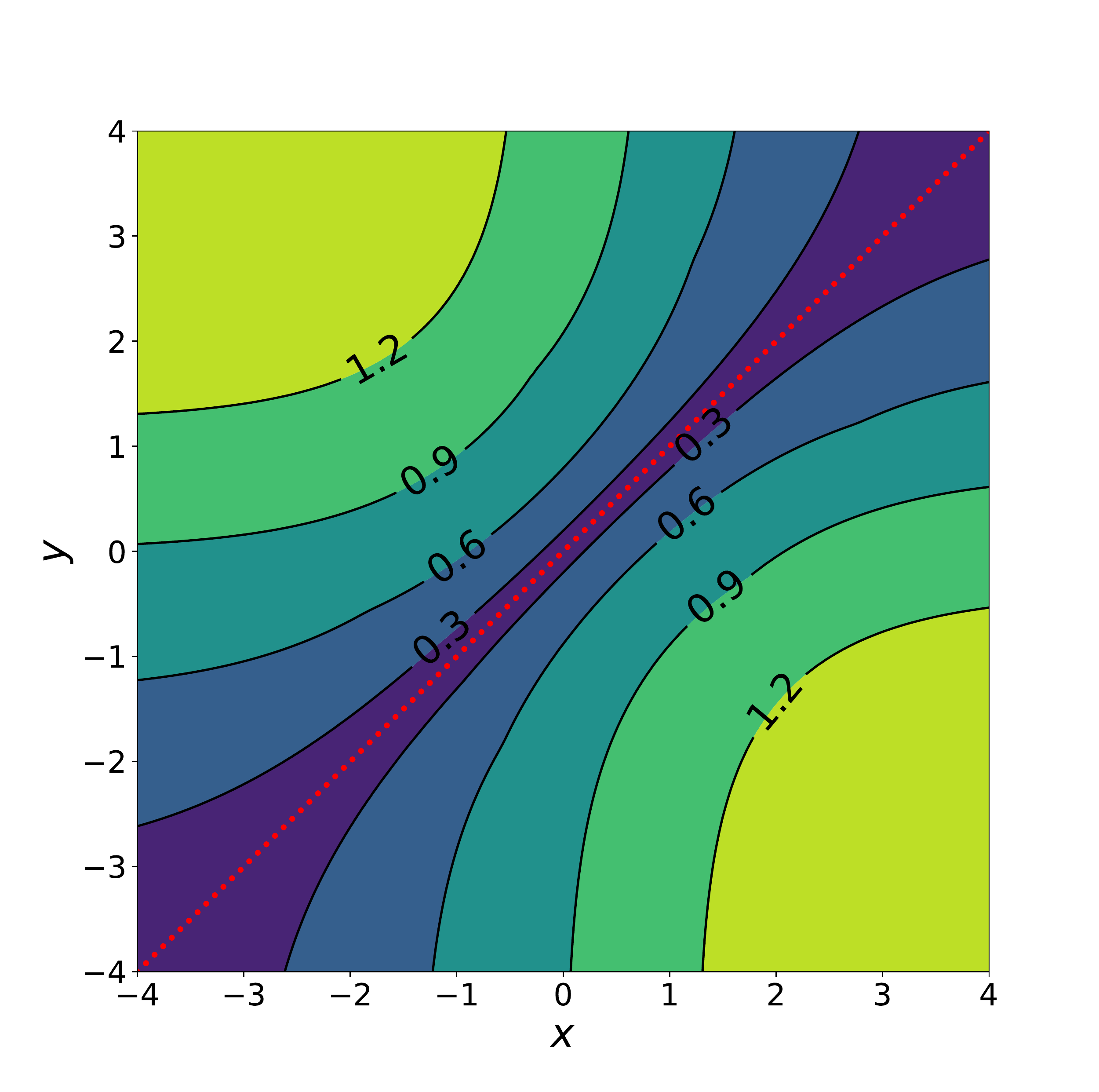}\end{tabular} & \begin{tabular}[c]{@{}l@{}}\includegraphics[trim=0.4in 0.4in 0.4in 0.4in, clip, width=0.30\linewidth]{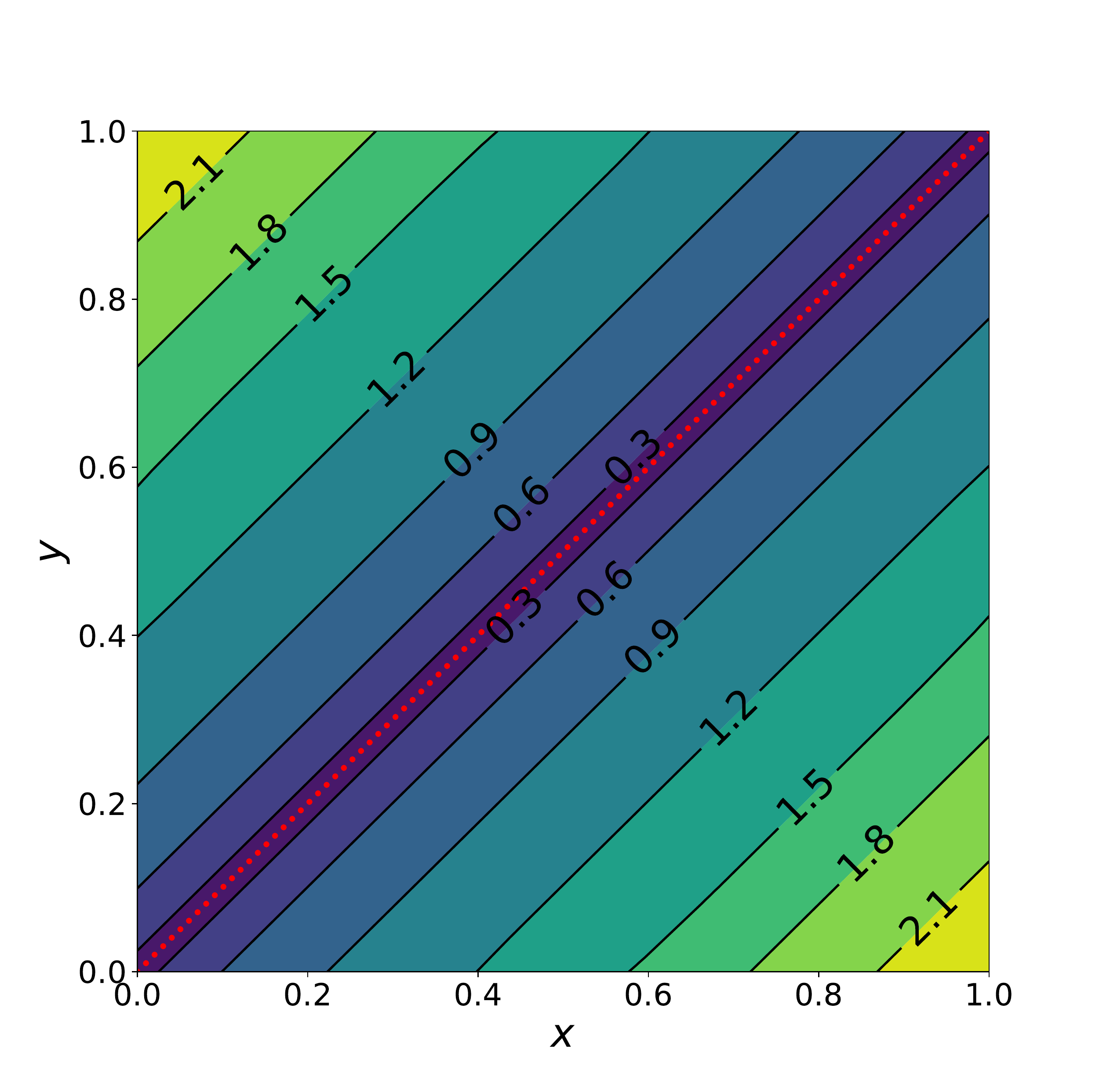}\end{tabular} \\
    \end{tabular}
    \caption{Contour maps showing the binary auto-regressive sequence examples of subadditivity or linear subadditivity of $\SH$, $\KL$, $\SKL$, $\JS$, $\TV$, $\W_1$, and $\W_2$.  The two distributions $P^x, Q^y$ are distributions of binary auto-regressive sequences with length $n=4$ and order $p=2$, following definitions in \cref{exp:verify-subadditivity-discrete-dependency} and \cref{exp:verify-subadditivity-discrete-initial}. The contours and colors indicate the subadditivity gap $\Delta=\sum_{t=p+1}^n \D(P^x_{\cup_{i=t-p}^t X_i}, Q^y_{\cup_{i=t-p}^t X_i})-\D(P^x, Q^y)$ (if $\D$ satisfies subadditivity) or $\Delta=\sum_{t=p+1}^n \D(P^x_{\cup_{i=t-p}^t X_i}, Q^y_{\cup_{i=t-p}^t X_i})-\ap\cdot\D(P^x, Q^y)$ (if $\D$ satisfies $\ap$-linear subadditivity). The red dotted line indicates places where the subadditivity gap is $0$. White regions have too large subadditivity gap to be colored.}
    \label{fig:verify-subadditivity}
\end{figure}

We verify the subadditivity of $\SH$, $\KL$, $\SKL$, and the linear subadditivity of $\JS$, $\TV$, $\W_1$ and $\W_2$ on these two examples, as shown in \cref{fig:verify-subadditivity}. We draw contour plots of the subadditivity gap $\Delta=\sum_{t=p+1}^n \D(P^x_{\cup_{i=t-p}^t X_i}, Q^y_{\cup_{i=t-p}^t X_i})-\D(P^x, Q^y)$ (if $\D$ satisfies subadditivity) or $\Delta=\sum_{t=p+1}^n \D(P^x_{\cup_{i=t-p}^t X_i}, Q^y_{\cup_{i=t-p}^t X_i})-\ap\cdot\D(P^x, Q^y)$ (if $\D$ satisfies $\ap$-linear subadditivity). All the inequalities are verified as we can visually confirm all contours are positive.

%%%%%%%%%%%%%%%%%%%%%%%%%%%%%%%%%%%%%%%%%%%%%%%%%%%%%%%%%%%%%%%%%%%%%%%%%%%%%%%%%%%%%%%%%%%%%%%%%%%%
\section{Empirical Verification of the Local Approximations of \texorpdfstring{$\gf$}{f}-Divergences}
\label{appendix:verfication-local}

In this section, we observe the local behavior of common $\gf$-divergences when the two distributions $P$ and $Q$ are sufficiently close. And we verify the conclusion of \cref{lem:local-fdiv-approximation}: all $\gf$-divergences $\PD$ with a generator function $\gf(t)$ that is twice differentiable at $t=1$ and satisfies $\gf''(1)>0$ have similar local approximations up to a constant factor up to $\cO(\eps^3)$. More specifically, for a pair of two-sided $\eps$-close distributions $P$ and $Q$, we verify all such $\gf$-divergences satisfy:
\[
    \PD\PQ{}=\frac{\gf''(1)}{2}\CS\PQ{}+\cO(\eps^3)
\]

Let us consider a simple example of two-sided close distributions on $\Sp=\RR$. Suppose $Q=\cN(0, 1)$ is the 1-dimensional unit Gaussian. Let $P(x)=\left(1+\eps\sin(x)\right)Q(x)$ for some $\eps\in(0,1)$. It is easy to verify that $P$ is a valid probability distribution: $\int_{-\infty}^\infty P(x)\dx=\int_{-\infty}^\infty Q(x)\dx+\eps\int_{-\infty}^\infty \sin(x)Q(x)\dx=1$, where the term $\int_{-\infty}^\infty \sin(x)Q(x)\dx$ vanishes because $Q(x)$ is an even function and $\sin(x)$ is odd. Since for any $x\in\Sp=\RR$, it holds that $P(x)/Q(x)=1+\eps\sin(x)\in[1-\eps,1+\eps]$, we know $P$ and $Q$ are two-sided $\eps$-close.

We compute several common $\gf$-divergences between such $P$ and $Q$, for different $\eps\in[0, 0.5]$, as shown in \cref{fig:verify-local-divergences}. We can see that, except for Total Variation distance which has a generator $\gf_\TV$ not differentiable at $1$, all common $\gf$-divergences behave similarly up to a constant factor. Actually, these curves cluster into three groups according to $f''(1)$. In the first cluster: $\gf_\SKL''(1)=\gf_{\CS}''(1)=\gf_{R\CS}''(1)=2$. In the second cluster: $\gf_\KL''(1)=\gf_{R\KL}''(1)=1$. While in the third cluster: $\gf_{\SH}''(1)=\gf_\JS''(1)=\frac{1}{4}$. Moreover, we visualize the differences between $\gf$-divergences normalized with respect to $\gf''(1)$ and $\CS$ divergence, for $\eps\in{[0, 0.01]}$. We can see in \cref{fig:verify-local-differences}, all the differences are very small. This verifies that all $\gf$-divergences such that $\gf''(1)>0$ satisfy $\frac{2}{f''(1)}\PD\PQ{}=\CS\PQ{}$ up to $\cO(\eps^3)$.

\clearpage
\newpage

\begin{figure}[H]
    \centering
    \begin{subfigure}{.49\linewidth}
        \centering
        \begin{minipage}[t]{\linewidth}
            \includegraphics[width=\linewidth]{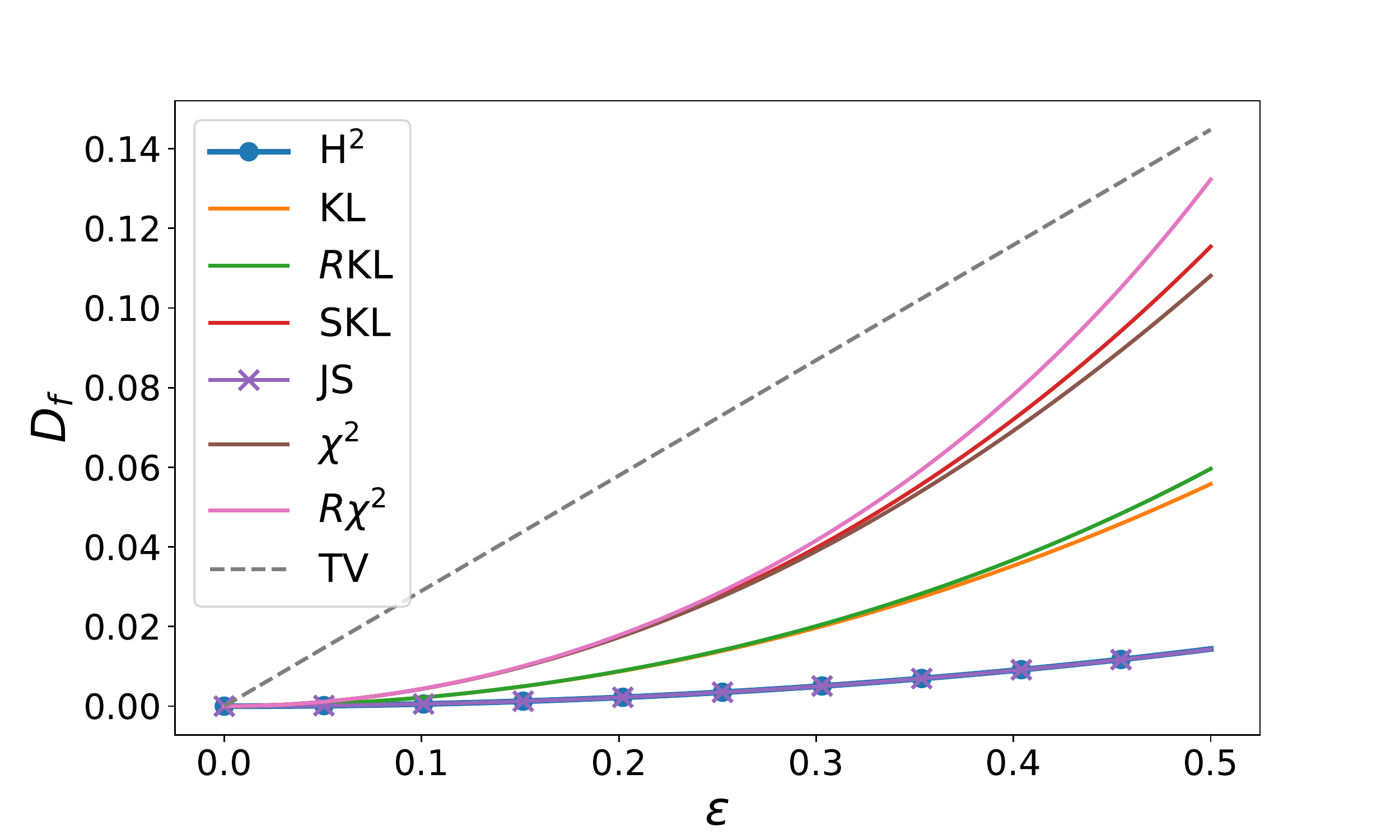}
        \end{minipage}
        \caption{Common $\gf$-divergences between such $P$ and $Q$\\ for $\eps\in{[0, 0.5]}$.}
        \label{fig:verify-local-divergences}
    \end{subfigure}
    \begin{subfigure}{.49\linewidth}
        \centering
        \begin{minipage}[t]{\linewidth}
            \includegraphics[width=\linewidth]{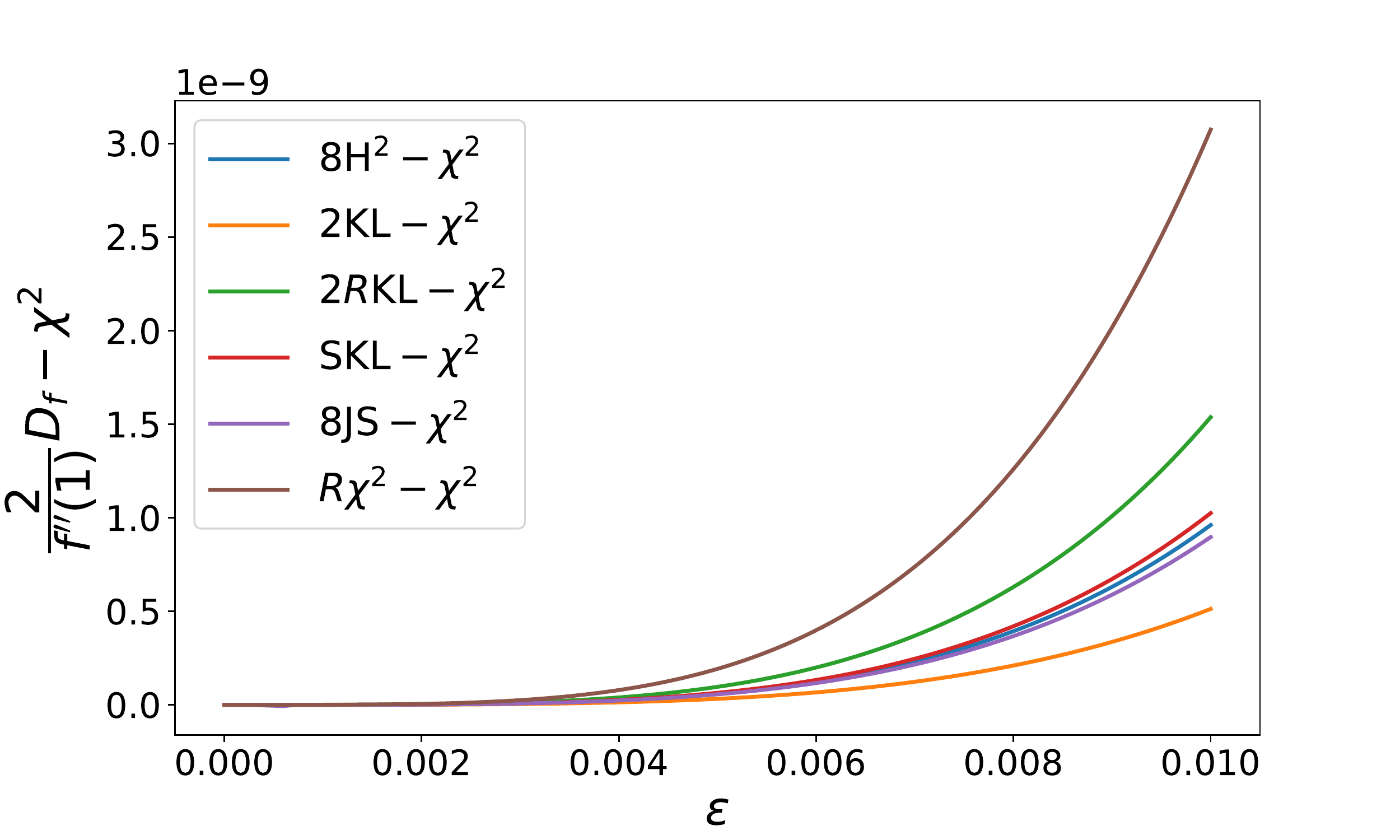}
        \end{minipage}
        \caption{Differences between $\gf$-divergences normalized with respect to $\gf''(1)$ and $\CS$ divergence for $\eps\in{[0, 0.01]}$.}
        \label{fig:verify-local-differences}
    \end{subfigure}
    \caption{Common $\gf$-divergences between two-sided $\eps$-close distributions $P, Q$, where $Q$ is the 1-dimensional unit Gaussian and $P(x)=\left(1+\eps\sin(x)\right)Q(x)$. In (a), we compare these $\gf$-divergences for $\eps\in[0, 0.5]$. In (b), we verify the conclusion of \cref{lem:local-fdiv-approximation}: $\frac{2}{f''(1)}\PD\PQ{}=\chi^2\PQ{}+\cO(\eps^3)$.}
    \label{fig:verify-local}
\end{figure}